\documentclass[pmlr]{jmlr}

\usepackage{amsmath,amsfonts,bm}

\def\Figref#1{Figure~\ref{#1}}

\def\Secref#1{Section~\ref{#1}}

\def\eqref#1{equation~\ref{#1}}
\def\Eqref#1{Equation~\ref{#1}}

\def\1{\bm{1}}

\def\vb{{\bm{b}}}

\def\vh{{\bm{h}}}

\def\vx{{\bm{x}}}

\def\vz{{\bm{z}}}

\def\mW{{\bm{W}}}
\def\mX{{\bm{X}}}

\DeclareMathAlphabet{\mathsfit}{\encodingdefault}{\sfdefault}{m}{sl}
\SetMathAlphabet{\mathsfit}{bold}{\encodingdefault}{\sfdefault}{bx}{n}

\usepackage{longtable}

 \usepackage{booktabs}
 
\usepackage[load-configurations=version-1]{siunitx} 

\usepackage{float}
\usepackage{makecell}
\usepackage{multirow}
\usepackage{colortbl}
\usepackage[bottom]{footmisc}
\usepackage{wrapfig}
\usepackage{xr}

\makeatletter
\def\set@curr@file#1{\def\@curr@file{#1}} 
\makeatother

\makeatletter
\def\moverlay{\mathpalette\mov@rlay}
\def\mov@rlay#1#2{\leavevmode\vtop{%
   \baselineskip\z@skip \lineskiplimit-\maxdimen
   \ialign{\hfil$\m@th#1##$\hfil\cr#2\crcr}}}
\newcommand{\charfusion}[3][\mathord]{
    #1{\ifx#1\mathop\vphantom{#2}\fi
        \mathpalette\mov@rlay{#2\cr#3}
      }
    \ifx#1\mathop\expandafter\displaylimits\fi}
\makeatother
\newcommand{\cupdot}{\charfusion[\mathbin]{\cup}{\cdot}}

\newcommand{\tentry}[2]{#1{$\pm$#2}}

\newcount\tmpnum \newdimen\tmpdim
{\lccode`\?=`\p \lccode`\!=`\t  \lowercase{\gdef\ignorept#1?!{#1}}}
\edef\widecharS{\expandafter\ignorept\the\fontdimen1\textfont1}

\def\widebar#1{\futurelet\next\widebarA#1\widebarA}
\def\widebarA#1\widebarA{%
   \def\tmp{0}\ifcat\noexpand\next A\def\tmp{1}\fi
   \widebarE
   \ifdim\tmp pt=0pt \overline{#1}%
   \else {\mathpalette\widebarB{#1}}\fi
}
\def\widebarB#1#2{%
   \setbox0=\hbox{$#1\overline{#2}$}%
   \tmpdim=\tmp\ht0 \advance\tmpdim by-.4pt
   \tmpdim=\widecharS\tmpdim
   \kern\tmpdim\overline{\kern-\tmpdim#2}%
}
\def\widebarC#1#2 {\ifx#1\end \else 
   \ifx#1\next\def\tmp{#2}\widebarD 
   \else\expandafter\expandafter\expandafter\widebarC
   \fi\fi
}
\def\widebarD#1\end. {\fi\fi}
\def\widebarE{\widebarC A1.4 J1.2 L.6 O.8 T.5 U.7 V.3 W.1 Y.2 
   a.5 b.2 d1.1 h.5 i.5 k.5 l.3 m.4 n.4 o.6 p.4 r.5 t.4 v.7 w.7 x.8 y.8
   \alpha1 \beta1 \gamma.6 \delta.8 \epsilon.8 \varepsilon.8 \zeta.6 \eta.4
   \theta.8 \vartheta.8 \iota.5 \kappa.8 \lambda.5 \mu1 \nu.5 \xi.7 \pi.6
   \varpi.9 \rho1 \varrho1 \sigma.7 \varsigma.7 \tau.6 \upsilon.7 \phi1
   \varphi.6 \chi.7 \psi1 \omega.5 \cal1 \end. }

\newcommand{\Mod}[1]{\ (\mathrm{mod}\ #1)}

\definecolor{Gainsboro}{HTML}{DCDCDC}
\definecolor{ForestGreen}{HTML}{009a55}
\definecolor{BrickRed}{HTML}{b6311e}
\definecolor{SunYellow}{HTML}{feec71}
\definecolor{DeepBlue}{HTML}{586e8b}
\definecolor{LegendaryBlue}{HTML}{3b7ab5}
\definecolor{LegendaryOrange}{HTML}{ff8225}
\definecolor{LegendaryGreen}{HTML}{3caf50}
\definecolor{LegendaryPink}{HTML}{fc7fbd}
\definecolor{VennOrange}{HTML}{ec7c30}
\definecolor{VennBlue}{HTML}{2e5496}
\definecolor{VennGreen}{HTML}{538235}
\definecolor{CentreBlue}{HTML}{2e5496}
\definecolor{CentreGreen}{HTML}{538235}
\definecolor{CentreOrange}{HTML}{c55a11}
\definecolor{CentreYellow}{HTML}{be9000}
\definecolor{WorkflowBlue}{HTML}{093b92}
\definecolor{WorkflowGreen}{HTML}{929f00}
\definecolor{WorkflowOrange}{HTML}{9f5800}

\jmlrvolume{182}
\jmlryear{2022}
\jmlrworkshop{Machine Learning for Healthcare}

\title[Debiasing Deep Chest X-Ray Classifiers]{Debiasing Deep Chest X-Ray Classifiers using \\Intra- and Post-processing Methods}

\author{\Name{Ri\v{c}ards Marcinkevi\v{c}s}
       \Email{ricards.marcinkevics@inf.ethz.ch} 
       \AND
       \Name{Ece Ozkan}
       \Email{ece.oezkanelsen@inf.ethz.ch} 
       \AND
       \Name{Julia E. Vogt}
       \Email{julia.vogt@inf.ethz.ch}
       \AND
       \addr Department of Computer Science\\
       ETH Zurich\\
       Zurich, Switzerland} 

\begin{document}

\maketitle

\begin{abstract}
    Deep neural networks for image-based screening and computer-aided diagnosis have achieved expert-level performance on various medical imaging modalities, including chest radiographs. 
    Recently, several works have indicated that these state-of-the-art classifiers can be biased with respect to sensitive patient attributes, such as race or gender, leading to growing concerns about demographic disparities and discrimination resulting from algorithmic and model-based decision-making in healthcare. 
    Fair machine learning has focused on mitigating such biases against disadvantaged or marginalised groups, mainly concentrating on tabular data or natural images.
    This work presents two novel intra-processing techniques based on fine-tuning and pruning an already-trained neural network. These methods are simple yet effective and can be readily applied \emph{post hoc} in a setting where the protected attribute is unknown during the model development and test time. In addition, we compare several intra- and post-processing approaches applied to debiasing deep chest X-ray classifiers.  
    To the best of our knowledge, this is one of the first efforts studying debiasing methods on chest radiographs. Our results suggest that the considered approaches successfully mitigate biases in fully connected and convolutional neural networks offering stable performance under various settings. The discussed methods can help achieve group fairness of deep medical image classifiers when deploying them in domains with different fairness considerations and constraints.
\end{abstract}

\section{Introduction}
Chest X-ray imaging is an essential tool for screening and diagnosing conditions affecting the chest and its surrounding,  requiring special training for an appropriate interpretation. There has been an increasing effort in deploying deep neural networks for image-based screening and computer-aided diagnosis on chest radiographs \citep{Allaouzi2019,Cohen2020,Bressem2020} from various datasets \citep{Rajpurkar2017,Wang2019,mimic-cxr}, with some models achieving an expert-level performance \citep{Irvin2019}. However, several works \citep{Larrazabal2020,SeyyedKalantari2020,SeyyedKalantari2021} have shown that these classifiers may be biased, raising ethical concerns regarding ML systems involved in high-stakes decisions \citep{Char2018,Wiens2019,Obermeyer2019}. For instance, an ICU patient monitoring and management model trained on a dataset containing few patients from minority groups might suffer from under- or over-detection of events in these groups, leading to alarm fatigue among medical staff and disparate patient outcomes \citep{Rajkomar2018}. 

\begin{wrapfigure}[19]{O}{6.5cm}
    \vspace{-0.25cm}
    \centering
    \includegraphics[width=0.95\linewidth]{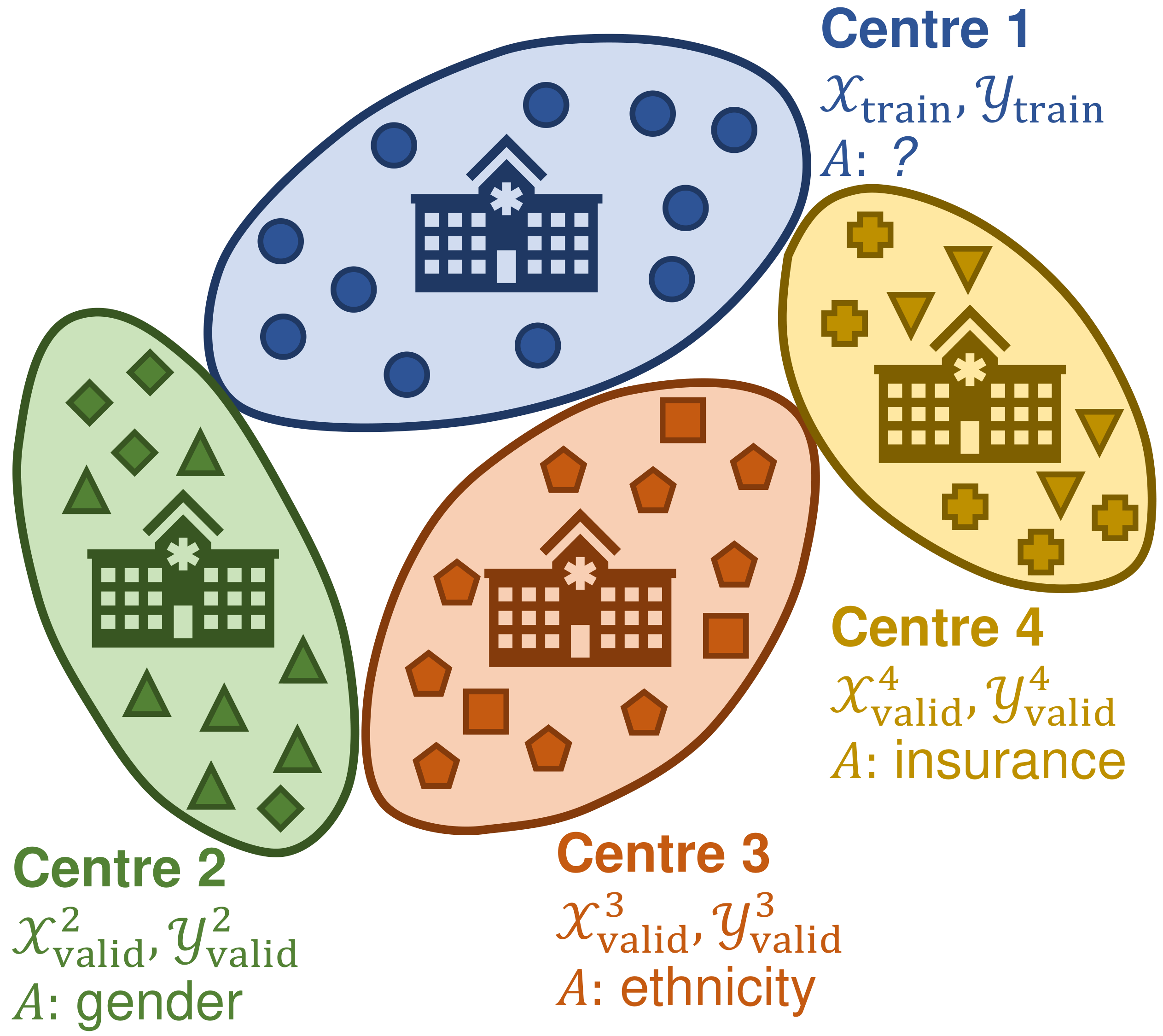}
    \caption{The intra-processing setting. A model is trained on centre {\color{CentreBlue}\textbf{1}}, and debiased on centres {\color{CentreGreen}\textbf{2}}, {\color{CentreOrange}\textbf{3}}, and {\color{CentreYellow}\textbf{4}} that have different fairness constraints ($A$).\label{fig:setting_motivation}}
\end{wrapfigure} Motivated by similar concerns, researchers have provided many solutions for adjusting models' outputs and directly incorporating fairness into the learning process \citep{Kearns2017}. In this paper, we will assess the fairness of neural networks from the perspective of classification parity \citep{CorbettDavies2018}: a classifier is said to be fair if some derivative of its confusion matrix, for instance, the true positive rate (TPR), is even across the categories of the protected attribute, such as race or gender. A practical scenario of mitigating bias w.r.t. protected attributes could be as follows. Consider deploying a predictive neural-network-based model in several clinical centres with different demographics, e.g. as explored by \citet{Zech2018} for chest X-ray classification. The constraints on the bias and protected attribute of interest might vary across clinical centres due to different population demographics (see \Figref{fig:setting_motivation}). Therefore, it might be more practical to debias the original model based on the local data, following an intra- or post-processing approach \citep{Bellamy2018,Savani2020}. The setting above is even more relevant with the widespread availability and use of pre-trained models \citep{Gupta2018,Raghu2019,Rasmy2021}.

Several prior works on debiasing classifiers, \mbox{i.e.} minimising bias, have mainly concentrated on tabular data or natural images, e.g. by \citet{Zafar2017,Zhang2018,Kim2019LearningNotToLearn,Reimers2021}, and assumed that bias constraints are known before or during training. In that case, the model's inputs could be transformed or reweighed, or bias constraints could be incorporated directly into the loss function. Another line of work has focused on a completely model-agnostic approach \citep{Kamiran2012,Hardt2016}, adjusting the model's predictions \emph{post hoc} and requiring the knowledge of the protected attribute at test time. On the other hand, the intra-processing scenario emerges naturally when potential biases are unknown or unexplored at the model development time, in other words, when it is impractical or impossible to train a debiased classifier from scratch, as described above, and when the protected attribute is unavailable at test time. Nevertheless, similar to most methods for enforcing classification parity, intra-processing does require the protected attribute \emph{during debiasing}. Although some techniques can be applied when the sensitive attribute is \emph{entirely} unknown, the current work focuses on a different setting.

We propose two novel intra-processing debiasing techniques based on fine-tuning and pruning. Furthermore, we compare several previously proposed debiasing approaches applied to deep chest X-ray classifiers in terms of statistical parity \citep{Besse2021} and equality of opportunity \citep{Hardt2016}. We believe that this is one of the first works comparing debiasing methods on chest X-ray images. 
We exploit the publicly available and widely used MIMIC-CXR (Medical Information Mart for Intensive Care -- Chest X-Ray) dataset \citep{mimic-cxr}. Classifiers trained on this dataset have been shown to be biased w.r.t. various sensitive attributes, such as gender, race, or insurance type \citep{SeyyedKalantari2020,SeyyedKalantari2021}.

\subsection*{Generalizable Insights about Machine Learning in the Context of Healthcare}

The main contributions of this work are as follows. \mbox{(\emph{i}) We} consider differentiable proxy functions for statistical parity and equality of opportunity and establish their correspondence to the covariance between the decision boundary of a neural network and the protected attribute. \mbox{(\emph{ii}) We} introduce simple yet effective intra-processing debiasing procedures based on minimising the proxy functions via fine-tuning and pruning an already-trained neural network, which can be effective in a setting where the protected attribute is unknown during the model development and test time. \mbox{(\emph{iii}) We} conduct a comprehensive comparison among the proposed and well-established debiasing approaches on fully connected and convolutional neural networks on several datasets, including chest X-rays, showing that the compared methods can help achieve group fairness when deploying them in domains with different fairness considerations and constraints.

\section{Preliminaries}

Below, we outline the setting considered throughout the paper. We assume that disjoint training, validation, and test datasets $\mathcal{D}=\left\{\left(\vx_i,y_i,a_i\right)\right\}_{i}=\mathcal{D}_{\mathrm{train}}\cupdot\mathcal{D}_{\mathrm{valid}}\cupdot\mathcal{D}_{\mathrm{test}}$ are given, where $\vx_i$ are features, e.g. a $p$-dimensional vector or an image, $y_i\in\{0,1\}$ is the label, and $a_i\in\{0,1\}$ is the protected attribute. Attribute $a_i$ may be present among the features in $\vx_i$ or may be completely exogenous. We will use the capital letters $\mX,\,Y,$ and $A$ to refer to the corresponding random variables. Furthermore, let $\mathcal{X}=\left\{\vx_i\right\}_i$, $\mathcal{Y}=\left\{y_i\right\}_i$, and $\mathcal{A}=\left\{a_i\right\}_i$. Let $f_{\boldsymbol{\theta}}(\cdot)$ denote a neural network parameterised by $\boldsymbol{\theta}$ and trained on data points $\left\{\left(\vx_i,y_i\right)\right\}_i$ from $\mathcal{D}_{\mathrm{train}}$. In our experiments (see \Secref{sec:expsetup}), we consider fully connected and convolutional architectures for $f_{\boldsymbol{\theta}}(\cdot)$. If  $f_{\boldsymbol{\theta}}(\cdot)$ is a multilayer perceptron, $\boldsymbol{\theta}$ is given by weight matrices $\left\{\mW^{\mathrm{in}},\mW^1,...,\mW^L,\mW^{\mathrm{out}}\right\}$.
We will use $\vz^l\left(\vx\right)$ for the pre-activations and $\vh^l(\vx)=\sigma\left(\vz^l(\vx)\right)$ for activations in layer $1\leq l \leq L$, at the input $\vx$, where $\sigma(\cdot)$ is an activation function. The output of $f_{\boldsymbol{\theta}}(\cdot)$ is given by $\mathrm{sigmoid}\left(\mW^{\mathrm{out}}\vh^L(\vx)\right)$. For final classification, a threshold $t\in\left[0,\,1\right]$ on the output is chosen by maximising some performance measure, e.g. accuracy, on held-out data $\mathcal{D}_{\mathrm{valid}}$. Thus, for input $\vx$, the prediction is $\hat{y}=\boldsymbol{1}_{\left\{f_{\boldsymbol{\theta}\left(\vx\right)\geq t}\right\}}$, where $\boldsymbol{1}_{\left\{\cdot\right\}}$ is the indicator function.

\section{Background and Related Work}

\paragraph{Classification Parity}

Many criteria for the fairness of machine learning models have been considered so far \citep{CorbettDavies2018}. The two most common and practical classification parity metrics are statistical parity and equality of opportunity. Statistical parity difference (SPD) \citep{Savani2020,Besse2021} is defined as the difference between the probabilities of positive outcomes, i.e. predictions made by the model $f_{\boldsymbol{\theta}}(\cdot)$, across the groups of the protected attribute $A$:
\begin{equation}
    \mathrm{SPD}=\mathbb{P}_{\mX,A}\left(\widehat{Y}=1\middle|A=0\right)-\mathbb{P}_{\mX,A}\left(\widehat{Y}=1\middle|A=1\right).
   \label{eqn:spd}
\end{equation} 
On the other hand, the equal opportunity difference (EOD) \citep{Hardt2016,Savani2020} quantifies the discrepancy between the TPRs of the classifier $f_{\boldsymbol{\theta}}(\cdot)$: 
\begin{equation}
    \mathrm{EOD}=\mathbb{P}_{\mX,Y,A}\left(\widehat{Y}=1\middle|Y=1,A=0\right)-\mathbb{P}_{\mX,Y,A}\left(\widehat{Y}=1\middle|Y=1,A=1\right).
    \label{eqn:eod}
\end{equation}
In practice, quantities from Equations~\ref{eqn:spd} and \ref{eqn:eod} can be evaluated using empirical estimators on held-out test data.

\paragraph{Debiasing}

Minimisation of the SPD or EOD is a solvable technical problem. Debiasing, i.e. the minimisation of bias, often leads to a decrease in the overall predictive performance of the classifier \citep{Reimers2021}. Therefore, ideally, a debiasing algorithm should reduce bias $\mu(\cdot)$, given by the SPD or EOD, without sacrificing performance $\rho(\cdot)$, e.g. balanced accuracy (BA) \citep{Brodersen2010}.
One can view this problem as an instance of constrained optimisation \citep{Zafar2017,Zafar2019,Kim2019,Savani2020}, either minimising the bias subject to performance constraints or \emph{vice versa}, maximising performance under bias constraints. 

Many debiasing algorithms have been proposed for the setting outlined above. \citet{Bellamy2018} and \citet{Savani2020} provide a practical taxonomy: \mbox{(\emph{i}) pre-processing} algorithms usually reweigh or transform original data, obfuscating protected variables or attenuating group disparities \citep{Kamiran2011,Zemel2013,Calmon2017,Celis2020}; \mbox{(\emph{ii}) in-processing} methods incorporate debiasing explicitly into learning, e.g. using an adversarial loss or regularisation \citep{Kamishima2012,Zafar2017,Zhang2018,Reimers2021}; \mbox{(\emph{iii}) post-processing} approaches treat the biased model as a black-box and merely edit its predictions \citep{Kamiran2012,Hardt2016,Pleiss2017}; last but not least, \mbox{(\emph{iv}) intra-processing} techniques are inspired by fine-tuning and achieve parity by changing the model's parameters \emph{post hoc} \citep{Savani2020}. An essential difference between post- and intra-processing are assumptions about the access to model parameters and the protected attribute at test time: post-processing adjusts \emph{predictions} based on the given protected attribute value. 

All of the methods mentioned above assume the knowledge of the protected attribute at some point in the model's life cycle. Another line of work \citep{Nam2020LfF,Lee2021}, beyond the scope of the current paper, focuses on the setting wherein the source of bias is \emph{entirely} unknown, usually resorting to strong assumptions, such as that the bias is easier to learn than other relevant associations. 

\paragraph{Pruning}

In neural networks, parameter pruning usually refers to removing irrelevant weights or entire structural elements \citep{Cheng2017}, e.g. filters in convolutional neural networks. Early works on pruning neural networks, such as optimal brain damage \citep{LeCun1990} and optimal brain surgeon \citep{Hassibi1993}, leveraged criteria based on the second derivative of the error function to prune unimportant weights throughout the training process. Several modern techniques focus on pruning entire structures \citep{Wen2016,Molchanov2016,He2017}, e.g. convolutional filters or channels. However, the main principle remains the same: parameters are pruned based on some criterion, and the network is subsequently fine-tuned by backpropagation, if necessary.

\paragraph{Role of Individual Units in Neural Networks}

Several works have investigated the importance and interpretation of \emph{individual} neurons within deep neural network models, in contrast to the previous research on attribution, which primarily examined input-output relationships \citep{Ancona2019}. For instance, \citet{Bau2020} observed the emergence of single-unit object detectors whose activations are correlated with high-level concepts in discriminative and generative convolutional neural networks (CNN). \citet{Leino2018, Dhamdhere2018, Srinivas2019, Nam2020} introduce new attribution measures that quantify the influence of individual neurons.

\paragraph{Fairness of Deep Chest X-ray Classifiers}
Recently, researchers have scrutinised the fairness of deep classifiers trained on well-known and publicly available chest X-ray datasets \citep{mimic-cxr,Wang2019,Irvin2019}. 
\citet{Larrazabal2020} reported a consistently lower AUROC for underrepresented genders on imbalanced datasets. 
In a multi-centre setting, \citet{Zech2018} observed that the performance of chest X-ray classifiers was significantly lower on held-out external data, indicating possible bias due to confounding. 
Underdiagnosis and TPR disparity were evaluated by \citet{SeyyedKalantari2020, SeyyedKalantari2021} across three large chest X-ray datasets, showing higher underdiagnosis and lower TPRs in underserved patient populations.

\section{Methods\label{sec:methods}}

We introduce novel intra-processing approaches to debiasing classifiers w.r.t. the SPD and EOD, which build on the work by \citet{Savani2020}, who have proposed the intra-processing setting. For an extended comparison with the related works, see \Secref{sec:discussion}. Our techniques are tailored towards differentiable classifiers and neural networks in particular.

\subsection{Classification Parity Proxies}
\label{subsec:class_parity_proxies}

The proposed methods focus on the minimisation of classification disparity.
In particular, we minimise the SPD or EOD \emph{directly} without a need for adversarial training using differentiable proxy functions. Given sets of $N$ data points $\mathcal{X}=\left\{\vx_i\right\}_{i=1}^N$, $\mathcal{Y}=\left\{y_i\right\}_{i=1}^N$, and $\mathcal{A}=\left\{a_i\right\}_{i=1}^N$, the proxy $\tilde{\mu}$ for the SPD is given by
\begin{align}
    \tilde{\mu}_{\mathrm{SPD}}\left(f_{\boldsymbol{\theta}}(\cdot),\,\mathcal{X},\,\mathcal{Y},\,\mathcal{A}\right)=\frac{\sum_{i=1}^Nf_{\boldsymbol{\theta}}\left(\vx_i\right)\left(1-a_i\right)}{\sum_{i=1}^N1-a_i} - \frac{\sum_{i=1}^Nf_{\boldsymbol{\theta}}\left(\vx_i\right)a_i}{\sum_{i=1}^Na_i},
    \label{eqn:spd_}
\end{align}
and, for the EOD, we have
\begin{align}
    \tilde{\mu}_{\mathrm{EOD}}\left(f_{\boldsymbol{\theta}}(\cdot),\,\mathcal{X},\,\mathcal{Y},\,\mathcal{A}\right)=\frac{\sum_{i=1}^Nf_{\boldsymbol{\theta}}\left(\vx_i\right)\left(1-a_i\right)y_i}{\sum_{i=1}^N\left(1-a_i\right)y_i} - \frac{\sum_{i=1}^Nf_{\boldsymbol{\theta}}\left(\vx_i\right)a_iy_i}{\sum_{i=1}^Na_iy_i}.
    \label{eqn:eod_}
\end{align}
Notably, Equations~\ref{eqn:spd_} and \ref{eqn:eod_} are similar to the objective functions considered by \citet{Zafar2017,Zafar2019} in the context of fair logistic regression and SVM models. In Appendix~\ref{app:boundary}, we show that \Eqref{eqn:spd_} corresponds to the empirical estimate $\widehat{\mathrm{Cov}}\left(A,\,f_{\boldsymbol{\theta}}\left(\mX\right)\right)$. Similarly, \Eqref{eqn:eod_} corresponds to the empirical estimate of the conditional covariance $\widehat{\mathrm{Cov}}\left(A,\,f_{\boldsymbol{\theta}}\left(\mX\right)\,\middle|\,Y=1\right)$. The intuition behind the algorithms described in Sections~\ref{sec:pruning} and \ref{sec:biasgda} is to fine-tune the given biased neural network and minimise these proxies, thus, reducing the covariance between the protected attribute and decision boundary. The general debiasing procedure is schematically summarised in \Figref{fig:pipeline}: an already-trained network is debiased on held-out validation data, using the classification parity proxies, and can produce unbiased predictions without the protected attribute at test time.

\begin{figure}[H]
    \centering
    \includegraphics[width=0.6\linewidth]{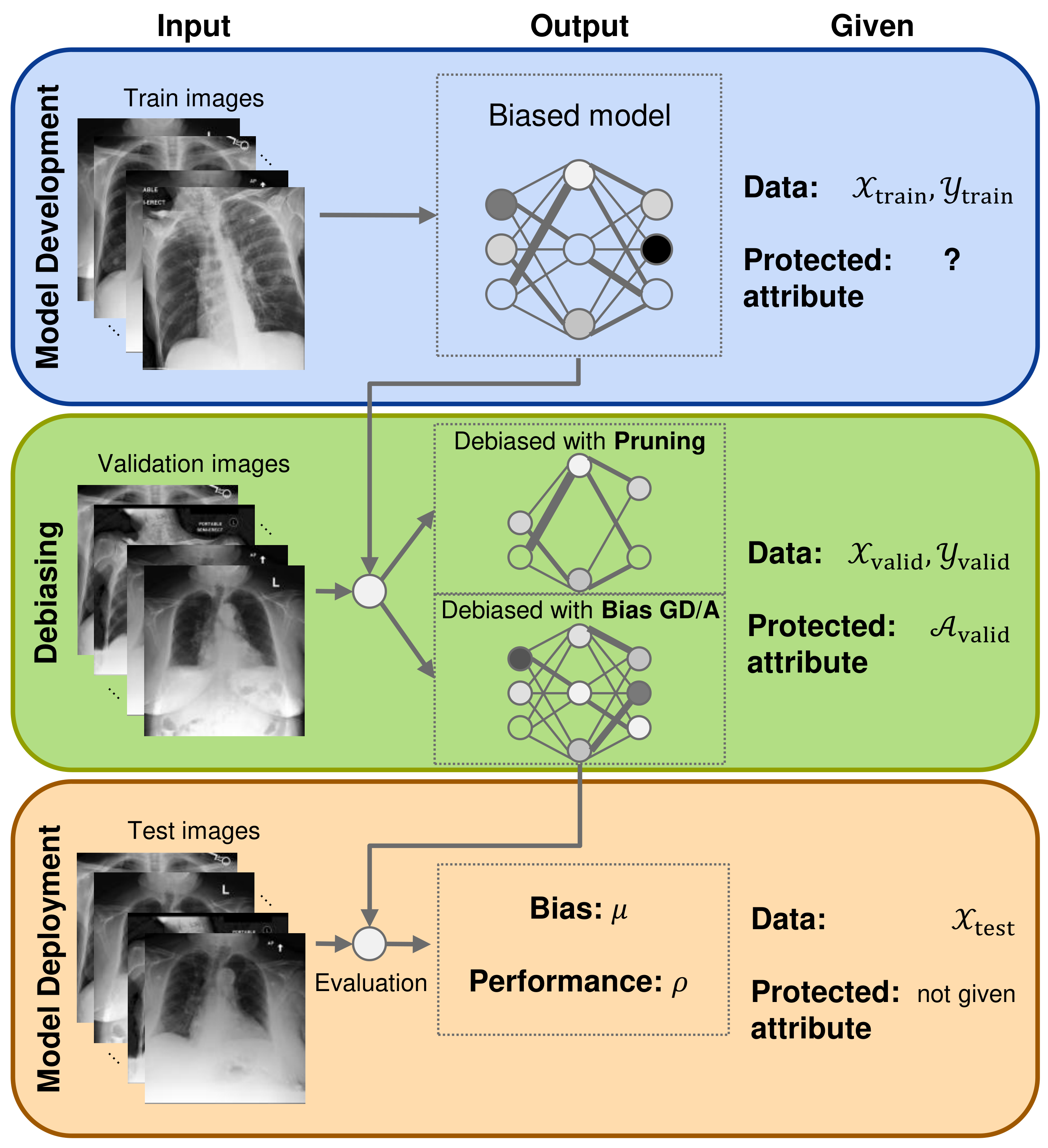}
    \caption{A summary of the debiasing procedure. \mbox{\textbf{\color{WorkflowBlue}(\emph{i})} A} biased model is trained without the knowledge of the protected attribute. \mbox{\textbf{\color{WorkflowGreen}(\emph{ii})} Using} the differentiable classification parity proxies, the model is debiased by performing pruning or bias gradient descent/ascent on validation data. \mbox{\textbf{\color{WorkflowOrange}(\emph{iii})} The} debiased model is evaluated on test data and can produce unbiased predictions \emph{without} the protected attribute.} 
    \label{fig:pipeline}
\end{figure}

\subsection{Neural Network Pruning for Debiasing\label{sec:pruning}}

Pruning refers to the procedure of reducing the effective number of parameters in a model. There has been renewed interest in neural network pruning \citep{Cheng2017,Blalock2020}, mainly for compressing models and reducing computational complexity and energy consumption. Different from the existing work, we propose using pruning to mitigate bias in neural network classifiers. In particular, we introduce a procedure for pruning individual units, or neurons, based on their contributions to classification disparity. In fully connected (FC) layers, a unit is a single component of the (pre-)activation vector; in convolutional layers, it is a component of the three-dimensional tensor. Below we use a one-dimensional index $j$ to enumerate units for both FC and convolutional layers.

\paragraph{Gradient-based Bias Influence}
Building on the influence-directed explanations proposed by \citet{Leino2018} for measuring the influence of individual neurons in CNNs, we propose a gradient-based statistic for quantifying the influence of units on the classification disparity. For a differentiable bias measure $\tilde{\mu}$, e.g. \Eqref{eqn:spd_} or \ref{eqn:eod_}, the influence of the $j$-th unit in the $l$-th layer is given by
\begin{align}
    S_{l,j}=\frac{1}{N}\sum_{i=1}^N\frac{\partial\tilde{\mu}\left(f_{\boldsymbol{\theta}}(\cdot),\mathcal{X},\mathcal{Y},\mathcal{A}\right)}{\partial z^l_j(\vx_i)},
    \label{eqn:saliency_grad}
\end{align}
where $z^l_j(\vx_i)$ denotes the unit's pre-activation at input $\vx_i$. In practice, partial derivatives such as those above can be computed efficiently using automatic differentiation, e.g. \mbox{PyTorch's} \mbox{Autograd} module \citep{Paszke2017}. The measure in \Eqref{eqn:saliency_grad} can identify the most influential units that need to be pruned. Empirical results in \Secref{sec:results} suggest that removing influential units effectively reduces the bias.

\paragraph{Pruning Procedure}
Algorithm~\ref{alg:fair_pruning} outlines the proposed pruning procedure comprising a few simple steps. \mbox{(\emph{i}) For} layer $1\leq l\leq L$, the influence $S_{l,j}$ (see \Eqref{eqn:saliency_grad}) is evaluated for each unit $j$ on the validation data $\mathcal{D}_{\mathrm{valid}}$. The memory complexity can be reduced by evaluating and averaging the influence across mini-batches rather than the entire validation set. For similar reasons, one may choose to prune only specific layers selectively rather than all $L$ intermediate layers. \mbox{(\emph{ii}) A} specified number, determined by the number of steps $B\geq1$, of the most influential units are pruned. For FC layers, a unit can be pruned by setting all outgoing weights to $0$'s, e.g. for the $j$-th unit in the $l$-th layer; this amounts to the assignment $\mW^{l}_{j,\cdot}\leftarrow\boldsymbol{0}$. For convolutional layers, we implement a dropout-like binary mask applied to pre-activations \citep{Srivastava2014}. Furthermore, note that the order in which units are pruned depends on the sign of the initial model's bias, i.e. whether the bias needs to be driven down or up towards 0. \mbox{(\emph{iii}) The} bias $\mu(\cdot)$ (see Equations~\ref{eqn:spd} and \ref{eqn:eod}) and performance $\rho(\cdot)$ of the pruned network are evaluated on the validation set. \mbox{(\emph{iv}) Influence} $S_{l,j}$ is recomputed for the pruned network, and steps \mbox{(\emph{ii})--(\emph{iv})} are repeated. In the end, an optimal sparsity level is chosen by returning a pruned network with the lowest bias and the performance at least $\varrho>0$, a hyperparameter determined by the user-specified constraint. 

\begin{algorithm}[t!]
    \SetAlgoLined
    \SetKwInOut{Output}{Output}
    
    \KwIn{Held-out validation set $\mathcal{D}_{\mathrm{valid}}=\left\{\left(\vx_i,y_i,a_i\right)\right\}_{i=1}^{N_{\mathrm{valid}}}$; neural network $f_{\boldsymbol{\theta}}(\cdot)$ with parameters $\boldsymbol{\theta}$ and $L$ intermediate layers; classification threshold $t\in[0,\,1]$; predictive performance measure $\rho(\cdot)$; bias measure $\mu(\cdot)$; differentiable bias proxy $\tilde{\mu}(\cdot)$; lower bound on performance $\varrho>0$; number of steps $B\geq1$}
    
    \Output{Pruned and debiased network $f_{\tilde{\boldsymbol{\theta}}}(\cdot)$ with parameters $\tilde{\boldsymbol{\theta}}$}
    
    $\mu_0\leftarrow\mu\left(\boldsymbol{1}_{\left\{f_{\boldsymbol{\theta}}(\cdot)\geq t\right\}},\left\{\vx_i\right\}_{i=1}^{N_{\mathrm{valid}}},\left\{y_i\right\}_{i=1}^{N_{\mathrm{valid}}}, \left\{a_i\right\}_{i=1}^{N_{\mathrm{valid}}}\right)$, where $(\vx_i,y_i,a_i)\in\mathcal{D}_{\mathrm{valid}}$
    
    Initialise $\tilde{\boldsymbol{\theta}}\leftarrow\boldsymbol{\theta}$
    
    Given $\mathcal{D}_{\mathrm{valid}}$ and $\tilde{\mu}(\cdot)$, evaluate $S_{l,j}$ (see Equation~\ref{eqn:saliency_grad}) for every unit $j$ in layer $1\leq l\leq L$ 
    
    \For{$b=0$ to $B-1$} {
        
        Let $\tau_{b}\leftarrow q_{1-1/B}\left(\left\{\mathrm{sgn}\left(\mu_0\right)S_{l,j}\right\}\right)$, where $q_{\alpha}(\cdot)$ denotes the empirical $\alpha$-quantile
        
        Prune unit $j$ in layer $1\leq l\leq L$ if $\mathrm{sgn}\left(\mu_0\right)S_{l,j}>\tau_{b}$ and adjust $\tilde{\boldsymbol{\theta}}$ accordingly
        
        $\tilde{t}\leftarrow\arg\max_{t'\in[0,\,1]}\rho\left(\boldsymbol{1}_{\left\{f_{\tilde{\boldsymbol{\theta}}}(\cdot),\geq t'\right\}},\left\{\vx_i\right\}_{i=1}^{N_{\mathrm{valid}}},\left\{y_i\right\}_{i=1}^{N_{\mathrm{valid}}}\right)$
        
        $\mu_b\leftarrow\mu\left(\boldsymbol{1}_{\left\{f_{\tilde{\boldsymbol{\theta}}}(\cdot)\geq \tilde{t}\right\}},\left\{\vx_i\right\}_{i=1}^{N_{\mathrm{valid}}},\left\{y_i\right\}_{i=1}^{N_{\mathrm{valid}}}, \left\{a_i\right\}_{i=1}^{N_{\mathrm{valid}}}\right)$
        
        $\rho_b\leftarrow\rho\left(\boldsymbol{1}_{\left\{f_{\tilde{\boldsymbol{\theta}}}(\cdot)\geq \tilde{t}\right\}},\left\{\vx_i\right\}_{i=1}^{N_{\mathrm{valid}}},\left\{y_i\right\}_{i=1}^{N_{\mathrm{valid}}}\right)$
        
        Reevaluate $S_{l,j}$ for the pruned network $f_{\tilde{\boldsymbol{\theta}}}(\cdot)$
        
        $\tilde{\boldsymbol{\theta}}_b\leftarrow\tilde{\boldsymbol{\theta}}$
    }
    
    $b^{*}\leftarrow\arg\min_{\substack{0\leq b\leq B-1\\ \rho_b\geq\varrho}}\left|\mu_b\right|$
    
    \Return $f_{\tilde{\boldsymbol{\theta}}_{b^*}}(\cdot)$
    \caption{Pruning procedure for debiasing neural networks. Individual units are removed greedily based on their influence, and a network with minimal bias subject to the specified performance constraint is returned.  \label{alg:fair_pruning}}
\end{algorithm}

\vspace{0.25cm}

The procedure above greedily removes individual units in the intermediate layers of the neural network step-by-step based on the criterion given by \Eqref{eqn:saliency_grad}. It returns a pruned network with minimal bias subject to the performance constraint. We will see that, in practice, it often allows making few changes to the classifier without retraining from scratch and sacrificing the predictive performance while reducing the classification disparity.

\subsection{Bias Gradient Descent/Ascent\label{sec:biasgda}}

Since the proxies given by Equations~\ref{eqn:spd_} and \ref{eqn:eod_} are differentiable w.r.t. $\boldsymbol{\theta}$, one could reduce the bias directly using gradient descent or ascent, depending on the sign of the bias. Therefore, another approach we propose is fine-tuning the classifier $f_{\boldsymbol{\theta}}(\cdot)$ for a few epochs with a small learning rate, for instance, using mini-batch gradient descent and \Eqref{eqn:spd_} or \ref{eqn:eod_} as a loss function. Algorithm~\ref{alg:biasGD} contains the pseudocode for the bias gradient descent/ascent (GD/A) procedure. This method is at first glance similar to the adversarial debiasing by \citet{Zhang2018}, who apply a discriminator to the network's output. However, we perform gradient descent/ascent on the differentiable bias proxies \emph{after} the network has been trained and do not require knowledge of the protected attribute during training.

Similar to Algorithm~\ref{alg:fair_pruning}, the weight update direction in bias GD/A depends on the sign of the initial bias. Likewise, at the end of the algorithm, a fine-tuned debiased network is returned, minimising the bias with the performance of at least $\varrho$. This procedure has several additional hyperparameters, namely, learning rate $\eta>0$, which, in practice, should be chosen sufficiently small, mini-batch size $M\geq1$, and a maximum number of fine-tuning epochs $E\geq1$. In our experiments, we observed that, compared to the training of the original model, relatively few fine-tuning epochs suffice to reduce the bias. Although Algorithm~\ref{alg:biasGD} is based on the mini-batch gradient descent, other optimisation procedures can be adopted, e.g. batch gradient descent, as long as the procedure supports the evaluation of the differentiable proxy $\tilde{\mu}(\cdot)$ on several data points.

\begin{algorithm}[t!]
    \SetAlgoLined
    \SetKwInOut{Output}{Output}
    
    \KwIn{Held-out validation set $\mathcal{D}_{\mathrm{valid}}=\left\{\left(\vx_i,y_i,a_i\right)\right\}_{i=1}^{N_{\mathrm{valid}}}$; neural network $f_{\boldsymbol{\theta}}(\cdot)$ with parameters $\boldsymbol{\theta}$; classification threshold $t\in[0,\,1]$; predictive performance measure $\rho(\cdot)$; bias measure $\mu(\cdot)$; differentiable bias proxy $\tilde{\mu}(\cdot)$; lower bound on performance $\varrho>0$; learning rate $\eta>0$; number of epochs $E\geq1$; mini-batch size $M\geq1$}
    
    \Output{\mbox{Fine-tuned and debiased network $f_{\tilde{\boldsymbol{\theta}}}(\cdot)$ with parameters $\tilde{\boldsymbol{\theta}}$}}
    
    $\mu_0\leftarrow\mu\left(\boldsymbol{1}_{\left\{f_{\boldsymbol{\theta}}(\cdot)\geq t\right\}},\left\{\vx_i\right\}_{i=1}^{N_{\mathrm{valid}}},\left\{y_i\right\}_{i=1}^{N_{\mathrm{valid}}}, \left\{a_i\right\}_{i=1}^{N_{\mathrm{valid}}}\right)$, where $(\vx_i,y_i,a_i)\in\mathcal{D}_{\mathrm{valid}}$
    
    Initialise $\tilde{\boldsymbol{\theta}}\leftarrow\boldsymbol{\theta}$
    
    \For{$e=0$ to $E-1$} {
        
        Draw mini-batch $\mathcal{B}=\left\{\left(\vx_i,y_i,a_i\right)\right\}_{i=1}^M$ without replacement, s.t. $\mathcal{B}\subseteq \mathcal{D}_{\mathrm{valid}}$
        
        $\tilde{\mu}_e\leftarrow\tilde{\mu}\left(f_{\tilde{\boldsymbol{\theta}}}(\cdot), \left\{\vx_i\right\}_{i=1}^M, \left\{y_i\right\}_{i=1}^M, \left\{a_i\right\}_{i=1}^M\right)$, where $(\vx_i,y_i,a_i)\in\mathcal{B}$
        
        $\tilde{\boldsymbol{\theta}}\leftarrow\tilde{\boldsymbol{\theta}}-\mathrm{sgn}\left(\mu_0\right)\eta\nabla_{\tilde{\boldsymbol{\theta}}}\tilde{\mu}_e$
        
        $\tilde{t}\leftarrow\arg\max_{t'\in[0,\,1]}\rho\left(\boldsymbol{1}_{\left\{f_{\tilde{\boldsymbol{\theta}}}(\cdot),\geq t'\right\}},\left\{\vx_i\right\}_{i=1}^{N_{\mathrm{valid}}},\left\{y_i\right\}_{i=1}^{N_{\mathrm{valid}}}\right)$
        
        $\mu_e\leftarrow\mu\left(\boldsymbol{1}_{\left\{f_{\tilde{\boldsymbol{\theta}}}(\cdot)\geq \tilde{t}\right\}},\left\{\vx_i\right\}_{i=1}^{N_{\mathrm{valid}}},\left\{y_i\right\}_{i=1}^{N_{\mathrm{valid}}}, \left\{a_i\right\}_{i=1}^{N_{\mathrm{valid}}}\right)$
        
        $\rho_e\leftarrow\rho\left(\boldsymbol{1}_{\left\{f_{\tilde{\boldsymbol{\theta}}}(\cdot)\geq \tilde{t}\right\}},\left\{\vx_i\right\}_{i=1}^{N_{\mathrm{valid}}},\left\{y_i\right\}_{i=1}^{N_{\mathrm{valid}}}\right)$
        
        $\tilde{\boldsymbol{\theta}}_e\leftarrow\tilde{\boldsymbol{\theta}}$
    }
    
    $e^{*}\leftarrow\arg\min_{\substack{0\leq e\leq E-1\\ \rho_e\geq\varrho}}\left|\mu_e\right|$
    
    \Return $f_{\tilde{\boldsymbol{\theta}}_{e^*}}(\cdot)$
    \caption{Bias gradient descent/ascent procedure for debiasing neural networks. A biased classifier is fine-tuned by performing bias gradient descent or ascent on a differentiable bias proxy function. In the end, a network with minimal bias subject to the specified performance constraint is returned.\label{alg:biasGD}}
\end{algorithm}

\section{Experimental Setup\label{sec:expsetup}}

The purpose of our experiments was twofold: \mbox{(\emph{i}) test} the proposed pruning and bias GD/A methods on tabular and image data for FC and CNN architectures and \mbox{(\emph{ii}) explore} the use of intra- and post-processing to mitigate biases in deep chest X-ray classifiers. Below, we briefly summarise the datasets, pre-processing, compared techniques, and the evaluation procedure. Further implementation details can be found in Appendix~\ref{app:imp}.

\subsection{Datasets}

We compared debiasing techniques on tabular and image data (see Table~\ref{tab:datasets} in Appendix~\ref{app:data}). Tabular datasets include several publicly available benchmarks, most of them part of IBM AIF 360 toolkit \citep{Bellamy2018}. We refer the reader to \citet{Quy2021} for a thorough exploratory analysis of these datatsets. Furthermore, we applied debiasing to CNNs trained on the large-scale chest X-ray dataset -- MIMIC-CXR \citep{mimic-cxr}. In addition, we performed experiments on synthetic data (see Appendix~\ref{app:synth}).

\paragraph{Adult} The Adult Census Income data contains 48,842 instances and includes seven categorical, two binary, and six numerical features. The task is to predict whether a person's annual income exceeds 50,000\$ \citep{Kohavi1996, Quy2021}. In our experiments, we focused on the protected attribute ``\emph{sex}''. Note that here and below, we use the term ``\emph{sex}'' to match the reported terminology in the underlying data.

\paragraph{Bank} The dataset was collected during phone call marketing campaigns \citep{Moro2014, Quy2021} and comprises 45,211 samples with six categorical, four binary, and seven numerical features. The classification task is to predict a deposit subscription by a potential client. We used ``\emph{age}'' as the protected variable.

\paragraph{COMPAS} The Correctional Offender Management Profiling for Alternative Sanctions dataset \citep{ProPublica,Quy2021} includes 7,214 samples with 31 categorical, 6 binary, and 14 numerical covariates. The underlying classification problem is predicting the risk of recidivism. The protected attribute is ``\emph{race}''.

\paragraph{MIMIC-III} 

Medical Information Mart for Intensive Care \mbox{(MIMIC-III-v1.4)} database consists of information on the admissions of patients who stayed in critical care units at a large tertiary care hospital \citep{Johnson2016}. It includes demographics, vital sign measurements, laboratory results, medications, notes, imaging reports, mortality rates, etc. We used pre-processing routine provided by \citet{benchmarkingPurushotham2018} that retains only the first admissions of adult patients ($>$ 15 years). Pre-processed data consist of 17 features from the SAPS-II score. We averaged time-series data for each feature/admission. Our experiments to predict in-hospital mortality focused on the ``\emph{age}'', ``\emph{marital status}'', and ``\emph{insurance type}'' as  protected attributes. For ``\emph{age}'', we grouped subjects $\geq 78$ years old into one category and the rest into another. For ``\emph{marital status}'', the two groups comprised \textit{single} and the rest. ``\emph{Insurance type}'' was dichotomised by grouping \textit{Medicare} and \textit{Medicaid} into one category (\emph{public} health insurance) and the rest into another (\emph{private}), similarly to \citet{mimicifMeng2021}. Not all protected attribute groupings are clinically meaningful, and debiasing might not be relevant in all cases. For instance, insurance type dichotomisation may be too simplistic since Medicare and Medicaid are distinctively different programs \citep{Altman2015} and should be treated separately in practice. However, we focus on binary-valued protected attributes and contextualise our analysis in the previous work by \citet{mimicifMeng2021}. In a similar vein, debiasing w.r.t. the SPD with ``\emph{age}'' as the protected attribute may not be clinically relevant; however, we included these results for completeness. 

\paragraph{MIMIC-CXR}

MIMIC-CXR is a large dataset of chest X-rays from 227,835 studies of 65,379 patients \citep{mimic-cxr}. Each study contains one or more images, usually frontal and lateral views. We only used frontal view images, resizing them to 224$\times$224 px. We focused on ``\emph{sex}'' and ``\emph{ethnicity}'' as protected variables since the groups of these attributes were previously shown to have disparate classification outcomes  \citep{SeyyedKalantari2020,SeyyedKalantari2021}. 
For each image, one or more labels are reported, comprising 14 binary attributes. For the protected attribute ``\textit{sex}'', \textit{male} patients formed the privileged and \textit{female} patients the unprivileged group. Since the classifiers trained using the following combinations of disease labels, protected attributes, and privileged/unprivileged groups were shown to have disparate TPRs \citep{SeyyedKalantari2020}, we took ``\textit{enlarged cardiomediastinum}'' (enlarged CM) as the classification label for the attribute ``\textit{sex}''. For ``\textit{ethnicity}'', \textit{white} patients were taken as the privileged, whereas patients with \textit{Hispanic/Latino} ethnicity as unprivileged group. For this attribute, we chose ``\textit{pneumonia}'' as the classification label.
Studies with \emph{no findings} were used as the negative class in both cases.

\subsection{Debiasing Methods}

In addition to the proposed pruning and bias GD/A procedures, we applied several other debiasing methods, focusing on intra- and post-processing approaches. \textsc{Standard} refers to the original, potentially biased classifier $f_{\boldsymbol{\theta}}(\cdot)$ with the classification threshold $t\in[0,\,1]$ chosen to maximise the balanced accuracy on the held-out validation data. We used the random perturbation procedure (\textsc{Random}) described by \citet{Savani2020} as a baseline. This method perturbs the parameters of the original network $f_{\boldsymbol{\theta}}\left(\cdot\right)$ several times by multiplicative Gaussian noise, distributed as $\mathcal{N}\left(1,\,0.01\right)$. The procedure returns a perturbed network maximising the bias-constrained objective proposed by \citet{Savani2020} on the validation set. \textsc{ROC} refers to the reject option classification post-processing algorithm \citep{Kamiran2012} that swaps classification outcomes for the subjects from the underprivileged group who fall within the confidence band around the decision boundary. \textsc{Eq. Odds} is the equalised odds post-processing method \citep{Hardt2016}. This algorithm adjusts output labels probabilistically to balance the odds across the protected attribute categories. Lastly, we considered adversarial fine-tuning (\textsc{Adv. Intra}) \citep{Savani2020}, an intra-processing technique closely related to ours that fine-tunes the biased classifier via adversarial training. In Appendix~\ref{app:zhang}, we also compare with the adversarial in-processing algorithm by \citet{Zhang2018}.

\subsection{Classification Models and Debiasing Evaluation}

For tabular datasets, we used the same FC architecture and training scheme for the classifier $f_{\boldsymbol{\theta}}(\cdot)$ (see Appendix~\ref{app:imp}), following the experimental setup of \citet{Savani2020}. For MIMIC-CXR, we used the VGG-16 \citep{Simonyan2014} and ResNet-18 \citep{He2016} CNN architectures, initialising them with pre-trained weights. All models were trained by minimising the binary cross-entropy loss using the Adam optimiser \citep{Kingma2015}. For chest X-ray classifiers, to avoid overfitting, we applied random augmentations during training, such as centre crop, horizontal flip, translation, and rotation.

For all compared techniques, debiasing was performed only on the validation set (see Appendix~\ref{app:imp}). The classifiers were trained and debiased repeatedly on the independent replicates of the train-validation-test split in the manner of Monte Carlo cross-validation. Classifiers were evaluated on the test data w.r.t. the bias and performance. We used balanced accuracy to reflect true positive and negative rates equally. For tabular data, the bias was evaluated in terms of the SPD and EOD. For MIMIC-CXR, we focused on the EOD rather than SPD since achieving even positive prediction outcomes across the groups of the protected attributes may not be clinically relevant.

\section{Results\label{sec:results}}

In this section, we provide the results of the empirical comparison among several debiasing techniques, including the proposed pruning and bias GD/A intra-processing algorithms. Further results and additional experiments on synthetic data, investigating the stability of pruning and bias GD/A, are discussed in Appendix~\ref{app:results}.

\subsection{Results on Tabular Benchmarks\label{sec:results_tabular}}

Tables~\ref{tab:tabular_results} and \ref{tab:mimiciii_results} contain quantitative results obtained on tabular data: EOD, SPD, and BA before and after debiasing. Compared to other post- and intra-processing techniques, pruning and bias GD/A successfully mitigate biases and tend to sacrifice less accuracy on most datasets. On average, pruning performs slightly worse than GD/A and has larger variability across seeds.
The results for the adversarial fine-tuning are in line with those reported in the original paper by \citet{Savani2020}. Generally, while this method visibly reduces the bias, it tends to sacrifice the BA more, likely due to minimising a loss function different from that of the bias GD/A. Interestingly, on Adult dataset, both of our procedures drastically reduce the BA of the classifier and perform worse than ROC: we attribute this to the general sensitivity of intra-processing methods \citep{Savani2020} to initial conditions. In Appendix~\ref{app:results}, we explore this phenomenon further on synthetic data. In brief, we observed that when the bias of the original classifier is high, proposed techniques reduce the accuracy considerably or fail to reduce the bias, therefore, it may be prudent to retrain the model from scratch using an in-processing approach or resort to post-processing in such cases.

In addition, we examined changes in the bias and balanced accuracy of the neural network throughout the process of pruning and bias GD/A. \Figref{fig:mimiciii_trajectories} shows the trajectories of the EOD, SPD, and BA obtained on \mbox{MIMIC-III} data for predicting in-hospital mortality with the ``\emph{insurance type}'' as the protected attribute. Encouragingly, both methods drive the classification disparity towards zero while not affecting the balanced accuracy of the classifier significantly. We observed that few pruning steps or fine-tuning epochs, compared to the training time of the original model (a maximum of 1,000 epochs), were necessary to reduce the bias, suggesting the viability of fine-tuning an already-trained biased model on the validation set. Generally, the bias and BA trajectories for pruning featured slightly higher variance across seeds. An intuitive explanation could be that pruning, compared to GD/A, explores a relatively limited number of debiased network weight configurations, particularly for smaller architectures, such as the one in our tabular experiments (see Table~\ref{tab:arch}). We observed similar debiasing dynamics on other tabular benchmarks (see \Figref{fig:trajectories_full}).

\begin{table}[h]
    \centering
    \caption{\mbox{Bias (\emph{a})} and balanced \mbox{accuracy (\emph{b})} attained before and after debiasing neural networks trained on nonclinical tabular data. If necessary, debiasing was run twice for each dataset: for the SPD and EOD separately. The results are reported as averages followed by standard deviations across 20 train-validation-test splits. Best results are shown in \textbf{bold}, second-best -- in \textit{italic}, except for \textsc{Standard}. \label{tab:tabular_results}}
    
    \subtable[Bias]{
        \scriptsize
        \begin{tabular}{p{1.8cm}lp{2.1cm}p{2.1cm}p{2.1cm}}
            \toprule
            \makecell[l]{\textbf{Bias}\\ \textbf{Measure}} & \textbf{Method} & \textbf{\makecell[l]{Adult: \\ \textit{Sex}}} & \textbf{\makecell[l]{Bank: \\ \textit{Age}}} & \textbf{\makecell[l]{COMPAS: \\ \textit{Race}}} \\
            \toprule
            \multirow{7}{*}{\makecell[l]{\textbf{SPD}}} & \cellcolor{Gainsboro}\textsc{Standard} & \cellcolor{Gainsboro}\tentry{-0.32}{0.02} & \cellcolor{Gainsboro}\tentry{0.18}{0.04} & \cellcolor{Gainsboro}\tentry{{\;}0.19}{0.03} \\
            & \textsc{Random} & \tentry{-0.04}{0.01} & \textit{\tentry{0.03}{0.04}} & \tentry{{\;}0.09}{0.04} \\
            & \textsc{ROC} & \tentry{-0.04}{0.02} & \tentry{0.08}{0.04} & \textbf{\tentry{-0.01}{0.01}} \\
            & \textsc{Eq. Odds} & \tentry{-0.09}{0.01} & \tentry{0.06}{0.03} & \tentry{{\;}0.03}{0.06} \\
            & \textsc{Adv. Intra} &  \textit{\tentry{-0.03}{0.00}} &  \tentry{0.05}{0.03} &  \tentry{\;0.03}{0.03} \\
            & \underline{\textsc{Pruning}} & \tentry{-0.04}{0.05} & \textbf{\tentry{0.02}{0.04}} & \textit{\tentry{\;0.02}{0.03}} \\
            & \underline{\textsc{Bias GD/A}} & \textbf{\tentry{-0.01}{0.04}} & \tentry{0.04}{0.05} & \tentry{{\;}0.04}{0.04} \\
            \midrule
            \multirow{7}{*}{\makecell[l]{\textbf{EOD}}} & \cellcolor{Gainsboro}\textsc{Standard} & \cellcolor{Gainsboro}\tentry{-0.14}{0.02} & \cellcolor{Gainsboro}\tentry{0.01}{0.04} & \cellcolor{Gainsboro}\tentry{{\;}0.20}{0.05} \\
            & \textsc{Random} & \tentry{-0.07}{0.03} & \textit{\tentry{0.02}{0.04}} & \tentry{\;0.09}{0.04} \\
            & \textsc{ROC} & \tentry{-0.05}{0.03} & \tentry{0.04}{0.04} & \textbf{\tentry{-0.01}{0.01}} \\
            & \textsc{Eq. Odds} & \textbf{\tentry{-0.01}{0.04}} & \tentry{0.04}{0.10} & \textit{\tentry{{\;}0.03}{0.06}} \\
            & \textsc{Adv. Intra} &  \tentry{-0.09}{0.03} &  \tentry{0.03}{0.06} &  \tentry{\;0.14}{0.07} \\
            & \underline{\textsc{Pruning}} & \textbf{\tentry{-0.01}{0.03}} & \textbf{\tentry{0.00}{0.07}} & \tentry{\;0.04}{0.06} \\
            & \underline{\textsc{Bias GD/A}} & \textit{\tentry{-0.03}{0.03}} & \textit{\tentry{0.02}{0.06}} & \tentry{\;0.06}{0.06} \\
            \bottomrule
        \end{tabular}
    }
    
    \vspace{0.25cm}
    
    \subtable[Balanced accuracy]{
        \scriptsize
        \begin{tabular}{p{1.8cm}lp{2.1cm}p{2.1cm}p{2.1cm}}
            \toprule
            \makecell[l]{\textbf{Bias}\\ \textbf{Measure}} & \textbf{Method} & \textbf{\makecell[l]{Adult: \\ \textit{Sex}}} & \textbf{\makecell[l]{Bank: \\ \textit{Age}}} & \textbf{\makecell[l]{COMPAS: \\ \textit{Race}}} \\
            \toprule
            \multirow{7}{*}{\makecell[l]{\textbf{SPD}}} & \cellcolor{Gainsboro}\textsc{Standard} & \cellcolor{Gainsboro}\tentry{0.82}{0.01} & \cellcolor{Gainsboro}\tentry{0.86}{0.01} & \cellcolor{Gainsboro}\tentry{0.65}{0.01} \\
            & \textsc{Random} & \tentry{0.60}{0.01} & \tentry{0.60}{0.10} & \tentry{0.60}{0.03} \\
            & \textsc{ROC} & \textbf{\tentry{0.79}{0.01}} & \tentry{0.66}{0.10} & \tentry{0.50}{0.00} \\
            & \textsc{Eq. Odds} & \textit{\tentry{0.73}{0.02}} & \tentry{0.70}{0.02} & \tentry{0.60}{0.01} \\
            & \textsc{Adv. Intra} & \tentry{0.56}{0.01} & \tentry{0.61}{0.09} & \tentry{0.56}{0.04} \\
            & \underline{\textsc{Pruning}} & \tentry{0.56}{0.04} & \textit{\tentry{0.84}{0.01}} & \textit{\tentry{0.63}{0.02}} \\
            & \underline{\textsc{Bias GD/A}} & \tentry{0.66}{0.01} & \textbf{\tentry{0.86}{0.01}} & \textbf{\tentry{0.64}{0.01}} \\
            \midrule
            \multirow{7}{*}{\makecell[l]{\textbf{EOD}}} & \cellcolor{Gainsboro}\textsc{Standard} & \cellcolor{Gainsboro}\tentry{0.82}{0.01} & \cellcolor{Gainsboro}\tentry{0.86}{0.01} & \cellcolor{Gainsboro}\tentry{0.65}{0.01} \\
            & \textsc{Random} & \tentry{0.78}{0.03} & \textbf{\tentry{0.86}{0.01}} & \tentry{0.61}{0.03} \\
            & \textsc{ROC} & \textbf{\tentry{0.82}{0.01}} & \textbf{\tentry{0.86}{0.01}} & \tentry{0.50}{0.00} \\
            & \textsc{Eq. Odds} & \tentry{0.73}{0.02} & \tentry{0.70}{0.02} & \tentry{0.60}{0.01} \\
            & \textsc{Adv. Intra} & \textit{\tentry{0.78}{0.02}} & \textit{\tentry{0.84}{0.01}} & \tentry{0.61}{0.02} \\
            & \underline{\textsc{Pruning}} & \textit{\tentry{0.78}{0.02}} & \textbf{\tentry{0.86}{0.03}} & \textit{\tentry{0.62}{0.03}} \\
            & \underline{\textsc{Bias GD/A}} & \textbf{\tentry{0.82}{0.01}} & \textbf{\tentry{0.86}{0.01}} & \textbf{\tentry{0.64}{0.01}} \\
            \bottomrule
        \end{tabular}
    }%
\end{table}

\begin{table}[h]
    \centering
    \caption{Bias and balanced accuracy before and after debiasing neural networks trained on \mbox{MIMIC-III} with \emph{age}, \emph{marital status}, and \emph{insurance type} as protected attributes.\label{tab:mimiciii_results}}
    \scriptsize
    \begin{tabular}{p{1.1cm}p{1.6cm}p{1.6cm}p{1.4cm}p{1.6cm}p{1.4cm}p{1.6cm}p{1.4cm}}
        & & \multicolumn{2}{c}{\textbf{\textit{Age}}} & \multicolumn{2}{c}{\textbf{\textit{Marital Status}}} & \multicolumn{2}{c}{\textbf{\textit{Insurance Type}}}\\
        \toprule
        \makecell[l]{\textbf{Bias}\\ \textbf{Measure}} &\textbf{Method} & \textbf{Bias} & \textbf{BA}
        & \textbf{Bias} & \textbf{BA} & \textbf{Bias} & \textbf{BA} \\
        \toprule
        \multirow{7}{*}{\makecell[l]{\textbf{SPD}}} &\cellcolor{Gainsboro}\textsc{Standard} & \cellcolor{Gainsboro}\tentry{-0.28}{0.03} & \cellcolor{Gainsboro}\tentry{0.76}{0.01} & \cellcolor{Gainsboro}\tentry{\;0.10}{0.02} & \cellcolor{Gainsboro}\tentry{0.76}{0.01} & \cellcolor{Gainsboro}\tentry{-0.19}{0.03} & \cellcolor{Gainsboro}\tentry{0.75}{0.01} \\
        & \textsc{Random} & \tentry{-0.04}{0.01} & \tentry{0.64}{0.01} & \tentry{\;0.05}{0.01} & \tentry{0.72}{0.02} & \tentry{-0.04}{0.01} & \tentry{0.67}{0.01} \\
        & \textsc{ROC} & \tentry{-0.05}{0.01} & \tentry{0.63}{0.01} & \tentry{\;0.03}{0.03} & \textbf{\tentry{0.75}{0.01}} & \tentry{-0.05}{0.01} & \tentry{0.68}{0.01} \\
        & \textsc{Eq. Odds} & \textit{\tentry{-0.01}{0.01}} & \tentry{0.57}{0.02} & \textit{\tentry{\;0.01}{0.00}} & \tentry{0.57}{0.01} & \textit{\tentry{-0.01}{0.00}} & \tentry{0.57}{0.01} \\
        & \textsc{Adv. Intra} & \tentry{-0.04}{0.01} & \tentry{0.60}{0.02} & \tentry{\;0.04}{0.02} & \tentry{0.67}{0.04} & \tentry{-0.03}{0.01} & \tentry{0.64}{0.02} \\
        & \underline{\textsc{Pruning}} & \textbf{\tentry{\;0.00}{0.02}} & \textit{\tentry{0.69}{0.02}} & \textbf{\tentry{\;0.00}{0.02}} & \textit{\tentry{0.73}{0.01}} & \textbf{\tentry{\;0.00}{0.02}} & \textit{\tentry{0.69}{0.03}} \\
        & \underline{\textsc{Bias GD/A}} & \textit{\tentry{-0.01}{0.02}} & \textbf{\tentry{0.72}{0.01}} & \textit{\tentry{\;0.01}{0.02}} & \textbf{\tentry{0.75}{0.01}} & \textit{\tentry{-0.01}{0.02}} & \textbf{\tentry{0.73}{0.01}} \\
        \midrule
        \multirow{7}{*}{\makecell[l]{\textbf{EOD}}} &\cellcolor{Gainsboro}\textsc{Standard} & \cellcolor{Gainsboro}\tentry{-0.11}{0.04} & \cellcolor{Gainsboro}\tentry{0.76}{0.01} & \cellcolor{Gainsboro}\tentry{\;0.08}{0.03} & \cellcolor{Gainsboro}\tentry{0.76}{0.01} & \cellcolor{Gainsboro}\tentry{-0.05}{0.04} & \cellcolor{Gainsboro}\tentry{0.75}{0.01} \\
        & \textsc{Random} & \textit{\tentry{-0.05}{0.05}} & \tentry{0.72}{0.03} & \tentry{\;0.06}{0.04} & \textit{\tentry{0.74}{0.03}} & \tentry{-0.04}{0.04} & \textbf{\tentry{0.75}{0.01}} \\
        & \textsc{ROC} & \textit{\tentry{-0.05}{0.06}} & \tentry{0.69}{0.04} & \tentry{\;0.03}{0.05} & \textbf{\tentry{0.75}{0.01}} & \tentry{-0.04}{0.04} & \textbf{\tentry{0.75}{0.02}} \\
        & \textsc{Eq. Odds} & \textbf{\tentry{\;0.01}{0.04}} & \tentry{0.57}{0.02} & \textbf{\tentry{\;0.01}{0.04}} & \tentry{0.57}{0.01} & \textit{\tentry{\;0.01}{0.04}} & \tentry{0.57}{0.01} \\
        & \textsc{Adv. Intra} & \tentry{-0.08}{0.03} & \tentry{0.71}{0.02} & \tentry{\;0.06}{0.04} & \tentry{0.73}{0.01} & \tentry{-0.02}{0.03} & \tentry{0.72}{0.03} \\
        & \underline{\textsc{Pruning}} & \textbf{\tentry{\;0.01}{0.06}} & \textit{\tentry{0.73}{0.01}} & \textit{\tentry{-0.02}{0.06}} & \tentry{0.73}{0.02} & \textbf{\tentry{{\;}0.00}{0.04}} & \textit{\tentry{0.74}{0.01}} \\
        & \underline{\textsc{Bias GD/A}} & \textbf{\tentry{-0.01}{0.05}} & \textbf{\tentry{0.75}{0.01}} & \textit{\tentry{\;0.02}{0.05}} & \textbf{\tentry{0.75}{0.01}} & \textbf{\tentry{\;0.00}{0.04}} & \textbf{\tentry{0.75}{0.01}} \\
        \bottomrule
    \end{tabular}
\end{table}

\begin{figure}[t]
    \centering
    \subfigure[{\footnotesize\centering \textsc{Pruning}, SPD}]{
        \includegraphics[width=0.25\linewidth]{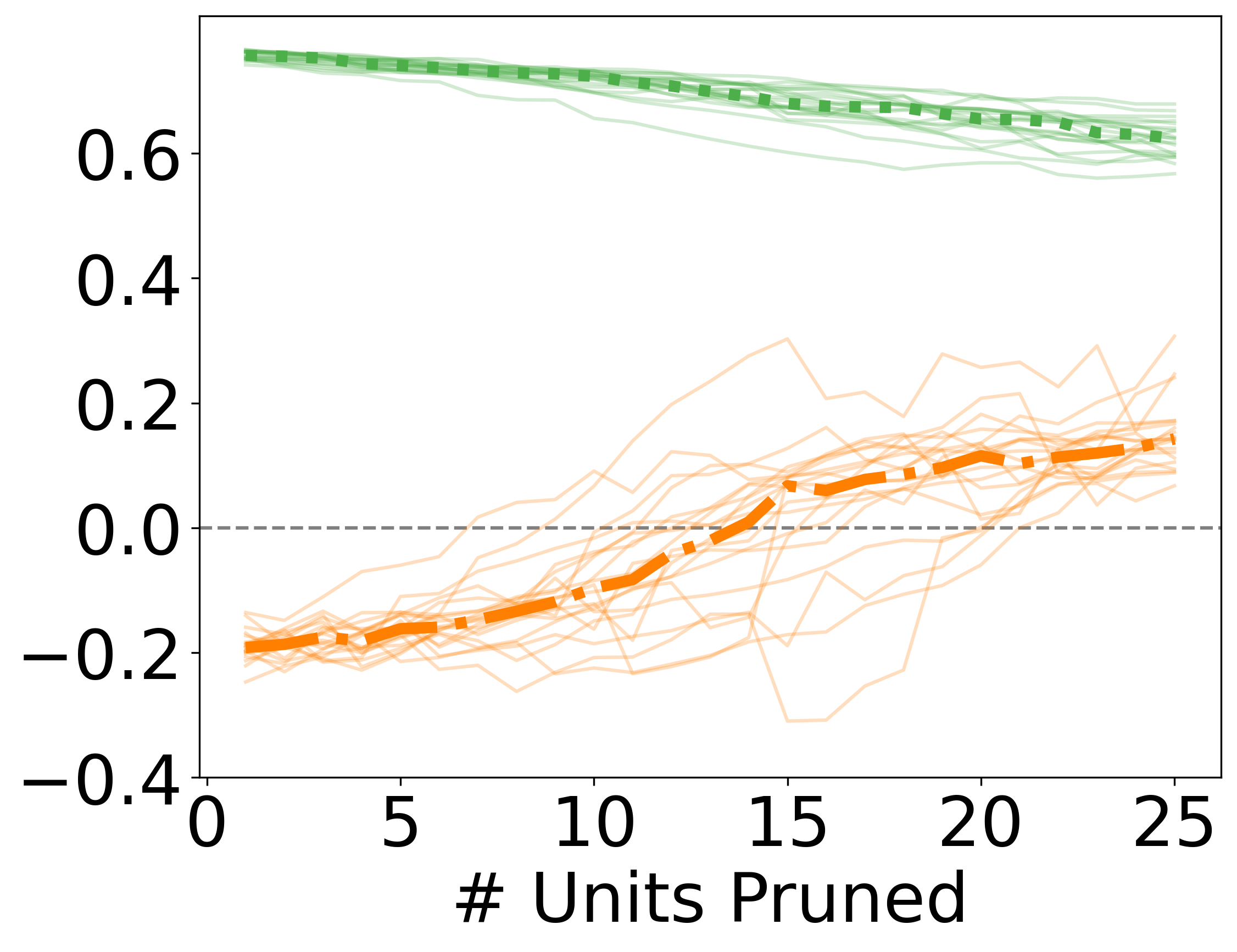}
    }
    \subfigure[{\footnotesize\centering \textsc{Pruning}, EOD}]{    
        \includegraphics[width=0.2125\linewidth]{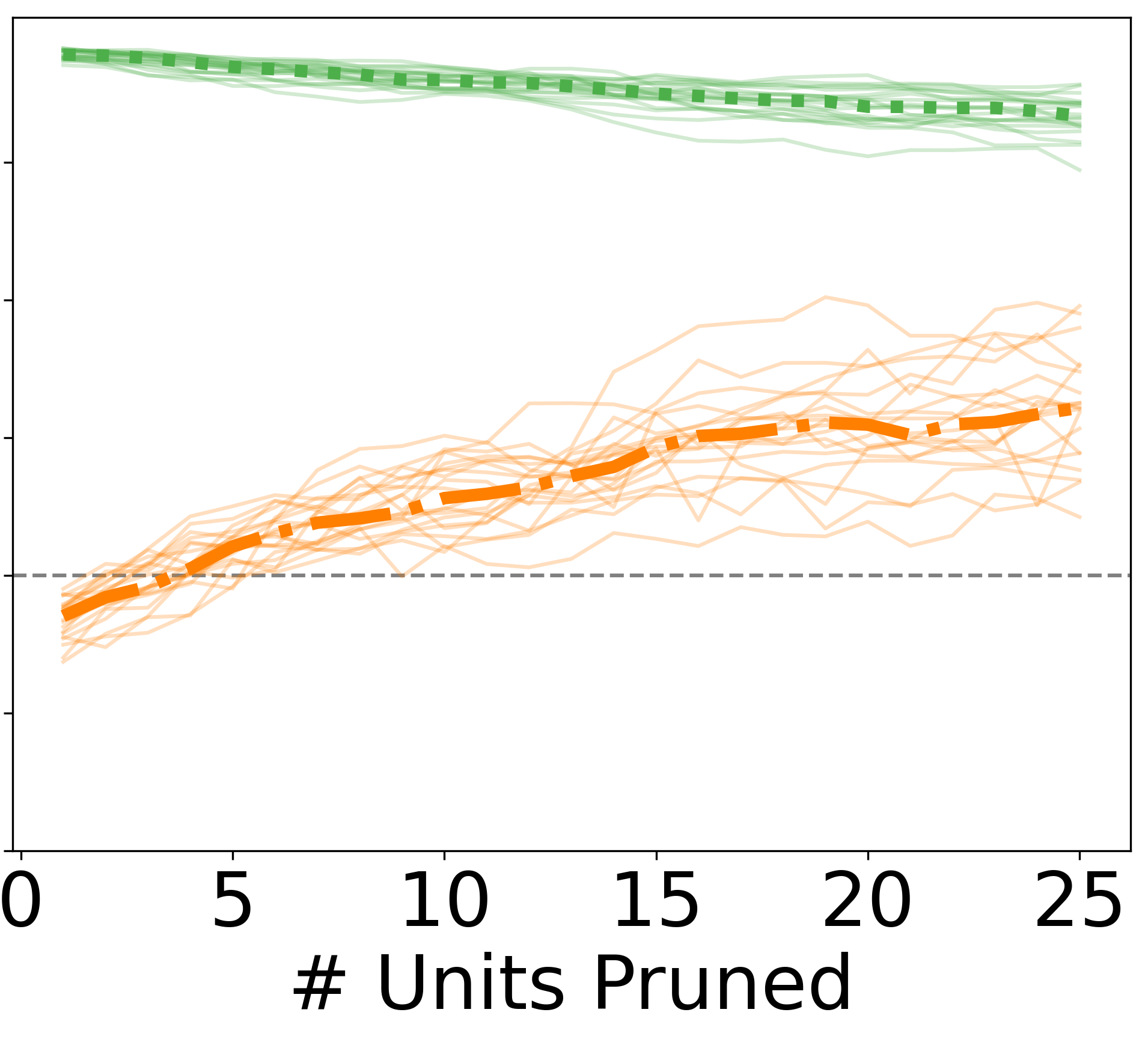}
    }
    \subfigure[{\footnotesize\centering \textsc{Bias GD/A}, SPD}]{
        \includegraphics[width=0.215\linewidth]{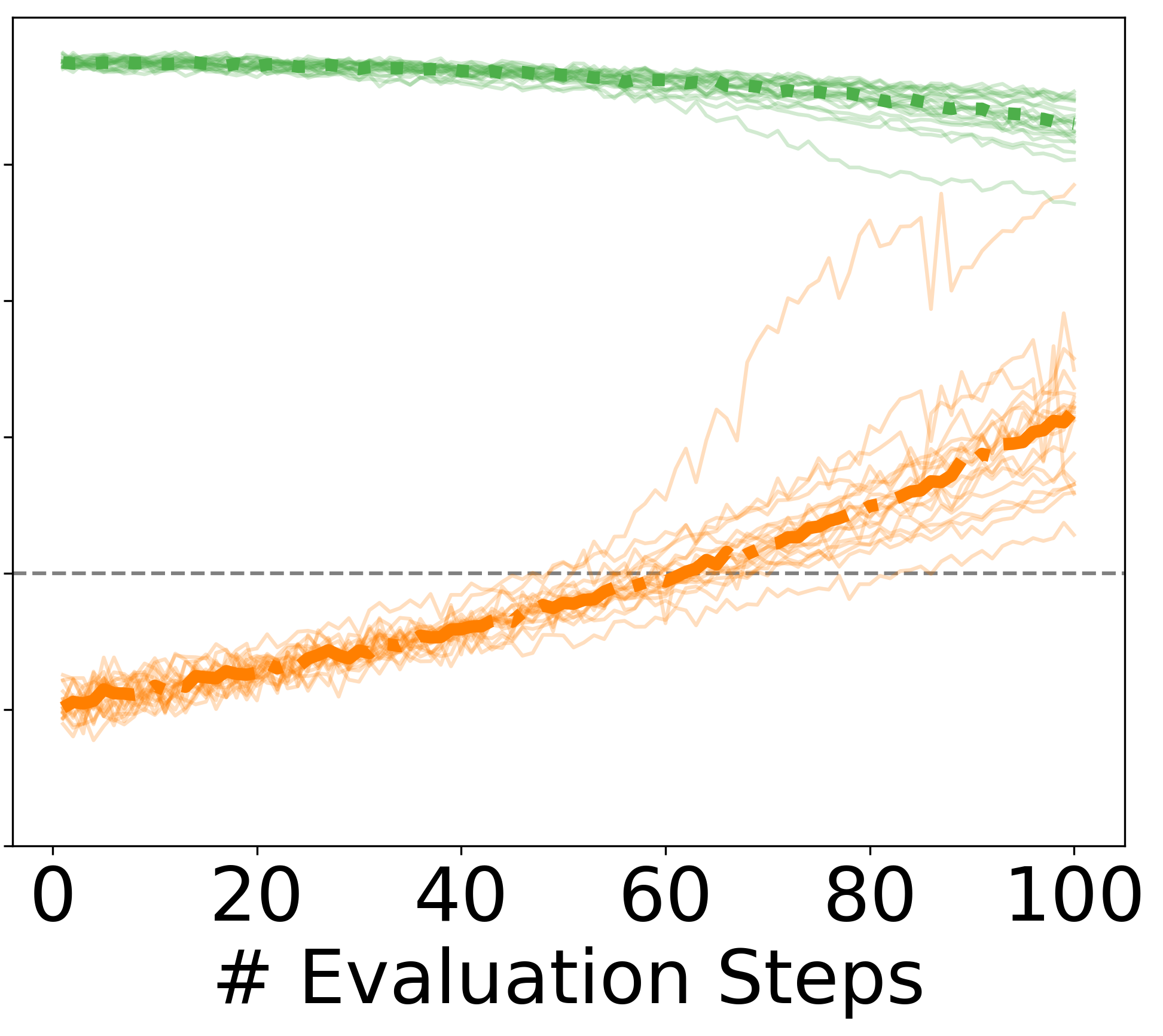}
    }
    \subfigure[{\footnotesize\centering \textsc{Bias GD/A}, EOD}]{        
        \includegraphics[width=0.215\linewidth]{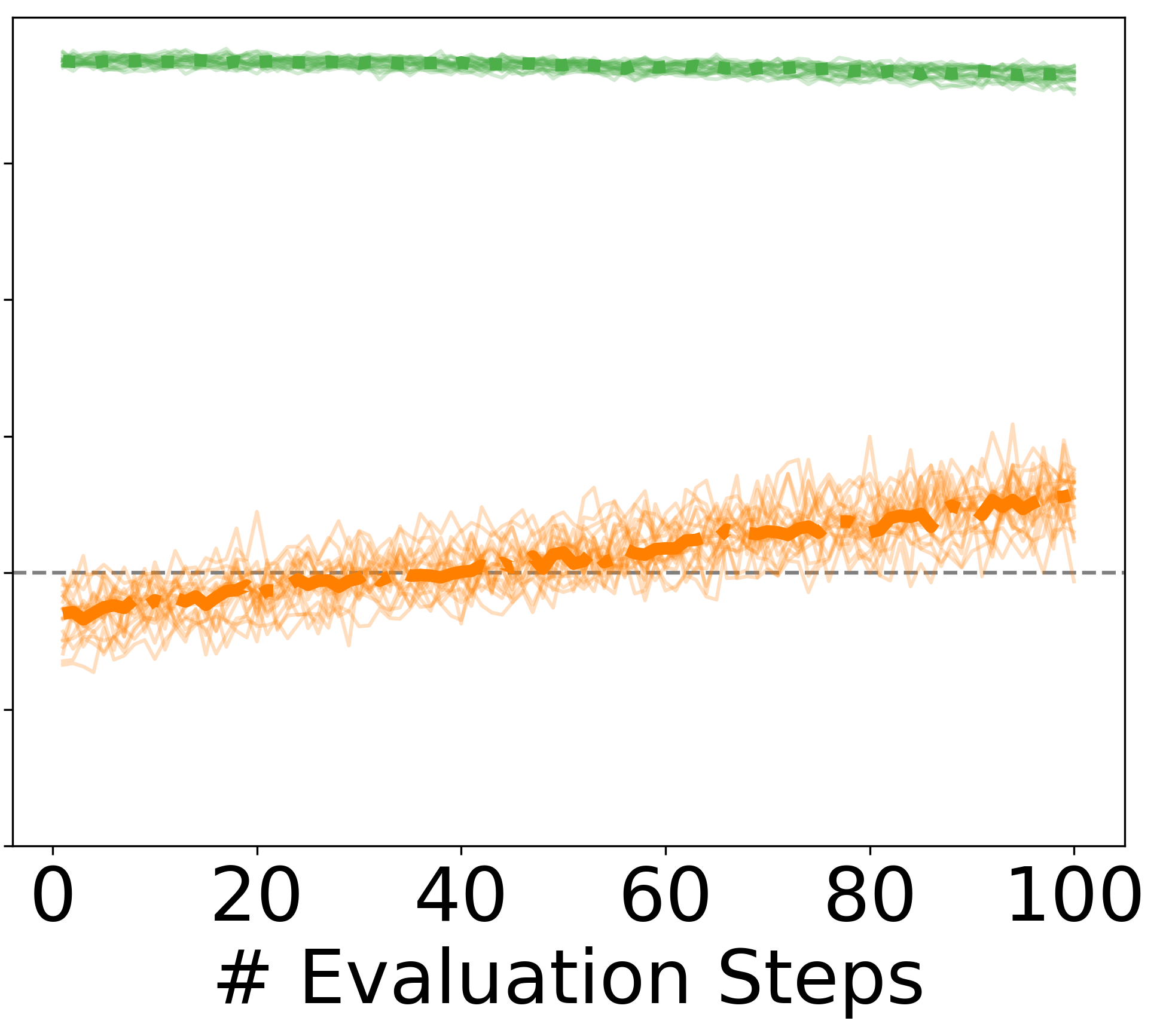}
    }
    \includegraphics[width=0.25\linewidth]{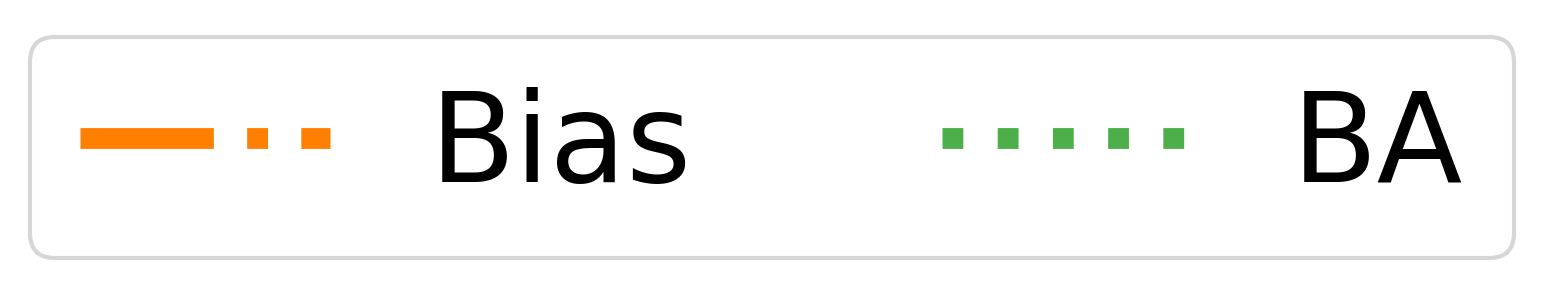}
    \caption{Changes in the {\color{LegendaryOrange}\textbf{bias}}, given by the SPD (\emph{a, c}) and EOD (\emph{b, d}), and {\color{LegendaryGreen}\textbf{balanced accuracy}} of the neural network during pruning (\emph{a, b}) and bias gradient descent/ascent (\emph{c, d}). The results were obtained on MIMIC-III for predicting in-hospital mortality with the ``\emph{insurance type}'' as the protected attribute from 20 train-validation-test splits. \textbf{Bold} lines correspond to the median across 20 seeds. Note that during the bias GD/A, the model was evaluated three times an epoch.}
    \label{fig:mimiciii_trajectories}
\end{figure}

\subsection{Chest X-ray Classification\label{sec:results_mimiccxr}}

Last but not least, we applied discussed intra- and post-processing techniques to CNNs trained on MIMIC-CXR. Table~\ref{tab:mimic_results} reports the EOD and BA after debiasing across 20 independent splits for the two pairs of protected attributes and labels and the two architectures. For predicting enlarged CM under the protected attribute ``\emph{sex}'', the EOD of the original classifier is mild for both VGG and ResNet, and most methods successfully reduce it without affecting the BA. Both pruning and bias GD/A achieve the best results on average alongside the equalised odds post-processing. Surprisingly, adversarial fine-tuning does not reduce the EOD equally well and leads to a lower BA. We attribute its poorer performance to overfitting on the validation data (see Table~\ref{tab:mimic_results_val} in Appendix~\ref{app:results}), likely, due to having many learnable parameters in the discriminator network. Nevertheless, it is encouraging that the classifiers trained on the data imbalanced w.r.t. the protected attribute can be debiased \emph{post hoc}, without retraining from scratch.

On the other hand, for predicting pneumonia under the protected attribute ``\emph{ethnicity}'', the average EOD of the original model is significantly higher for both architectures. The only method which satisfactorily reduces the bias, in this case, is equalised odds. While pruning and bias GD/A do not hurt the performance, the average EOD after debiasing is far from zero. However, both techniques still perform better than the na\"{i}ve random perturbation baseline. Similar to the results above, adversarial fine-tuning also fails to debias the network. We see several plausible explanations for the poorer performance of the proposed methods: \mbox{(\emph{i}) overfitting} to the small validation set --- for pneumonia and ``\emph{ethnicity}'', it only contains about 1,000 images; \mbox{(\emph{ii}) the} protected attribute ``\emph{ethnicity}'' is harder to detect from X-rays, compared to ``\emph{sex}'', and it may be prudent to use a post-processing technique which requires the attribute as input at \emph{test time}; \mbox{(\emph{iii}) the} general sensitivity of the intra-processing to the initial conditions (see Appendix~\ref{app:results}). Additionally, it must be noted that the attribute ``\emph{ethnicity}'' is self-reported \citep{SeyyedKalantari2020} and might be noisy and misaligned with the ground truth, introducing another source of variability into the results.

\begin{table}[h]
    \centering
    \caption{Equal opportunity difference and balanced accuracy attained before and after debiasing VGG-16 (\emph{a},\emph{b}) and ResNet-18 (\emph{c},\emph{d}) trained on \mbox{MIMIC-CXR} to predict \mbox{(\emph{a},\emph{c}) enlarged} cardiomediastinum (with the protected attribute ``\emph{sex}'') and \mbox{(\emph{b},\emph{d}) pneumonia} (with the protected attribute ``\emph{ethnicity}'').\label{tab:mimic_results}}
    
    \subtable[\centering Enlarged CM, \textit{Sex}; VGG-16]{
        \scriptsize
        \begin{tabular}{p{1.7cm}p{1.6cm}p{1.6cm}}
            \toprule
            \textbf{Method} & \textbf{EOD} & \textbf{BA} \\
            \toprule
            \cellcolor{Gainsboro}\textsc{Standard} & \cellcolor{Gainsboro}\tentry{-0.05}{0.02} & \cellcolor{Gainsboro}\tentry{0.77}{0.01} \\
            \textsc{Random} & \tentry{-0.03}{0.03} & \textit{\tentry{0.75}{0.01}} \\
            \textsc{ROC} & \tentry{-0.05}{0.02} & \textit{\tentry{0.75}{0.03}} \\
            \textsc{Eq. Odds} & \textit{\tentry{{\;}0.01}{0.03}} & \textit{\tentry{0.75}{0.01}} \\
            \textsc{Adv. Intra} & \tentry{-0.04}{0.03} & \tentry{0.73}{0.01} \\
            \underline{\textsc{Pruning}} & \textbf{\tentry{\;0.00}{0.02}} & \textbf{\tentry{0.76}{0.02}} \\
            \underline{\textsc{Bias GD/A}} & \textit{\tentry{-0.01}{0.04}} & \textbf{\tentry{0.76}{0.01}} \\
            \bottomrule
        \end{tabular}
    }
    \subtable[\centering Pneumonia, \textit{Ethnicity}; VGG-16]{
        \scriptsize
        \begin{tabular}{p{1.7cm}p{1.6cm}p{1.6cm}}
            \toprule
            \textbf{Method} & \textbf{EOD} & \textbf{BA} \\
            \toprule
            \cellcolor{Gainsboro}\textsc{Standard} & \cellcolor{Gainsboro}\tentry{-0.14}{0.04} & \cellcolor{Gainsboro}\tentry{0.73}{0.02}\\
            \textsc{Random} & \tentry{-0.11}{0.06} & \textbf{\tentry{0.71}{0.02}} \\
            \textsc{ROC} & \textit{\tentry{-0.07}{0.06}} & \tentry{0.65}{0.06} \\
            \textsc{Eq. Odds} & \textbf{\tentry{\;0.00}{0.06}} & \textit{\tentry{0.70}{0.01}}
            \\
            \textsc{Adv. Intra} & \tentry{-0.13}{0.05}  & \textit{\tentry{0.70}{0.02}} \\
            \underline{\textsc{Pruning}} & \tentry{-0.09}{0.05} & \textbf{\tentry{0.71}{0.03}} \\
            \underline{\textsc{Bias GD/A}} & \tentry{-0.08}{0.06} & \textbf{\tentry{0.71}{0.02}} \\
            \bottomrule
        \end{tabular}
    }
    
    \vspace{0.25cm}
    
    \subtable[\centering Enlarged CM, \textit{Sex}; ResNet-18]{
        \scriptsize
        \begin{tabular}{p{1.7cm}p{1.6cm}p{1.6cm}}
            \toprule
            \textbf{Method} & \textbf{EOD} & \textbf{BA} \\
            \toprule
            \cellcolor{Gainsboro}\textsc{Standard} & \cellcolor{Gainsboro}\tentry{-0.05}{0.04} & \cellcolor{Gainsboro}\tentry{0.76}{0.01} \\
            \textsc{Random} & \textbf{\tentry{\;0.00}{0.03}} & \tentry{0.73}{0.02} \\
            \textsc{ROC} & \tentry{-0.05}{0.03} & \textit{\tentry{0.74}{0.04}} \\
            \textsc{Eq. Odds} & \textit{\tentry{\;0.01}{0.03}} & \textit{\tentry{0.74}{0.01}} \\
            \textsc{Adv. Intra} & \tentry{-0.04}{0.04} & \tentry{0.73}{0.02} \\
            \underline{\textsc{Pruning}} & \textit{\tentry{-0.01}{0.03}} & \textit{\tentry{0.74}{0.02}} \\
            \underline{\textsc{Bias GD/A}} & \textbf{\tentry{\;0.00}{0.03}} & \textbf{\tentry{0.76}{0.01}} \\
            \bottomrule
        \end{tabular}
    }
    \subtable[\centering Pneumonia, \textit{Ethnicity}; ResNet-18]{
        \scriptsize
        \begin{tabular}{p{1.7cm}p{1.6cm}p{1.6cm}}
            \toprule
            \textbf{Method} & \textbf{EOD} & \textbf{BA} \\
            \toprule
            \cellcolor{Gainsboro}\textsc{Standard} & \cellcolor{Gainsboro}\tentry{-0.14}{0.05} & \cellcolor{Gainsboro}\tentry{0.73}{0.02}\\
            \textsc{Random} & \textit{\tentry{-0.06}{0.06}} & \tentry{0.65}{0.04} \\
            \textsc{ROC} & \tentry{-0.07}{0.04} & \tentry{0.65}{0.05} \\
            \textsc{Eq. Odds} & \textbf{\tentry{-0.01}{0.06}} & \tentry{0.70}{0.01}
            \\
            \textsc{Adv. Intra} & \tentry{-0.14}{0.03}  & \textit{\tentry{0.71}{0.02}} \\
            \underline{\textsc{Pruning}} & \tentry{-0.11}{0.05} & \tentry{0.70}{0.02} \\
            \underline{\textsc{Bias GD/A}} & \tentry{-0.11}{0.05} & \textbf{\tentry{0.73}{0.02}} \\
            \bottomrule
        \end{tabular}
    }
\end{table}

\section{Discussion\label{sec:discussion}}

The intra-processing methods proposed in this work (see \Secref{sec:methods}) offer a simple yet effective way of fine-tuning neural networks to mitigate classification disparity. 
The proposed differentiable bias proxy functions are directly related to the empirical covariance between the decision boundary and protected attribute (see Appendix~\ref{app:boundary}), similar to the loss functions considered by \citet{Zafar2017,Zafar2019} in the context of linear classification.
There exist criteria for fairness beyond the SPD and EOD \citep{CorbettDavies2018}. Our approaches can be readily combined with other parity measures, e.g. the average odds difference, by deriving proxies similar to those described in \Secref{subsec:class_parity_proxies}.

The presented debiasing techniques differ in several aspects from the related work. While there have been many efforts to debias neural networks resorting to adversarial training \citep{Zhang2018,Kim2019LearningNotToLearn,Reimers2021}, these works have mainly focused on the in-processing setting, where the source of bias is known during training. Notably, the adversarial in-processing technique by \citet{Zhang2018} is similar to the proposed bias GD/A (see \Secref{sec:biasgda}). Next to the main objective, it maximises the cross-entropy term for predicting the protected attribute using an adversary defined on the outputs of the base classifier. The current work offers a different perspective concentrating on the intra-processing scenario and chest X-ray classification. Moreover, to the best of our knowledge, neural network pruning or dropout have never been considered from the perspective of debiasing for group fairness. In particular, it would be interesting to investigate if the proposed pruning procedure may be augmented with other, non-gradient-based criteria for evaluating the influence of neurons (cf. \Eqref{eqn:saliency_grad}).

Although model-agnostic post-processing methods, such as ROC \citep{Kamiran2012} and equalised odds \citep{Hardt2016}, are applied \emph{post hoc}, they assume access to the protected attribute at test time. On the other hand, the methods introduced by \citet{Savani2020} are most closely related to ours. Random perturbation and layer-wise optimisation \citep{Savani2020} are based on computationally expensive zeroth-order optimisation techniques. While adversarial fine-tuning, similarly to bias GD/A, minimises a differentiable bias proxy by mini-batch gradient descent, it resorts to adversarial training, has more hyper- and learnable parameters than our methods, and leverages a different loss function. Furthermore, \citet{Savani2020} focus only on the conventional tabular debiasing benchmarks and natural images and do not apply their methods to clinical data.

Empirically, we have comprehensively evaluated our and related intra- and post-processing methods on various tabular and medical image datasets for fully connected and convolutional neural network architectures (see \Secref{sec:results}). In brief, the experiments on tabular data (see \Secref{sec:results_tabular}) suggest that the proposed intra-processing approaches effectively reduce the bias when it is present and offer improved performance over model-agnostic techniques. We have also demonstrated that pruning and GD/A can reduce the classification disparity in deep chest X-ray classifiers (see \Secref{sec:results_mimiccxr}) for VGG-16 and ResNet-18 networks. However, when the validation set is too small, and the bias of the initial model is too high, it might be more prudent to retrain the model from scratch or use a post-processing algorithm. 

Another contribution of this work is the application to deep chest X-ray classification. A body of previous literature has identified biases within the state-of-the-art models trained on large-scale publicly available datasets \citep{Larrazabal2020,SeyyedKalantari2020,SeyyedKalantari2021}. However, none of these works has investigated the \emph{mitigation} of the identified biases. We believe that the current paper is a valuable contribution to the discussion on achieving group fairness for medical image classifiers. Moreover, the considered intra-processing setting may become particularly pertinent to healthcare applications of machine learning due to the increasing adoption of transfer learning and pre-trained models.

\paragraph{Limitations}

The current study has several limitations. In particular, the experimental setup is restricted to few neural network architectures. It would be interesting to explore other CNN models, such as DenseNet \citep{Huang2017} and SqueezeNet \citep{Iandola2016}. For MIMIC-CXR, we have only focused on two protected-attribute-label pairs, and further investigation is warranted. Furthermore, we have considered binary classification under a binary-valued protected variable. Therefore, for practical use-cases, it is necessary to adapt our methods to the multilevel setting and extend them to the bias measures beyond the SPD and EOD, as explored by \citet{Zafar2019}.

\section{Conclusion\label{sec:conclusion}}

This work considered differentiable proxy functions for statistical parity and equality of opportunity, showing their correspondence to the covariance between the decision boundary of a neural network and the protected attribute. We proposed two novel intra-processing debiasing procedures based on neural network pruning and fine-tuning that utilise these proxies. Our experimental results on tabular data, including \mbox{MIMIC-III}, with fully connected neural networks indicate the viability of the proposed methods, especially compared to model-agnostic post-processing. Furthermore, we applied our and related techniques to mitigate disparity in chest X-ray classifiers trained on \mbox{MIMIC-CXR} and demonstrated that previously reported biases could be reduced without retraining models from scratch. 

\paragraph{Future Work}

In future work, it would be interesting to consider criteria for pruning beyond the gradient-based influence and study a more general setting with multiple classes and protected attribute categories. For tabular data, it could be helpful to investigate the use of pruning in the input layer to remove ``biased'' variables directly. The experimental results on \mbox{MIMIC-CXR} should be extended by more labels and protected attributes. 

\section*{Code and Data Availability}

The code is available at \url{https://github.com/i6092467/diff-bias-proxies}. All of the datasets in our experiments are publicly available.

\section*{Acknowledgements}

The authors would like to thank Dr. Mona Azadkia, Dr. Alexander Marx, and Imant Daunhawer for valuable discussions and feedback. RM was supported by the SNSF grant \#320038189096, EO was supported by the Hasler Foundation grant \#21050. 

\bibliography{bibliography.bib}

\begin{thebibliography}{71}
\providecommand{\natexlab}[1]{#1}
\providecommand{\url}[1]{\texttt{#1}}
\expandafter\ifx\csname urlstyle\endcsname\relax
  \providecommand{\doi}[1]{doi: #1}\else
  \providecommand{\doi}{doi: \begingroup \urlstyle{rm}\Url}\fi

\bibitem[Allaouzi and Ahmed(2019)]{Allaouzi2019}
Imane Allaouzi and Mohamed~Ben Ahmed.
\newblock A novel approach for multi-label chest {X}-ray classification of
  common thorax diseases.
\newblock \emph{{IEEE} Access}, 7:\penalty0 64279--64288, 2019.

\bibitem[Altman and Frist(2015)]{Altman2015}
Drew Altman and William~H. Frist.
\newblock {Medicare} and {Medicaid} at 50 years.
\newblock \emph{{JAMA}}, 314\penalty0 (4):\penalty0 384, 2015.

\bibitem[Ancona et~al.(2019)Ancona, Ceolini, \"{O}ztireli, and
  Gross]{Ancona2019}
Marco Ancona, Enea Ceolini, Cengiz \"{O}ztireli, and Markus Gross.
\newblock Gradient-based attribution methods.
\newblock In \emph{Explainable {AI}: Interpreting, Explaining and Visualizing
  Deep Learning}, pages 169--191. Springer International Publishing, 2019.

\bibitem[Bau et~al.(2020)Bau, Zhu, Strobelt, Lapedriza, Zhou, and
  Torralba]{Bau2020}
David Bau, Jun-Yan Zhu, Hendrik Strobelt, Agata Lapedriza, Bolei Zhou, and
  Antonio Torralba.
\newblock Understanding the role of individual units in a deep neural network.
\newblock \emph{Proceedings of the National Academy of Sciences}, 117\penalty0
  (48):\penalty0 30071--30078, 2020.

\bibitem[Bellamy et~al.(2018)Bellamy, Dey, Hind, Hoffman, Houde, Kannan, Lohia,
  Martino, Mehta, Mojsilovic, Nagar, Ramamurthy, Richards, Saha, Sattigeri,
  Singh, Varshney, and Zhang]{Bellamy2018}
Rachel K.~E. Bellamy, Kuntal Dey, Michael Hind, Samuel~C. Hoffman, Stephanie
  Houde, Kalapriya Kannan, Pranay Lohia, Jacquelyn Martino, Sameep Mehta,
  Aleksandra Mojsilovic, Seema Nagar, Karthikeyan~Natesan Ramamurthy, John
  Richards, Diptikalyan Saha, Prasanna Sattigeri, Moninder Singh, Kush~R.
  Varshney, and Yunfeng Zhang.
\newblock {AI Fairness 360}: An extensible toolkit for detecting,
  understanding, and mitigating unwanted algorithmic bias, 2018.
\newblock arXiv:1810.01943.

\bibitem[Besse et~al.(2021)Besse, del Barrio, Gordaliza, Loubes, and
  Risser]{Besse2021}
Philippe Besse, Eustasio del Barrio, Paula Gordaliza, Jean-Michel Loubes, and
  Laurent Risser.
\newblock A survey of bias in machine learning through the prism of statistical
  parity.
\newblock \emph{The American Statistician}, pages 1--11, 2021.

\bibitem[Blalock et~al.(2020)Blalock, Gonzalez~Ortiz, Frankle, and
  Guttag]{Blalock2020}
Davis Blalock, Jose~Javier Gonzalez~Ortiz, Jonathan Frankle, and John Guttag.
\newblock What is the state of neural network pruning?
\newblock In \emph{Proceedings of Machine Learning and Systems}, volume~2,
  pages 129--146, 2020.

\bibitem[Bressem et~al.(2020)Bressem, Adams, Erxleben, Hamm, Niehues, and
  Vahldiek]{Bressem2020}
Keno~K. Bressem, Lisa~C. Adams, Christoph Erxleben, Bernd Hamm, Stefan~M.
  Niehues, and Janis~L. Vahldiek.
\newblock Comparing different deep learning architectures for classification of
  chest radiographs.
\newblock \emph{Scientific Reports}, 10\penalty0 (1), 2020.

\bibitem[Brodersen et~al.(2010)Brodersen, Ong, Stephan, and
  Buhmann]{Brodersen2010}
Kay~Henning Brodersen, Cheng~Soon Ong, Klaas~Enno Stephan, and Joachim~M.
  Buhmann.
\newblock The balanced accuracy and its posterior distribution.
\newblock In \emph{20th International Conference on Pattern Recognition}.
  {IEEE}, 2010.

\bibitem[Calmon et~al.(2017)Calmon, Wei, Vinzamuri, Natesan~Ramamurthy, and
  Varshney]{Calmon2017}
Flavio Calmon, Dennis Wei, Bhanukiran Vinzamuri, Karthikeyan
  Natesan~Ramamurthy, and Kush~R Varshney.
\newblock Optimized pre-processing for discrimination prevention.
\newblock In \emph{Advances in Neural Information Processing Systems},
  volume~30. Curran Associates, Inc., 2017.

\bibitem[Celis et~al.(2020)Celis, Keswani, and Vishnoi]{Celis2020}
L.~Elisa Celis, Vijay Keswani, and Nisheeth Vishnoi.
\newblock Data preprocessing to mitigate bias: A maximum entropy based
  approach.
\newblock In \emph{Proceedings of the 37th International Conference on Machine
  Learning}, volume 119, pages 1349--1359. PMLR, 2020.

\bibitem[Char et~al.(2018)Char, Shah, and Magnus]{Char2018}
Danton~S. Char, Nigam~H. Shah, and David Magnus.
\newblock Implementing machine learning in health care {\textemdash} addressing
  ethical challenges.
\newblock \emph{New England Journal of Medicine}, 378\penalty0 (11):\penalty0
  981--983, 2018.

\bibitem[Cheng et~al.(2017)Cheng, Wang, Zhou, and Zhang]{Cheng2017}
Yu~Cheng, Duo Wang, Pan Zhou, and Tao Zhang.
\newblock A survey of model compression and acceleration for deep neural
  networks, 2017.
\newblock arXiv:1710.09282.

\bibitem[Chopra et~al.(2020)Chopra, Baghel, Annadate, and
  Shanmugasundaram]{Chopra2020}
Shweta Chopra, Nupur Baghel, Shubham Annadate, and Rajalakshmi~D.
  Shanmugasundaram.
\newblock Fighting algorithmic bias using adversarial networks, 2020.
\newblock URL
  \url{https://github.com/choprashweta/Adversarial-Debiasing/blob/master/CIS_519_Project_Report\%20(4).pdf}.

\bibitem[Cohen et~al.(2020)Cohen, Hashir, Brooks, and Bertrand]{Cohen2020}
Joseph~Paul Cohen, Mohammad Hashir, Rupert Brooks, and Hadrien Bertrand.
\newblock On the limits of cross-domain generalization in automated {X}-ray
  prediction.
\newblock In \emph{Medical Imaging with Deep Learning}, 2020.

\bibitem[Corbett-Davies and Goel(2018)]{CorbettDavies2018}
Sam Corbett-Davies and Sharad Goel.
\newblock The measure and mismeasure of fairness: A critical review of fair
  machine learning, 2018.
\newblock arXiv:1808.00023.

\bibitem[Dhamdhere et~al.(2019)Dhamdhere, Sundararajan, and Yan]{Dhamdhere2018}
Kedar Dhamdhere, Mukund Sundararajan, and Qiqi Yan.
\newblock How important is a neuron?
\newblock In \emph{International Conference on Learning Representations}, 2019.

\bibitem[Gupta et~al.(2018)Gupta, Malhotra, Vig, and Shroff]{Gupta2018}
Priyanka Gupta, Pankaj Malhotra, Lovekesh Vig, and Gautam~M. Shroff.
\newblock Using features from pre-trained {TimeNET} for clinical predictions.
\newblock In \emph{KHD@IJCAI}, 2018.

\bibitem[Hardt et~al.(2016)Hardt, Price, and Srebro]{Hardt2016}
Moritz Hardt, Eric Price, and Nati Srebro.
\newblock Equality of opportunity in supervised learning.
\newblock In \emph{Advances in Neural Information Processing Systems},
  volume~29. Curran Associates, Inc., 2016.

\bibitem[Hassibi and Stork(1993)]{Hassibi1993}
Babak Hassibi and David Stork.
\newblock Second order derivatives for network pruning: Optimal brain surgeon.
\newblock In \emph{Advances in Neural Information Processing Systems},
  volume~5. Morgan-Kaufmann, 1993.

\bibitem[He et~al.(2016)He, Zhang, Ren, and Sun]{He2016}
Kaiming He, Xiangyu Zhang, Shaoqing Ren, and Jian Sun.
\newblock Deep residual learning for image recognition.
\newblock In \emph{2016 {IEEE} Conference on Computer Vision and Pattern
  Recognition ({CVPR})}. {IEEE}, 2016.

\bibitem[He et~al.(2017)He, Zhang, and Sun]{He2017}
Yihui He, Xiangyu Zhang, and Jian Sun.
\newblock Channel pruning for accelerating very deep neural networks.
\newblock In \emph{Proceedings of the IEEE International Conference on Computer
  Vision (ICCV)}, 2017.

\bibitem[Huang et~al.(2017)Huang, Liu, van~der Maaten, and
  Weinberger]{Huang2017}
Gao Huang, Zhuang Liu, Laurens van~der Maaten, and Kilian~Q. Weinberger.
\newblock Densely connected convolutional networks.
\newblock In \emph{Proceedings of the IEEE Conference on Computer Vision and
  Pattern Recognition (CVPR)}, 2017.

\bibitem[Iandola et~al.(2016)Iandola, Han, Moskewicz, Ashraf, Dally, and
  Keutzer]{Iandola2016}
Forrest~N. Iandola, Song Han, Matthew~W. Moskewicz, Khalid Ashraf, William~J.
  Dally, and Kurt Keutzer.
\newblock {SqueezeNet}: {AlexNet}-level accuracy with 50x fewer parameters and
  $<$ 0.5{MB} model size, 2016.
\newblock arXiv:1602.07360.

\bibitem[Irvin et~al.(2019)Irvin, Rajpurkar, Ko, Yu, Ciurea-Ilcus, Chute,
  Marklund, Haghgoo, Ball, Shpanskaya, Seekins, Mong, Halabi, Sandberg, Jones,
  Larson, Langlotz, Patel, Lungren, and Ng]{Irvin2019}
Jeremy Irvin, Pranav Rajpurkar, Michael Ko, Yifan Yu, Silviana Ciurea-Ilcus,
  Chris Chute, Henrik Marklund, Behzad Haghgoo, Robyn Ball, Katie Shpanskaya,
  Jayne Seekins, {David A.} Mong, {Safwan S.} Halabi, {Jesse K.} Sandberg,
  Ricky Jones, {David B.} Larson, {Curtis P.} Langlotz, {Bhavik N.} Patel,
  {Matthew P.} Lungren, and {Andrew Y.} Ng.
\newblock {CheXpert}: A large chest radiograph dataset with uncertainty labels
  and expert comparison.
\newblock In \emph{33rd AAAI Conference on Artificial Intelligence, AAAI 2019,
  31st Innovative Applications of Artificial Intelligence Conference, IAAI 2019
  and the 9th AAAI Symposium on Educational Advances in Artificial
  Intelligence, EAAI 2019}, pages 590--597, 2019.

\bibitem[Johnson et~al.(2016)Johnson, Pollard, Shen, wei H.~Lehman, Feng,
  Ghassemi, Moody, Szolovits, Celi, and Mark]{Johnson2016}
Alistair E.~W. Johnson, Tom~J. Pollard, Lu~Shen, Li~wei H.~Lehman, Mengling
  Feng, Mohammad Ghassemi, Benjamin Moody, Peter Szolovits, Leo~Anthony Celi,
  and Roger~G. Mark.
\newblock {MIMIC}-{III}, a freely accessible critical care database.
\newblock \emph{Scientific Data}, 3\penalty0 (1), 2016.

\bibitem[Johnson et~al.(2019)Johnson, Pollard, Greenbaum, Lungren, Deng, Peng,
  Lu, Mark, Berkowitz, and Horng]{mimic-cxr}
Alistair E.~W. Johnson, Tom~J. Pollard, Nathaniel~R. Greenbaum, Matthew~P.
  Lungren, Chih-ying Deng, Yifan Peng, Zhiyong Lu, Roger~G. Mark, Seth~J.
  Berkowitz, and Steven Horng.
\newblock {MIMIC-CXR-JPG}: A large publicly available database of labeled chest
  radiographs, 2019.

\bibitem[Kamiran and Calders(2011)]{Kamiran2011}
Faisal Kamiran and Toon Calders.
\newblock Data preprocessing techniques for classification without
  discrimination.
\newblock \emph{Knowledge and Information Systems}, 33\penalty0 (1):\penalty0
  1--33, 2011.

\bibitem[Kamiran et~al.(2012)Kamiran, Karim, and Zhang]{Kamiran2012}
Faisal Kamiran, Asim Karim, and Xiangliang Zhang.
\newblock Decision theory for discrimination-aware classification.
\newblock In \emph{2012 IEEE 12th International Conference on Data Mining},
  pages 924--929, 2012.

\bibitem[Kamishima et~al.(2012)Kamishima, Akaho, Asoh, and
  Sakuma]{Kamishima2012}
Toshihiro Kamishima, Shotaro Akaho, Hideki Asoh, and Jun Sakuma.
\newblock Fairness-aware classifier with prejudice remover regularizer.
\newblock In \emph{Machine Learning and Knowledge Discovery in Databases},
  pages 35--50. Springer Berlin Heidelberg, 2012.

\bibitem[Kearns(2017)]{Kearns2017}
Michael Kearns.
\newblock Fair algorithms for machine learning.
\newblock In \emph{Proceedings of the 2017 ACM Conference on Economics and
  Computation}, page~1. Association for Computing Machinery, 2017.

\bibitem[Kim et~al.(2019{\natexlab{a}})Kim, Kim, Kim, Kim, and
  Kim]{Kim2019LearningNotToLearn}
Byungju Kim, Hyunwoo Kim, Kyungsu Kim, Sungjin Kim, and Junmo Kim.
\newblock Learning not to learn: Training deep neural networks with biased
  data.
\newblock In \emph{Proceedings of the IEEE/CVF Conference on Computer Vision
  and Pattern Recognition}, pages 9012--9020, 2019{\natexlab{a}}.

\bibitem[Kim et~al.(2019{\natexlab{b}})Kim, Ghorbani, and Zou]{Kim2019}
Michael~P. Kim, Amirata Ghorbani, and James Zou.
\newblock Multiaccuracy: Black-box post-processing for fairness in
  classification.
\newblock In \emph{Proceedings of the 2019 AAAI/ACM Conference on AI, Ethics,
  and Society}, pages 247--254. Association for Computing Machinery,
  2019{\natexlab{b}}.

\bibitem[Kingma and Ba(2015)]{Kingma2015}
Diederik~P. Kingma and Jimmy Ba.
\newblock Adam: {A} method for stochastic optimization.
\newblock In \emph{3rd International Conference on Learning Representations},
  2015.

\bibitem[Kohavi(1996)]{Kohavi1996}
Ron Kohavi.
\newblock Scaling up the accuracy of naive-{Bayes} classifiers: A decision-tree
  hybrid.
\newblock In \emph{Proceedings of the Second International Conference on
  Knowledge Discovery and Data Mining}, pages 202--207. AAAI Press, 1996.

\bibitem[Larrazabal et~al.(2020)Larrazabal, Nieto, Peterson, Milone, and
  Ferrante]{Larrazabal2020}
Agostina~J. Larrazabal, Nicol{\'{a}}s Nieto, Victoria Peterson, Diego~H.
  Milone, and Enzo Ferrante.
\newblock Gender imbalance in medical imaging datasets produces biased
  classifiers for computer-aided diagnosis.
\newblock \emph{Proceedings of the National Academy of Sciences}, 117\penalty0
  (23):\penalty0 12592--12594, 2020.

\bibitem[Larson et~al.(2016)Larson, Mattu, Kirchner, and Angwin]{ProPublica}
Jeff Larson, Surya Mattu, Lauren Kirchner, and Julia Angwin.
\newblock How we analyzed the {COMPAS} recidivism algorithm.
\newblock
  \url{https://www.propublica.org/article/how-we-analyzed-the-compas-recidivism-algorithm},
  2016.
\newblock Accessed: 2020.11.02.

\bibitem[LeCun et~al.(1990)LeCun, Denker, and Solla]{LeCun1990}
Yann LeCun, John Denker, and Sara Solla.
\newblock Optimal brain damage.
\newblock In \emph{Advances in Neural Information Processing Systems},
  volume~2. Morgan-Kaufmann, 1990.

\bibitem[Lee et~al.(2021)Lee, Kim, Lee, Lee, and Choo]{Lee2021}
Jungsoo Lee, Eungyeup Kim, Juyoung Lee, Jihyeon Lee, and Jaegul Choo.
\newblock Learning debiased representation via disentangled feature
  augmentation.
\newblock In \emph{Advances in Neural Information Processing Systems}, 2021.

\bibitem[Leino et~al.(2018)Leino, Sen, Datta, Fredrikson, and Li]{Leino2018}
Klas Leino, Shayak Sen, Anupam Datta, Matt Fredrikson, and Linyi Li.
\newblock Influence-directed explanations for deep convolutional networks.
\newblock In \emph{{IEEE} International Test Conference ({ITC})}. {IEEE}, 2018.

\bibitem[Loh et~al.(2019)Loh, Cao, and Zhou]{Loh2019}
Wei-Yin Loh, Luxi Cao, and Peigen Zhou.
\newblock Subgroup identification for precision medicine: A comparative review
  of 13 methods.
\newblock \emph{{WIREs} Data Mining and Knowledge Discovery}, 9\penalty0 (5),
  2019.

\bibitem[Loshchilov and Hutter(2019)]{Loshchilov2019}
Ilya Loshchilov and Frank Hutter.
\newblock Decoupled weight decay regularization.
\newblock In \emph{7th International Conference on Learning Representations,
  {ICLR}}, 2019.

\bibitem[Meng et~al.(2022)Meng, Trinh, Xu, Enouen, and Liu]{mimicifMeng2021}
Chuizheng Meng, Loc Trinh, Nan Xu, James Enouen, and Yan Liu.
\newblock Interpretability and fairness evaluation of deep learning models on
  {MIMIC}-{IV} dataset.
\newblock \emph{Scientific Reports}, 12\penalty0 (1), 2022.

\bibitem[Molchanov et~al.(2017)Molchanov, Tyree, Karras, Aila, and
  Kautz]{Molchanov2016}
Pavlo Molchanov, Stephen Tyree, Tero Karras, Timo Aila, and Jan Kautz.
\newblock Pruning convolutional neural networks for resource efficient
  inference.
\newblock In \emph{5th International Conference on Learning Representations,
  {ICLR}}, 2017.

\bibitem[Moro et~al.(2014)Moro, Cortez, and Rita]{Moro2014}
Sérgio Moro, Paulo Cortez, and Paulo Rita.
\newblock A data-driven approach to predict the success of bank telemarketing.
\newblock \emph{Decision Support Systems}, 62:\penalty0 22--31, 2014.

\bibitem[Nam et~al.(2020{\natexlab{a}})Nam, Cha, Ahn, Lee, and
  Shin]{Nam2020LfF}
Junhyun Nam, Hyuntak Cha, Sungsoo Ahn, Jaeho Lee, and Jinwoo Shin.
\newblock Learning from failure: De-biasing classifier from biased classifier.
\newblock In \emph{Advances in Neural Information Processing Systems},
  volume~33, pages 20673--20684. Curran Associates, Inc., 2020{\natexlab{a}}.

\bibitem[Nam et~al.(2020{\natexlab{b}})Nam, Gur, Choi, Wolf, and Lee]{Nam2020}
Woo-Jeoung Nam, Shir Gur, Jaesik Choi, Lior Wolf, and Seong-Whan Lee.
\newblock Relative attributing propagation: Interpreting the comparative
  contributions of individual units in deep neural networks.
\newblock \emph{Proceedings of the AAAI Conference on Artificial Intelligence},
  34\penalty0 (03):\penalty0 2501--2508, 2020{\natexlab{b}}.

\bibitem[Obermeyer et~al.(2019)Obermeyer, Powers, Vogeli, and
  Mullainathan]{Obermeyer2019}
Ziad Obermeyer, Brian Powers, Christine Vogeli, and Sendhil Mullainathan.
\newblock Dissecting racial bias in an algorithm used to manage the health of
  populations.
\newblock \emph{Science}, 366\penalty0 (6464):\penalty0 447--453, 2019.

\bibitem[Paszke et~al.(2017)Paszke, Gross, Chintala, Chanan, Yang, DeVito, Lin,
  Desmaison, Antiga, and Lerer]{Paszke2017}
Adam Paszke, Sam Gross, Soumith Chintala, Gregory Chanan, Edward Yang, Zachary
  DeVito, Zeming Lin, Alban Desmaison, Luca Antiga, and Adam Lerer.
\newblock Automatic differentiation in {PyTorch}, 2017.
\newblock NIPS 2017 Autodiff Workshop.

\bibitem[Pleiss et~al.(2017)Pleiss, Raghavan, Wu, Kleinberg, and
  Weinberger]{Pleiss2017}
Geoff Pleiss, Manish Raghavan, Felix Wu, Jon Kleinberg, and Kilian~Q.
  Weinberger.
\newblock On fairness and calibration.
\newblock In \emph{Proceedings of the 31st International Conference on Neural
  Information Processing Systems}, pages 5684--5693. Curran Associates Inc.,
  2017.

\bibitem[Purushotham et~al.(2018)Purushotham, Meng, Che, and
  Liu]{benchmarkingPurushotham2018}
Sanjay Purushotham, Chuizheng Meng, Zhengping Che, and Yan Liu.
\newblock Benchmarking deep learning models on large healthcare datasets.
\newblock \emph{Journal of Biomedical Informatics}, 83:\penalty0 112--134,
  2018.

\bibitem[Quy et~al.(2022)Quy, Roy, Iosifidis, Zhang, and Ntoutsi]{Quy2021}
Tai~Le Quy, Arjun Roy, Vasileios Iosifidis, Wenbin Zhang, and Eirini Ntoutsi.
\newblock A survey on datasets for fairness-aware machine learning.
\newblock \emph{{WIREs} Data Mining and Knowledge Discovery}, 12\penalty0 (3),
  2022.

\bibitem[Raghu et~al.(2019)Raghu, Zhang, Kleinberg, and Bengio]{Raghu2019}
Maithra Raghu, Chiyuan Zhang, Jon Kleinberg, and Samy Bengio.
\newblock Transfusion: Understanding transfer learning for medical imaging.
\newblock In \emph{Advances in Neural Information Processing Systems},
  volume~32. Curran Associates, Inc., 2019.

\bibitem[Rajkomar et~al.(2018)Rajkomar, Hardt, Howell, Corrado, and
  Chin]{Rajkomar2018}
Alvin Rajkomar, Michaela Hardt, Michael~D. Howell, Greg Corrado, and
  Marshall~H. Chin.
\newblock Ensuring fairness in machine learning to advance health equity.
\newblock \emph{Annals of Internal Medicine}, 169\penalty0 (12):\penalty0
  866--872, 2018.

\bibitem[Rajpurkar et~al.(2017)Rajpurkar, Irvin, Zhu, Yang, Mehta, Duan, Ding,
  Bagul, Langlotz, Shpanskaya, et~al.]{Rajpurkar2017}
Pranav Rajpurkar, Jeremy Irvin, Kaylie Zhu, Brandon Yang, Hershel Mehta, Tony
  Duan, Daisy Ding, Aarti Bagul, Curtis Langlotz, Katie Shpanskaya, et~al.
\newblock Chexnet: Radiologist-level pneumonia detection on chest {X}-rays with
  deep learning.
\newblock \emph{arXiv preprint arXiv:1711.05225}, 2017.

\bibitem[Rasmy et~al.(2021)Rasmy, Xiang, Xie, Tao, and Zhi]{Rasmy2021}
Laila Rasmy, Yang Xiang, Ziqian Xie, Cui Tao, and Degui Zhi.
\newblock Med-{BERT}: pretrained contextualized embeddings on large-scale
  structured electronic health records for disease prediction.
\newblock \emph{npj Digital Medicine}, 4\penalty0 (1), 2021.

\bibitem[Reimers et~al.(2021)Reimers, Bodesheim, Runge, and
  Denzler]{Reimers2021}
Christian Reimers, Paul Bodesheim, Jakob Runge, and Joachim Denzler.
\newblock Conditional adversarial debiasing: Towards learning unbiased
  classifiers from biased data.
\newblock In \emph{Pattern Recognition}, pages 48--62. Springer International
  Publishing, 2021.

\bibitem[Savani et~al.(2020)Savani, White, and Govindarajulu]{Savani2020}
Yash Savani, Colin White, and Naveen~Sundar Govindarajulu.
\newblock Intra-processing methods for debiasing neural networks.
\newblock In \emph{Advances in Neural Information Processing Systems},
  volume~33, pages 2798--2810. Curran Associates, Inc., 2020.

\bibitem[Seyyed-Kalantari et~al.(2020)Seyyed-Kalantari, Liu, McDermott, Chen,
  and Ghassemi]{SeyyedKalantari2020}
Laleh Seyyed-Kalantari, Guanxiong Liu, Matthew McDermott, Irene~Y Chen, and
  Marzyeh Ghassemi.
\newblock {CheXclusion}: Fairness gaps in deep chest {X}-ray classifiers.
\newblock In \emph{BIOCOMPUTING 2021: Proceedings of the Pacific Symposium},
  pages 232--243. World Scientific, 2020.

\bibitem[Seyyed-Kalantari et~al.(2021)Seyyed-Kalantari, Zhang, McDermott, Chen,
  and Ghassemi]{SeyyedKalantari2021}
Laleh Seyyed-Kalantari, Haoran Zhang, Matthew B.~A. McDermott, Irene~Y. Chen,
  and Marzyeh Ghassemi.
\newblock Underdiagnosis bias of artificial intelligence algorithms applied to
  chest radiographs in under-served patient populations.
\newblock \emph{Nature Medicine}, 27\penalty0 (12):\penalty0 2176--2182, 2021.

\bibitem[Simonyan and Zisserman(2015)]{Simonyan2014}
Karen Simonyan and Andrew Zisserman.
\newblock Very deep convolutional networks for large-scale image recognition.
\newblock In \emph{3rd International Conference on Learning Representations,
  {ICLR}}, 2015.

\bibitem[Srinivas and Fleuret(2019)]{Srinivas2019}
Suraj Srinivas and Fran\c{c}ois Fleuret.
\newblock Full-gradient representation for neural network visualization.
\newblock In \emph{Advances in Neural Information Processing Systems},
  volume~32. Curran Associates, Inc., 2019.

\bibitem[Srivastava et~al.(2014)Srivastava, Hinton, Krizhevsky, Sutskever, and
  Salakhutdinov]{Srivastava2014}
Nitish Srivastava, Geoffrey Hinton, Alex Krizhevsky, Ilya Sutskever, and Ruslan
  Salakhutdinov.
\newblock Dropout: A simple way to prevent neural networks from overfitting.
\newblock \emph{Journal of Machine Learning Research}, 15\penalty0
  (1):\penalty0 1929--1958, 2014.

\bibitem[Wang et~al.(2019)Wang, Peng, Lu, Lu, Bagheri, and Summers]{Wang2019}
Xiaosong Wang, Yifan Peng, Le~Lu, Zhiyong Lu, Mohammadhadi Bagheri, and
  Ronald~M. Summers.
\newblock {ChestX}-ray: Hospital-scale chest {X}-ray database and benchmarks on
  weakly supervised classification and localization of common thorax diseases.
\newblock In \emph{Deep Learning and Convolutional Neural Networks for Medical
  Imaging and Clinical Informatics}, pages 369--392. Springer International
  Publishing, 2019.

\bibitem[Wen et~al.(2016)Wen, Wu, Wang, Chen, and Li]{Wen2016}
Wei Wen, Chunpeng Wu, Yandan Wang, Yiran Chen, and Hai Li.
\newblock Learning structured sparsity in deep neural networks.
\newblock In \emph{Advances in Neural Information Processing Systems},
  volume~29. Curran Associates, Inc., 2016.

\bibitem[Wiens et~al.(2019)Wiens, Saria, Sendak, Ghassemi, Liu, Doshi-Velez,
  Jung, Heller, Kale, Saeed, Ossorio, Thadaney-Israni, and
  Goldenberg]{Wiens2019}
Jenna Wiens, Suchi Saria, Mark Sendak, Marzyeh Ghassemi, Vincent~X. Liu, Finale
  Doshi-Velez, Kenneth Jung, Katherine Heller, David Kale, Mohammed Saeed,
  Pilar~N. Ossorio, Sonoo Thadaney-Israni, and Anna Goldenberg.
\newblock Do no harm: a roadmap for responsible machine learning for health
  care.
\newblock \emph{Nature Medicine}, 25\penalty0 (9):\penalty0 1337--1340, 2019.

\bibitem[Zafar et~al.(2017)Zafar, Valera, Rogriguez, and Gummadi]{Zafar2017}
Muhammad~Bilal Zafar, Isabel Valera, Manuel~Gomez Rogriguez, and Krishna~P.
  Gummadi.
\newblock {Fairness Constraints: Mechanisms for Fair Classification}.
\newblock In \emph{Proceedings of the 20th International Conference on
  Artificial Intelligence and Statistics}, volume~54, pages 962--970. PMLR,
  2017.

\bibitem[Zafar et~al.(2019)Zafar, Valera, Gomez-Rodriguez, and
  Gummadi]{Zafar2019}
Muhammad~Bilal Zafar, Isabel Valera, Manuel Gomez-Rodriguez, and Krishna~P.
  Gummadi.
\newblock Fairness constraints: A flexible approach for fair classification.
\newblock \emph{Journal of Machine Learning Research}, 20\penalty0
  (75):\penalty0 1--42, 2019.

\bibitem[Zech et~al.(2018)Zech, Badgeley, Liu, Costa, Titano, and
  Oermann]{Zech2018}
John~R. Zech, Marcus~A. Badgeley, Manway Liu, Anthony~B. Costa, Joseph~J.
  Titano, and Eric~Karl Oermann.
\newblock Variable generalization performance of a deep learning model to
  detect pneumonia in chest radiographs: A cross-sectional study.
\newblock \emph{{PLOS} Medicine}, 15\penalty0 (11):\penalty0 e1002683, 2018.

\bibitem[Zemel et~al.(2013)Zemel, Wu, Swersky, Pitassi, and Dwork]{Zemel2013}
Rich Zemel, Yu~Wu, Kevin Swersky, Toni Pitassi, and Cynthia Dwork.
\newblock Learning fair representations.
\newblock In \emph{Proceedings of the 30th International Conference on Machine
  Learning}, volume~28, pages 325--333. PMLR, 2013.

\bibitem[Zhang et~al.(2018)Zhang, Lemoine, and Mitchell]{Zhang2018}
Brian~Hu Zhang, Blake Lemoine, and Margaret Mitchell.
\newblock Mitigating unwanted biases with adversarial learning.
\newblock In \emph{Proceedings of the 2018 {AAAI}/{ACM} Conference on {AI},
  Ethics, and Society}. {ACM}, 2018.

\end{thebibliography}

\newpage

\appendix

\counterwithin{figure}{section}
\counterwithin{table}{section}
\counterwithin{equation}{section}

\section{Decision Boundary Covariance\label{app:boundary}}

In this appendix, we study the relationship between the differentiable bias proxies in Equations~\ref{eqn:spd_} and \ref{eqn:eod_} (see \Secref{sec:methods} of the main text) and the covariance between the decision boundary of the classifier $f_{\boldsymbol{\theta}}(\cdot)$ and the protected attribute $A$.

\begin{lemma}
    For $\mathcal{X}=\left\{\vx_i\right\}_{i=1}^N$, $\mathcal{Y}=\left\{y_i\right\}_{i=1}^N$, and $\mathcal{A}=\left\{a_i\right\}_{i=1}^N$ and some classifier $f_{\boldsymbol{\theta}}(\cdot)$,  $-\tilde{\mu}_{\mathrm{SPD}}\left(f_{\boldsymbol{\theta}},\,\mathcal{X},\,\mathcal{Y},\,\mathcal{A}\right)\propto\widehat{\mathrm{Cov}}\left(A,\,f_{\boldsymbol{\theta}}\left(\mX\right)\right)$.
\end{lemma}

\begin{proof}
    Recall that the covariance is given by
    \[
    \mathrm{Cov}\left(A,\,f_{\boldsymbol{\theta}}\left(\mX\right)\right)=\mathbb{E}\left[Af_{\boldsymbol{\theta}}\left(\mX\right)\right]-\mathbb{E}\left[A\right]\mathbb{E}\left[f_{\boldsymbol{\theta}}(\mX)\right].
    \]
    Let $K=\sum_{i=1}^Na_i$ and $\widebar{f_{\boldsymbol{\theta}}\left(\vx\right)}=\frac{1}{N}\sum_{i}^Nf_{\boldsymbol{\theta}}\left(\vx_i\right)$, consider an empirical estimate
    \begin{equation}
        \widehat{\mathrm{Cov}}\left(A,\,f_{\boldsymbol{\theta}}\left(\mX\right)\right)=\frac{1}{N}\sum_{i=1}^Nf_{\boldsymbol{\theta}}\left(\vx_i\right)a_i-\frac{K}{N^2}\sum_{i=1}^Nf_{\boldsymbol{\theta}}\left(\vx_i\right)=\frac{1}{N}\sum_{i=1}^Nf_{\boldsymbol{\theta}}\left(\vx_i\right)a_i-\frac{K}{N}\widebar{f_{\boldsymbol{\theta}}\left(\vx\right)}.
        \label{eqn:cov_zafar}
    \end{equation}
    Observe that 
    \begin{align}
        \begin{split}
            -\tilde{\mu}_{\mathrm{SPD}}=&\frac{\sum_{i=1}^Nf_{\boldsymbol{\theta}}\left(\vx_i\right)a_i}{\sum_{i=1}^Na_i}-\frac{\sum_{i=1}^Nf_{\boldsymbol{\theta}}\left(\vx_i\right)\left(1-a_i\right)}{\sum_{i=1}^N\left(1-a_i\right)}=\frac{1}{K}\sum_{i=1}^Nf_{\boldsymbol{\theta}}\left(\vx_i\right)a_i-\frac{N}{N-K}\widebar{f_{\boldsymbol{\theta}}\left(\vx\right)}-\\
            &\frac{1}{N-K}\sum_{i=1}^Nf_{\boldsymbol{\theta}}\left(\vx_i\right)a_i=\frac{N}{K(N-K)}\sum_{i=1}^Nf_{\boldsymbol{\theta}}\left(\vx_i\right)a_i-\frac{NK}{K(N-K)}\widebar{f_{\boldsymbol{\theta}}\left(\vx\right)}.
        \end{split}
        \label{eqn:spd__rewritten}
    \end{align}
    Note that (\ref{eqn:cov_zafar}) $\propto$ (\ref{eqn:spd__rewritten}) by a factor of $\frac{N^2}{K(N-K)}$, constant in $\boldsymbol{\theta}$.
\end{proof}

\begin{lemma}
    For $\mathcal{X}=\left\{\vx_i\right\}_{i=1}^N$, $\mathcal{Y}=\left\{y_i\right\}_{i=1}^N$, and $\mathcal{A}=\left\{a_i\right\}_{i=1}^N$ and some classifier $f_{\boldsymbol{\theta}}(\cdot)$,  $-\tilde{\mu}_{\mathrm{EOD}}\left(f_{\boldsymbol{\theta}},\,\mathcal{X},\,\mathcal{Y},\,\mathcal{A}\right)\propto\widehat{\mathrm{Cov}}\left(A,\,f_{\boldsymbol{\theta}}\left(\mX\right)\,|\,Y=1\right)$.
\end{lemma}

\begin{proof}
    Recall that, by the law of total covariance, 
    \[
    \begin{split}
        \mathrm{Cov}\left(A,\,f_{\boldsymbol{\theta}}\left(\mX\right)\,|\,Y=1\right)=&\mathbb{E}\left[\left(A-\mathbb{E}\left[A\,|\,Y=1\right]\right)\left(f_{\boldsymbol{\theta}}\left(\mX\right)-\mathbb{E}\left[f_{\boldsymbol{\theta}}\left(\mX\right)\,|\,Y=1\right]\right)\,|\,Y=1\right]=\\ &\mathbb{E}\left[Af_{\boldsymbol{\theta}}\left(\mX\right)\,|\,Y=1\right]-\mathbb{E}\left[A\,|\,Y=1\right]\mathbb{E}\left[f_{\boldsymbol{\theta}}\left(\mX\right)\,|\,Y=1\right].    
    \end{split}
    \]
    Let $M=\sum_{i=1}^Ny_i$, $R=\sum_{i=1}^Na_iy_i$, and $\widebar{f_{\boldsymbol{\theta}}\left(\vx\right)}=\frac{1}{N}\sum_{i}^Nf_{\boldsymbol{\theta}}\left(\vx_i\right)$, consider an empirical estimate
    \begin{align}
        \begin{split}
            \widehat{\mathrm{Cov}}\left(A,\,f_{\boldsymbol{\theta}}\left(\mX\right)\,|\,Y=1\right)=&\frac{\sum_{i=1}^Nf_{\boldsymbol{\theta}}\left(\vx_i\right)a_iy_i}{\sum_{i=1}^Ny_i}-\frac{\sum_{i=1}^Na_iy_i}{\sum_{i=1}^Ny_i} \cdot \frac{\sum_{i=1}^Nf_{\boldsymbol{\theta}}\left(\vx_i\right)y_i}{\sum_{i=1}^Ny_i}=\\
            &\frac{1}{M}\sum_{i=1}^Nf_{\boldsymbol{\theta}}\left(\vx_i\right)a_iy_i-\frac{R}{M^2}\sum_{i=1}^Nf_{\boldsymbol{\theta}}\left(\vx_i\right)y_i.
        \end{split}
        \label{eqn:cov_eod}
    \end{align}
    Observe that 
    \begin{align}
        \begin{split}
            -\tilde{\mu}_{\mathrm{EOD}}=&\frac{\sum_{i=1}^Nf_{\boldsymbol{\theta}}\left(\vx_i\right)y_ia_i}{\sum_{i=1}^Ny_ia_i}-\frac{\sum_{i=1}^Nf_{\boldsymbol{\theta}}\left(\vx_i\right)y_i(1-a_i)}{\sum_{i=1}^Ny_i(1-a_i)}=\\
            &\frac{1}{R}\sum_{i=1}^Nf_{\boldsymbol{\theta}}\left(\vx_i\right)y_ia_i-\frac{1}{M-R}\sum_{i=1}^Nf_{\boldsymbol{\theta}}\left(\vx_i\right)y_i-\frac{1}{M-R}\sum_{i=1}^Nf_{\boldsymbol{\theta}}\left(\vx_i\right)y_ia_i=\\
            &\frac{M}{R(M-R)}\sum_{i=1}^Nf_{\boldsymbol{\theta}}\left(\vx_i\right)y_ia_i-\frac{1}{M-R}\sum_{i=1}^Nf_{\boldsymbol{\theta}}\left(\vx_i\right)y_i.
        \end{split}
        \label{eqn:eod__rewritten}
    \end{align}
    Note that (\ref{eqn:cov_eod}) $\propto$ (\ref{eqn:eod__rewritten}) by a factor of $\frac{M^2}{R(M-R)}$, constant in $\boldsymbol{\theta}$.
\end{proof}

\section{Datasets\label{app:data}}

Nonclinical tabular benchmarking datasets used in our experiments (see \Secref{sec:results_tabular}) are publicly available in the IBM AIF 360 library \citep{Bellamy2018}. Table~\ref{tab:datasets} below summarises all real-world datasets considered throughout the paper.

\begin{table}[H]
    \scriptsize
    \centering
    \caption{Summary of the datasets. $N_\mathrm{train}$, $N_\mathrm{valid}$, and $N_\mathrm{test}$ are the sizes of the training, validation, and test sets, respectively; $D$ is the input dimensionality after pre-processing; and $A$ is the protected attribute.\label{tab:datasets}}
    \begin{tabular}{p{2.1cm}p{0.7cm}p{0.7cm}p{0.7cm}p{1.0cm}p{2cm}p{1.8cm}}
        \toprule
        \textbf{Dataset} & \textbf{$N_{\mathrm{train}}$} & \textbf{$N_{\mathrm{valid}}$} & \textbf{$N_{\mathrm{test}}$} & \textbf{$D$} & \textbf{$A$} & \textbf{Architecture} \\
        \toprule
        \textbf{Adult} & 27,133 & 9,044 & 9,045 & 98 & \textit{Sex} & FCNN\\
        \specialrule{0.1pt}{0.5pt}{0.5pt}
        \textbf{Bank} & 18,292 & 6,098 & 6,098 & 57 & \textit{Age} & FCNN\\
        \specialrule{0.1pt}{0.5pt}{0.5pt}
        \textbf{COMPAS} & 3,700 & 1,233 & 1,234 & 401 & \textit{Race} & FCNN\\
        \specialrule{0.1pt}{0.5pt}{0.5pt}
        \multirow{3}{*}{\textbf{MIMIC-III}} & 21,595 & 7,199 & 7,199 & 43 & \textit{Age} & FCNN \\
        & 21,595 & 7,199 & 7,199 & 44 & \textit{Marital Status} & FCNN \\
        & 21,595 & 7,199 & 7,199 & 44 & \textit{Insurance Type} & FCNN \\
        \specialrule{0.1pt}{0.5pt}{0.5pt}
        \multirow{2}{*}{\textbf{MIMIC-CXR}} & 5,528 & 3,368 & 3,426 & 224$\times$224 & \textit{Sex} & CNN \\
        & 3,984 & 930 & 1,122 & 224$\times$224 & \textit{Ethnicity} & CNN \\
        \bottomrule
    \end{tabular}
\end{table}

\section{Syntehtic Data\label{app:synth}}

In addition to the real-world data (see \Secref{sec:expsetup} and Appendix~\ref{app:data}), we conducted experiments on two synthetic datasets adapted from the literature (for the results, see Appendix~\ref{app:sensitivity}). Below we summarise their generative process. 

\paragraph{Synthetic by \citet{Loh2019}} 

\citet{Loh2019} performed extensive simulation experiments comparing subgroup identification methods. Their simulation models are suitable for benchmarking debiasing algorithms. We adopted one of their synthetic datasets with the following generative procedure. For $N$ independent data points: 
\begin{enumerate}
    \item Randomly draw features with marginal distributions given by $X_{1,\,2,\,3,\,7,\,8,\,9,\,10}\sim\mathcal{N}\left(0,\,1\right)$, $X_4\sim\mathrm{Exp}(1)$, $X_5\sim\mathrm{Bernoulli}(\frac{1}{2})$, $X_6\sim\mathrm{Cat}(10)$ and $\mathrm{corr}\left(X_2,X_3\right)=0.5$ and $\mathrm{corr}\left(X_j,X_k\right)=0.5$, for $j,\,k\in\left\{7,\,8,\,9,\,10\right\}$, $j\neq k$.
    \item Randomly draw the protected attribute $A\sim\mathrm{Bernoulli}\left(\frac{1}{2}\right)$.
    \item Let 
    \begin{equation}
        \mathrm{logit}=\log\frac{\mathbb{P}(Y=1)}{\mathbb{P}(Y=0)}=\frac{1}{2}\left(X_1+X_2-X_5\right)+2\alpha A\mathbf{1}_{\left\{X_6\Mod{2}=1\right\}},
        \label{eqn:lohetal}
    \end{equation}
    where $\mathbf{1}_{\{\cdot\}}$ is an indicator function and $\alpha>0$ is the parameter controlling the magnitude of the correlation between $Y$ and $A$.
    \item Randomly draw the binary classification label $Y\sim\mathrm{Bernoulli}\left(\frac{\exp\left(\mathrm{logit}\right)}{\exp\left(\mathrm{logit}\right)+1}\right)$.
\end{enumerate}
Although this dataset is relatively simplistic, the simulation allows controlling the magnitude of classification disparity in the original classifier. In practice, we observe that the higher the value of $\alpha$, the higher the absolute SPD or EOD of the classifier trained on features $X_{1:10}$ and labels $Y$.

\paragraph{Synthetic by \citet{Zafar2017}}

\citet{Zafar2017} proposed another simple simulation model for generating datasets with different degrees of disparity in classification outcomes. We extended their model\footnote{\url{https://github.com/mbilalzafar/fair-classification}} to higher dimensionality and classes that are not linearly separable. The data generating process is specified below. For $N$ independent data points:
\begin{enumerate}
    \item Randomly draw the binary classification label $Y\sim\textrm{Bernoulli}\left(\frac{1}{2}\right)$.
    \item If $Y=0$, randomly draw $\tilde{\mX}\sim\mathcal{N}_2\left(\begin{bmatrix}-2\\-2\end{bmatrix},\,\begin{bmatrix}10 & 1\\ 1 & 3\end{bmatrix}\right)$; \\ otherwise
    \mbox{$\tilde{\mX}\sim\mathcal{N}_2\left(\begin{bmatrix}2\\2\end{bmatrix},\,\begin{bmatrix}5 & 1\\ 1 & 5\end{bmatrix}\right)$}.
    \item Let 
    \begin{equation}
        \tilde{\mX}'=\begin{bmatrix}\cos(\vartheta) & -\sin(\vartheta) \\ \sin(\vartheta) & \cos(\vartheta) \end{bmatrix}\tilde{\mX},
        \label{eqn:zafaretal}
    \end{equation}
    where $\vartheta$ is the rotation angle controlling the correlation between between $Y$ and $A$.
    \item Let $$\mathbb{P}\left(A=1\right)=\frac{p\left(\tilde{\mX}'\,|\,Y=1\right)}{p\left(\tilde{\mX}'\,|\,Y=1\right)+p\left(\tilde{\mX}'\,|\,Y=0\right)}.$$
    \item Randomly draw the protected attribute $A\sim\textrm{Bernoulli}\left(\mathbb{P}\left(A=1\right)\right)$.
    \item Let $g(\tilde{\vx})=\boldsymbol{\Theta}_2\textrm{ReLU}\left(\boldsymbol{\Theta}_1\textrm{ReLU}\left(\boldsymbol{\Theta}_0\tilde{\vx} + \vb_0\right) + \vb_1\right) + \vb_2$, where $\boldsymbol{\Theta}_0\in\mathbb{R}^{h\times 2}$, $\boldsymbol{\Theta}_1\in\mathbb{R}^{h\times h}$, $\boldsymbol{\Theta}_2\in\mathbb{R}^{p\times h}$ and $\vb_0\in\mathbb{R}^h$, $\vb_1\in\mathbb{R}^h$, $\vb_2\in\mathbb{R}^p$ are randomly generated matrices and vectors.
    \item Let $\mX=g\left(\tilde{\mX}\right)$ be a $p$-dimensional real-valued feature vector.
\end{enumerate}

Similar to the dataset above, this simulation allows controlling the degree of bias by adjusting the parameter $\vartheta$. In practice, values of $\vartheta$ closer to zero result in classifiers with higher absolute SPDs and EODs.

\section{Implementation Details\label{app:imp}}

In this appendix, we present further implementation details. For a general overview of the experimental setup, see \Secref{sec:expsetup} of the main text.

\subsection{Train-validation-test Split}

All tabular datasets were split into 60\% train, 20\% validation, and 20\% test instances. For MIMIC-CXR, we performed a 50\%-25\%-25\% train-validation-test split stratified by \emph{patients} to avoid data leakage. To mitigate other sources of variability in the original models, we balanced the training, validation, and test sets w.r.t. the number of healthy and pathological cases. Since the prior work \citep{Larrazabal2020} has investigated imbalance w.r.t. the protected attribute as a potential cause of bias in deep chest X-ray classifiers, we sampled the images of patients so that the training set had 75\% and 25\% from the privileged and unprivileged groups, respectively. For validation and test sets, the ratio of privileged and unprivileged groups was 50\%-50\%.

\subsection{Model Development}

All experiments and methods were implemented in PyTorch (v 1.9.1) \citep{Paszke2017}. For all tabular datasets, we used the same architecture and training scheme for the classifier $f_{\boldsymbol{\theta}}(\cdot)$. We trained a fully connected feedforward neural network with ten hidden layers, 32 units each, ReLU activations, dropout ($p=0.05$), and batch normalisation (see Table~\ref{tab:arch}). The network was trained for 1,000 epochs with early stopping by minimising the binary cross-entropy loss using the Adam optimiser \citep{Kingma2015} with \texttt{ReduceLROnPlateau} learning rate schedule and mini-batch size of 64. 

\begin{table}[H]
    \caption{Fully connected neural network architecture used in debiasing experiments on tabular data (see \Secref{sec:results_tabular}). \texttt{nn} stands for \texttt{torch.nn}; \texttt{F} stands for \texttt{torch.nn.functional}; \texttt{input\_dim} corresponds to the number of features $d$. \label{tab:arch}}
    \scriptsize
    \centering
    \begin{tabular}{ll}
        \toprule
        & \textbf{Classifier} \\
        \midrule
        \textbf{1} & \texttt{nn.Linear(input\_dim, 32)} \\
        & \texttt{F.relu()} \\
        & \texttt{nn.Dropout(0.05)} \\
        & \texttt{nn.BatchNorm1d(32)} \\
        \textbf{2} & \texttt{for l in range(10):} \\
        & \texttt{\;\;\; nn.Linear(32, 32)}\\
        & \texttt{\;\;\; F.relu()}\\
        & \texttt{\;\;\; nn.Dropout(0.05)}\\
        & \texttt{\;\;\; nn.BatchNorm1d(32)}\\
        \textbf{3} & \texttt{out = nn.Linear(32, 1)} \\
        \textbf{4} & \texttt{torch.sigmoid()} \\
        \bottomrule
    \end{tabular}
\end{table}

For MIMIC-CXR, we used classifiers based on the VGG-16 \citep{Simonyan2014} and ResNet-18 \citep{He2016} networks. We initialised classifiers with pre-trained weights and trained them using the binary cross-entropy loss and the AdamW optimiser \citep{Loshchilov2019} with default parameters, a mini-batch size of 32, and an initial learning rate of $10^{-4}$ with \texttt{StepLR} learning rate schedule. The networks were trained for a maximum of 20 epochs with early stopping on the validation set.

\subsection{Method Implementation}

We used the following implementation of the debiasing algorithms:
\begin{itemize}
    \item \textsc{Random} and \textsc{Adv. Intra}: we used the original implementation by \citet{Savani2020} available at \url{https://github.com/abacusai/intraprocessing_debiasing}.
    \item \textsc{ROC} and \textsc{Eq. Odds}: we used the implementation available in the AIF 360 toolkit at \url{https://github.com/Trusted-AI/AIF360}.
    \item \textsc{Pruning} and \textsc{Bias GD/A}: the PyTorch (v 1.9.1) \citep{Paszke2017} implementation is available at \url{https://github.com/i6092467/diff-bias-proxies}.
\end{itemize}

\subsection{Hyperparameters}

Table~\ref{tab:hyperparams} provides hyperparameter values for the pruning and bias GD/A. For both algorithms, the most sensitive hyperparameter is the lower bound on performance $\varrho$ (see Algorithms~\ref{alg:fair_pruning} and \ref{alg:biasGD}), which effectively controls the decrease in performance as a result of debiasing. For \mbox{MIMIC-III}, the same hyperparameter configuration was used across all protected attributes (``\emph{age}'', ``\emph{marital status}'', and ``\emph{insurance type}''). For MIMIC-CXR experiments with VGG-16, we performed pruning only in the convolutional layers. For ResNet-18, we only pruned the first \texttt{conv1} block, comprising a single convolutional layer. During experimentation, we observed little gain from pruning additional downstream layers. 

\begin{table}[H]
    \centering
    \caption{Hyperparameter values used for the \mbox{(\emph{a}) pruning} and \mbox{(\emph{b}) bias} GD/A algorithms throughout the experiments. Herein, ``\# units'' corresponds to the number of units pruned per step and relates to the hyperparameter $B$ from Algorithm~\ref{alg:fair_pruning}. $\varrho_{\mathrm{SPD}}$ and $\varrho_{\mathrm{EOD}}$ denote lower bounds on the balanced accuracy for the SPD and EOD experiments, respectively. For the bias GD/A, $\eta$ denotes the learning rate; $M$ is the mini-batch size; and $E$ is the number of epochs. \label{tab:hyperparams}}
    
    \subtable[\centering \textsc{Pruning}]{
        \scriptsize
        \begin{tabular}{p{1.9cm}p{1.2cm}p{0.65cm}p{0.65cm}}
            \toprule
            \textbf{Dataset} & \textbf{\# Units} & \textbf{$\varrho_{\mathrm{SPD}}$} & \textbf{$\varrho_{\mathrm{EOD}}$} \\
            \toprule
            Adult & 1 & 0.52 & 0.75 \\
            \specialrule{0.1pt}{0.5pt}{0.5pt}
            Bank & 1 & 0.80 & 0.70 \\
            \specialrule{0.1pt}{0.5pt}{0.5pt}
            COMPAS & 1 & 0.55 & 0.55 \\
            \specialrule{0.1pt}{0.5pt}{0.5pt}
            MIMIC-III & 1 & 0.60 & 0.60 \\
            \specialrule{0.1pt}{0.5pt}{0.5pt}
            \makecell[l]{MIMIC-CXR,\\ \textit{Sex}} & 1,500 & --- & 0.65 \\
            \specialrule{0.1pt}{0.5pt}{0.5pt}
            \makecell[l]{MIMIC-CXR,\\ \textit{Ethnicity}} & 1,500 & --- & 0.55 \\
            \bottomrule
        \end{tabular}
    }
    \subtable[\centering \textsc{Bias GD/A}]{
        \scriptsize
        \begin{tabular}{p{1.9cm}p{0.8cm}p{0.3cm}p{0.3cm}p{0.65cm}p{0.65cm}}
            \toprule
            \textbf{Dataset} & \textbf{$\eta$} & \textbf{$M$} & \textbf{$E$} & \textbf{$\varrho_{\mathrm{SPD}}$} & \textbf{$\varrho_{\mathrm{EOD}}$} \\
            \toprule
            Adult & 1.0e-5 & 256 & 200 & 0.62 & 0.80 \\
            \specialrule{0.1pt}{0.5pt}{0.5pt}
            Bank & 1.0e-5 & 256 & 200 & 0.70 & 0.70 \\
            \specialrule{0.1pt}{0.5pt}{0.5pt}
            COMPAS & 1.0e-5 & 256 & 200 & 0.61 & 0.58 \\
            \specialrule{0.1pt}{0.5pt}{0.5pt}
            MIMIC-III & 1.0e-5 & 256 & 100 & 0.60 & 0.60 \\
            \specialrule{0.1pt}{0.5pt}{0.5pt}
            \makecell[l]{MIMIC-CXR,\\ \textit{Sex}} & 1.0e-5 & 32 & 10 & --- & 0.75 \\
            \specialrule{0.1pt}{0.5pt}{0.5pt}
            \makecell[l]{MIMIC-CXR,\\ \textit{Ethnicity}} & 7.5e-6 & 32 & 8 & --- & 0.55 \\
            \bottomrule
        \end{tabular}
    }
\end{table}

For the random perturbation intra-processing, we used multiplicative noise distributed as $\mathcal{N}\left(1,\,0.01\right)$ and performed 101 perturbations to maximise the constrained objective proposed by \citet{Savani2020} with an upper/lower bound on the bias of $\pm$0.05 and a margin of 0.01. The same bias bounds were used for the ROC post-processing procedure.
For adversarial intra-processing, on tabular datasets, we used hyperparameter values from the original work by \citet{Savani2020}: a critic network with three hidden layers, a learning rate of 10\textsuperscript{-3}, 16 training epochs, a mini-batch size of 64, $\lambda=0.75$, 201 and 101 critic and actor training steps, respectively. Similarly to the random perturbation, we utilised the constrained objective with an upper/lower bound on the bias of $\pm$0.05 and a margin of 0.01. Different from the tabular experiments, for MIMIC-CXR and the protected attribute ``\emph{sex}'', we used a learning rate of 10\textsuperscript{-4}, five training epochs, an upper/lower bound on the bias of $\pm$0.03, and a margin of 0.02. For MIMIC-CXR and ``\emph{ethnicity}'', we trained networks for four epochs under the constrained objective with an upper/lower bound of $\pm$0.05, and a margin of 0.01. We attempted tuning the number of epochs, critic and actor steps, and $\lambda$, however, we did not observe improved results in both MIMIC-CXR experiments.

\section{Further Results\label{app:results}}

\subsection{Further Quantitative Results}

\begin{table}[H]
    \centering
    \caption{Equal opportunity difference and balanced accuracy attained on the \emph{validation} set before and after debiasing VGG-16 (\emph{a},\emph{b}) and ResNet-18 (\emph{c},\emph{d}) trained on \mbox{MIMIC-CXR} to predict \mbox{(\emph{a},\emph{c}) enlarged} cardiomediastinum (with the protected attribute ``\emph{sex}'') and \mbox{(\emph{b},\emph{d}) pneumonia} (with the protected attribute ``\emph{ethnicity}''). Test set results can be found in Table~\ref{tab:mimic_results}, \Secref{sec:results_mimiccxr}.\label{tab:mimic_results_val}}
    
    \subtable[\centering Enlarged CM, \textit{Sex}; VGG-16]{
        \scriptsize
        \begin{tabular}{p{1.7cm}p{1.6cm}p{1.6cm}}
            \toprule
            \textbf{Method} & \textbf{EOD} & \textbf{BA} \\
            \toprule
            \cellcolor{Gainsboro}\textsc{Standard} & \cellcolor{Gainsboro}\tentry{-0.06}{0.03} & \cellcolor{Gainsboro}\tentry{0.77}{0.01} \\
            \textsc{Random} & \tentry{-0.02}{0.01} & \tentry{0.76}{0.01} \\
            \textsc{ROC} & \tentry{-0.05}{0.01} & \tentry{0.74}{0.04} \\
            \textsc{Eq. Odds} & \tentry{\;0.00}{0.01} & \tentry{0.75}{0.01} \\
            \textsc{Adv. Intra} & \tentry{-0.01}{0.01} & \tentry{0.97}{0.02} \\
            \underline{\textsc{Pruning}} & \tentry{\;0.00}{0.03} & \tentry{0.76}{0.01} \\
            \underline{\textsc{Bias GD/A}} & \tentry{-0.01}{0.02} & \tentry{0.76}{0.01} \\
            \bottomrule
        \end{tabular}
    }
    \subtable[\centering Pneumonia, \textit{Ethnicity}; VGG-16]{
        \scriptsize
        \begin{tabular}{p{1.7cm}p{1.6cm}p{1.6cm}}
            \toprule
            \textbf{Method} & \textbf{EOD} & \textbf{BA} \\
            \toprule
            \cellcolor{Gainsboro}\textsc{Standard} & \cellcolor{Gainsboro}\tentry{-0.14}{0.05} & \cellcolor{Gainsboro}\tentry{0.75}{0.01} \\
            \textsc{Random} & \tentry{-0.07}{0.04} & \tentry{0.73}{0.02} \\
            \textsc{ROC} & \tentry{-0.05}{0.01} & \tentry{0.65}{0.06} \\
            \textsc{Eq. Odds} & \tentry{\;0.00}{0.01} & \tentry{0.71}{0.02} \\
            \textsc{Adv. Intra} & \tentry{-0.06}{0.05} & \tentry{0.93}{0.05} \\
            \underline{\textsc{Pruning}} & \tentry{\;0.01}{0.02} & \tentry{0.72}{0.03} \\
            \underline{\textsc{Bias GD/A}} & \tentry{\;0.00}{0.03} & \tentry{0.73}{0.02} \\
            \bottomrule
        \end{tabular}
    }
    
    \vspace{0.25cm}
    
    \subtable[\centering Enlarged CM, \textit{Sex}; ResNet-18]{
        \scriptsize
        \begin{tabular}{p{1.7cm}p{1.6cm}p{1.6cm}}
            \toprule
            \textbf{Method} & \textbf{EOD} & \textbf{BA} \\
            \toprule
            \cellcolor{Gainsboro}\textsc{Standard} & \cellcolor{Gainsboro}\tentry{-0.06}{0.04} & \cellcolor{Gainsboro}\tentry{0.76}{0.01} \\
            \textsc{Random}& \tentry{\;0.00}{0.01} & \tentry{0.74}{0.02} \\
            \textsc{ROC} & \tentry{-0.04}{0.01} & \tentry{0.74}{0.04} \\
            \textsc{Eq. Odds} & \tentry{\;0.00}{0.01} & \tentry{0.74}{0.01} \\
            \textsc{Adv. Intra} & \tentry{\;0.00}{0.01} & \tentry{0.99}{0.02} \\
            \underline{\textsc{Pruning}} & \tentry{\;0.00}{0.02} & \tentry{0.74}{0.02} \\
            \underline{\textsc{Bias GD/A}} & \tentry{\;0.00}{0.02} & \tentry{0.76}{0.01} \\
            \bottomrule
        \end{tabular}
    }
    \subtable[\centering Pneumonia, \textit{Ethnicity}; ResNet-18]{
        \scriptsize
        \begin{tabular}{p{1.7cm}p{1.6cm}p{1.6cm}}
            \toprule
            \textbf{Method} & \textbf{EOD} & \textbf{BA} \\
            \toprule
            \cellcolor{Gainsboro}\textsc{Standard} & \cellcolor{Gainsboro}\tentry{-0.13}{0.05} & \cellcolor{Gainsboro}\tentry{0.74}{0.01} \\
            \textsc{Random}& \tentry{-0.01}{0.01} & \tentry{0.67}{0.04} \\
            \textsc{ROC} & \tentry{-0.04}{0.01} & \tentry{0.65}{0.05} \\
            \textsc{Eq. Odds} & \tentry{\;0.00}{0.01} & \tentry{0.71}{0.02} \\
            \textsc{Adv. Intra} & \tentry{\;0.00}{0.00} & \tentry{1.00}{0.00} \\
            \underline{\textsc{Pruning}} & \tentry{\;0.00}{0.03} & \tentry{0.71}{0.02} \\
            \underline{\textsc{Bias GD/A}} & \tentry{\;0.00}{0.03} & \tentry{0.73}{0.02} \\
            \bottomrule
        \end{tabular}
    }
\end{table}


\subsection{Comparison with Adversarial In-processing\label{app:zhang}}

Since the proposed bias GD/A procedure (see Algorithm~\ref{alg:biasGD}) bears similarity to the adversarial in-processing method by \citet{Zhang2018}, we additionally evaluated models trained from scratch with an adversary for predicting the protected attribute based on the classifier's output as described by \citet{Zhang2018}. The evaluation was performed on the same datasets (see Table~\ref{tab:datasets}) and with the same setup as in the experiments from the main body of the paper. Notably, debiased models were trained directly on the training set and \emph{not} on the validation data, as for intra- and post-processing. For the tabular datasets, we used the implementation available in the AIF 360 toolkit \citep{Bellamy2018}. For the MIMIC-CXR, we adapted the publicly available implementation\footnote{\url{https://github.com/choprashweta/Adversarial-Debiasing}} by \citet{Chopra2020}.

Table~\ref{tab:zhang_results} reports bias and balanced accuracy of the models trained using adversarial in-processing across all dataset-protected-attribute pairs and network architectures. Encouragingly, the method by \citet{Zhang2018} performed comparably or slightly worse on average than pruning and bias GD/A (cf. Tables~\ref{tab:tabular_results}, \ref{tab:mimiciii_results}, and \ref{tab:mimic_results}). For MIMIC-CXR, we observed a pattern similar to intra-processing where the method failed to remove the bias associated with the attribute ``\emph{ethnicity}''. Thus, in the latter setting, post-processing, which adjusts the model’s predictions at test time based on the protected attribute value, is the only effective family of techniques considered.

\begin{table}[H]
    \centering
    \caption{Test-set bias and balanced accuracy attained by the networks trained using adversarial in-processing. Models were trained separately for the SPD and EOD. \label{tab:zhang_results}}
    \scriptsize
    \begin{tabular}{p{3.0cm}p{0.25cm}p{1.6cm}p{1.6cm}p{0.25cm}p{1.6cm}p{1.6cm}}
        \toprule
        \textbf{Experiment} &  & \textbf{SPD} & \textbf{BA} &  & \textbf{EOD} & \textbf{BA}\\
        \toprule
        \makecell[l]{\textbf{Adult}, \\ \textit{Sex}} & & \tentry{-0.02}{0.00} & \tentry{0.55}{0.01} &  & \tentry{\;0.03}{0.02} & \tentry{0.79}{0.01} \\
        \specialrule{0.1pt}{0.5pt}{0.5pt}
        \makecell[l]{\textbf{Bank}, \\ \textit{Age}} & & \tentry{\;0.06}{0.04} & \tentry{0.69}{0.07} &  & \tentry{-0.04}{0.06} & \tentry{0.86}{0.01} \\
        \specialrule{0.1pt}{0.5pt}{0.5pt}
        \makecell[l]{\textbf{COMPAS}, \\ \textit{Race}} & & \tentry{\;0.05}{0.04} & \tentry{0.60}{0.03} &  & \tentry{\;0.07}{0.04} & \tentry{0.62}{0.02} \\
        \specialrule{0.1pt}{0.5pt}{0.5pt}
        \makecell[l]{\textbf{MIMIC-III}, \\ \textit{Age}} & & \tentry{-0.04}{0.02} & \tentry{0.68}{0.04} &  & \tentry{\;0.06}{0.03} & \tentry{0.70}{0.02} \\
        \specialrule{0.1pt}{0.5pt}{0.5pt}
        \makecell[l]{\textbf{MIMIC-III}, \\ \textit{Marital Status}} & & \tentry{\;0.02}{0.03} & \tentry{0.74}{0.02} &  & \tentry{\;0.00}{0.04} & \tentry{0.73}{0.01} \\
        \specialrule{0.1pt}{0.5pt}{0.5pt}
        \makecell[l]{\textbf{MIMIC-III}, \\ \textit{Insurance Type}} & & \tentry{-0.02}{0.02} & \tentry{0.71}{0.03} &  & \tentry{\;0.07}{0.03} & \tentry{0.72}{0.02} \\
        \specialrule{0.1pt}{0.5pt}{0.5pt}
        \makecell[l]{\textbf{MIMIC-CXR}, \\ \textit{Sex}, VGG-16} & & --- & --- &  & \tentry{\;0.00}{0.03} & \tentry{0.75}{0.01} \\
        \specialrule{0.1pt}{0.5pt}{0.5pt}
        \makecell[l]{\textbf{MIMIC-CXR}, \\ \textit{Sex}, ResNet-18} & & --- & --- &  & \tentry{-0.01}{0.13} & \tentry{0.72}{0.03} \\
        \specialrule{0.1pt}{0.5pt}{0.5pt}
        \makecell[l]{\textbf{MIMIC-CXR}, \\ \textit{Ethnicity}, VGG-16} & & --- & --- &  & \tentry{-0.13}{0.05} & \tentry{0.71}{0.03} \\
        \specialrule{0.1pt}{0.5pt}{0.5pt}
        \makecell[l]{\textbf{MIMIC-CXR}, \\ \textit{Ethnicity}, ResNet-18} & & --- & --- &  & \tentry{-0.13}{0.07} & \tentry{0.71}{0.02} \\
        \bottomrule
    \end{tabular}
\end{table}

In summary, the results above suggest that, despite relying on the smaller validation set and resorting to editing the model's parameters \emph{post hoc}, for the considered datasets, intra-processing methods achieve the model's bias and performance that are comparable to those of the network retrained from scratch on the original training set. This experiment further supports the viability of the intra-processing approach adopted by us.

\subsection{Sensitivity to Initial Conditions\label{app:sensitivity}}

As observed before (see Section~\ref{sec:results_tabular}), the performance of the classifier debiased using pruning or bias GD/A can vary considerably, for instance, for the Adult dataset (see Table~\ref{tab:tabular_results}). To investigate the sensitivity of the proposed methods to initial conditions, particularly, to the degree of bias within the original classifier, we performed further experiments on two synthetic datasets described in Appendix~\ref{app:synth}. We trained and debiased FC neural networks (see Table~\ref{tab:arch}) while varying the correlation between the label and protected attribute. Intuitively, we expect debiasing to be less effective when the bias of the classifier is high.

For the dataset by \citet{Loh2019}, we trained and debiased classifiers under different values of the parameter $\alpha\in\left[0.0, 2.5\right]$ (see \Eqref{eqn:lohetal}). The resulting SPD varies between approximately 0.0 to 0.4, and the EOD is between 0.0 and 0.5. \mbox{Table~\ref{tab:results_sensitivity}}(\emph{a}) shows changes in the BA and SPD of the original classifier and the network obtained after pruning and bias GD/A. Notably, both methods exhibit similar patterns. For $\alpha\in[0.0, 1.5]$, debiased classifiers retain a BA of approximately 0.63, which corresponds to an unbiased performance and reduce the bias to zero with low variance. In contrast, for $\alpha>1.5$, the variance of the disparity across seeds increases considerably, e.g., for $\alpha=2.5$, pruning yields an SPD of 0.01$\pm$0.13. Similar patterns occur when debiasing w.r.t. the EOD (see Table~\mbox{\ref{tab:results_sensitivity}(\emph{b})}).

For the dataset by \citet{Zafar2017}, we varied the value of the parameter $\vartheta\in\left[0.7, 1.2\right]$ (see \Eqref{eqn:zafaretal}). Tables~\mbox{\ref{tab:results_sensitivity}(\emph{c-d})} contain the results across the range of rotation angles. Analogously to the synthetic dataset by \citet{Loh2019}, we observe either a decrease in the BA or an increase in the residual bias for lower values of the parameter $\vartheta$, i.e. under a higher initial bias. In summary, while proposed techniques successfully mitigate disparity, when the bias of the original classifier is relatively high, debiasing may either fail or lead to a considerable decrease in predictive performance.

\begin{table}[H]
    \centering
    \caption{Changes in the balanced accuracy and bias of the original and debiased classifier, given by the \mbox{SPD (\emph{a, c})} and \mbox{EOD (\emph{b, d})} across varying simulation parameters for the synthetic datasets by \citet{Loh2019} \mbox{(\emph{a-b})} and \citet{Zafar2017} \mbox{(\emph{c-d})}. Averages and standard deviations are reported across ten independent simulations. \label{tab:results_sensitivity}}
    
    \vspace{0.25cm}
    
    \subtable[\centering Synthetic by \citet{Loh2019}, SPD]{
        \scriptsize
        \begin{tabular}{p{0.5cm}p{1.8cm}p{1.8cm}p{1.8cm}p{1.8cm}p{1.8cm}p{1.8cm}}
            \toprule
            \textbf{$\alpha$} & \textbf{\textsc{Standard}, BA} & \textbf{\textsc{Pruning}, BA} & \textbf{\textsc{Bias GD/A}, BA} & \textbf{\textsc{Standard}, Bias} & \textbf{\textsc{Pruning}, Bias} & \textbf{\textsc{Bias GD/A}, Bias} \\
            \toprule
            0.1 & \tentry{0.63}{0.00} & \tentry{0.63}{0.01} & \tentry{0.62}{0.01} & \tentry{-0.05}{0.02} & \tentry{\;0.01}{0.02} & \tentry{\;0.00}{0.02} \\
            0.5 & \tentry{0.64}{0.01} & \tentry{0.63}{0.02} & \tentry{0.64}{0.01} & \tentry{-0.22}{0.02} & \tentry{-0.02}{0.03} & \tentry{-0.02}{0.02} \\
            1.0 & \tentry{0.68}{0.01} & \tentry{0.63}{0.02} & \tentry{0.66}{0.01} & \tentry{-0.35}{0.03} & \tentry{\;0.00}{0.01} & \tentry{-0.02}{0.03} \\
            1.5 & \tentry{0.71}{0.01} & \tentry{0.63}{0.02} & \tentry{0.66}{0.03} & \tentry{-0.39}{0.03} & \tentry{\;0.03}{0.05} & \tentry{-0.01}{0.04} \\
            2.0 & \tentry{0.73}{0.01} & \tentry{0.63}{0.04} & \tentry{0.56}{0.07} & \tentry{-0.39}{0.03} & \tentry{-0.03}{0.10} & \tentry{\;0.03}{0.08} \\
            2.5 & \tentry{0.73}{0.00} & \tentry{0.65}{0.03} & \tentry{0.57}{0.07} & \tentry{-0.41}{0.03} & \tentry{\;0.01}{0.13} & \tentry{\;0.04}{0.09} \\
            \bottomrule
        \end{tabular}
    }
    
    \vspace{0.25cm}
    
    \subtable[\centering Synthetic by \citet{Loh2019}, EOD]{
        \scriptsize
        \begin{tabular}{p{0.5cm}p{1.8cm}p{1.8cm}p{1.8cm}p{1.8cm}p{1.8cm}p{1.8cm}}
            \toprule
            \textbf{$\alpha$} & \textbf{\textsc{Standard}, BA} & \textbf{\textsc{Pruning}, BA} & \textbf{\textsc{Bias GD/A}, BA} & \textbf{\textsc{Standard}, Bias} & \textbf{\textsc{Pruning}, Bias} & \textbf{\textsc{Bias GD/A}, Bias} \\
            \toprule
            0.1 & \tentry{0.63}{0.00} & \tentry{0.62}{0.01} & \tentry{0.62}{0.01} & \tentry{-0.04}{0.02} & \tentry{\;0.01}{0.02} & \tentry{\;0.01}{0.02} \\
            0.5 & \tentry{0.64}{0.01} & \tentry{0.63}{0.01} & \tentry{0.64}{0.01} & \tentry{-0.22}{0.02} & \tentry{-0.01}{0.03} & \tentry{-0.02}{0.02} \\
            1.0 & \tentry{0.68}{0.01} & \tentry{0.64}{0.01} & \tentry{0.66}{0.01} & \tentry{-0.39}{0.05} & \tentry{-0.01}{0.02} & \tentry{-0.03}{0.02} \\
            1.5 & \tentry{0.71}{0.01} & \tentry{0.63}{0.01} & \tentry{0.64}{0.03} & \tentry{-0.45}{0.04} & \tentry{\;0.02}{0.04} & \tentry{-0.01}{0.03} \\
            2.0 & \tentry{0.73}{0.01} & \tentry{0.63}{0.03} & \tentry{0.63}{0.02} & \tentry{-0.47}{0.04} & \tentry{\;0.01}{0.05} & \tentry{\;0.00}{0.04} \\
            2.5 & \tentry{0.73}{0.00} & \tentry{0.62}{0.05} & \tentry{0.64}{0.05} & \tentry{-0.50}{0.05} & \tentry{\;0.04}{0.11} & \tentry{-0.04}{0.08} \\
            \bottomrule
        \end{tabular}
    }
    
    \vspace{0.25cm}
    
    \subtable[\centering Synthetic by \citet{Zafar2017}, SPD]{
        \scriptsize
        \begin{tabular}{p{0.5cm}p{1.8cm}p{1.8cm}p{1.8cm}p{1.8cm}p{1.8cm}p{1.8cm}}
            \toprule
            \textbf{$\vartheta$} & \textbf{\textsc{Standard}, BA} & \textbf{\textsc{Pruning}, BA} & \textbf{\textsc{Bias GD/A}, BA} & \textbf{\textsc{Standard}, Bias} & \textbf{\textsc{Pruning}, Bias} & \textbf{\textsc{Bias GD/A}, Bias} \\
            \toprule
            1.2 & \tentry{0.87}{0.00} & \tentry{0.84}{0.03} & \tentry{0.87}{0.00} & \tentry{-0.03}{0.01} & \tentry{-0.01}{0.02} & \tentry{-0.01}{0.01} \\
            1.1 & \tentry{0.87}{0.00} & \tentry{0.72}{0.04} & \tentry{0.83}{0.03} & \tentry{-0.13}{0.01} & \tentry{\;0.01}{0.02} & \tentry{-0.03}{0.04} \\
            1.0 & \tentry{0.87}{0.00} & \tentry{0.58}{0.03} & \tentry{0.74}{0.04} & \tentry{-0.23}{0.02} & \tentry{-0.09}{0.03} & \tentry{-0.02}{0.04} \\
            0.9 & \tentry{0.87}{0.00} & \tentry{0.58}{0.04} & \tentry{0.66}{0.04} & \tentry{-0.33}{0.01} & \tentry{-0.11}{0.04} & \tentry{-0.02}{0.03} \\
            0.8 & \tentry{0.87}{0.00} & \tentry{0.62}{0.06} & \tentry{0.59}{0.02} & \tentry{-0.42}{0.01} & \tentry{-0.18}{0.08} & \tentry{-0.02}{0.03} \\
            0.7 & \tentry{0.87}{0.00} & \tentry{0.60}{0.04} & \tentry{0.57}{0.02} & \tentry{-0.49}{0.01} & \tentry{-0.18}{0.07} & \tentry{-0.03}{0.03} \\
            \bottomrule
        \end{tabular}
    }
    
    \vspace{0.25cm}
    
    \subtable[\centering Synthetic by \citet{Zafar2017}, EOD]{
        \scriptsize
        \begin{tabular}{p{0.5cm}p{1.8cm}p{1.8cm}p{1.8cm}p{1.8cm}p{1.8cm}p{1.8cm}}
            \toprule
            \textbf{$\vartheta$} & \textbf{\textsc{Standard}, BA} & \textbf{\textsc{Pruning}, BA} & \textbf{\textsc{Bias GD/A}, BA} & \textbf{\textsc{Standard}, Bias} & \textbf{\textsc{Pruning}, Bias} & \textbf{\textsc{Bias GD/A}, Bias} \\
            \toprule
            1.2 & \tentry{0.87}{0.00} & \tentry{0.76}{0.06} & \tentry{0.87}{0.00} & \tentry{-0.05}{0.01} & \tentry{-0.01}{0.01} & \tentry{-0.01}{0.01} \\
            1.1 & \tentry{0.87}{0.00} & \tentry{0.76}{0.04} & \tentry{0.86}{0.02} & \tentry{-0.09}{0.01} & \tentry{-0.01}{0.01} & \tentry{-0.03}{0.03} \\
            1.0 & \tentry{0.87}{0.00} & \tentry{0.74}{0.04} & \tentry{0.86}{0.01} & \tentry{-0.12}{0.02} & \tentry{\;0.00}{0.03} & \tentry{-0.05}{0.03} \\
            0.9 & \tentry{0.87}{0.00} & \tentry{0.73}{0.04} & \tentry{0.84}{0.02} & \tentry{-0.16}{0.02} & \tentry{-0.03}{0.04} & \tentry{-0.05}{0.02} \\
            0.8 & \tentry{0.87}{0.00} & \tentry{0.77}{0.06} & \tentry{0.83}{0.03} & \tentry{-0.20}{0.02} & \tentry{-0.05}{0.06} & \tentry{-0.07}{0.05} \\
            0.7 & \tentry{0.87}{0.00} & \tentry{0.75}{0.05} & \tentry{0.83}{0.03} &  \tentry{-0.23}{0.02} & \tentry{-0.05}{0.07} & \tentry{-0.09}{0.03} \\
            \bottomrule
        \end{tabular}
    }
\end{table}

\subsection{Further Qualitative Results}

\begin{figure}[H]
    \centering
    \subfigure[{\scriptsize\centering $\;\;\;$ Adult, \textsc{Pruning},$\;\;\;$ SPD}]{
        \includegraphics[width=0.25\linewidth]{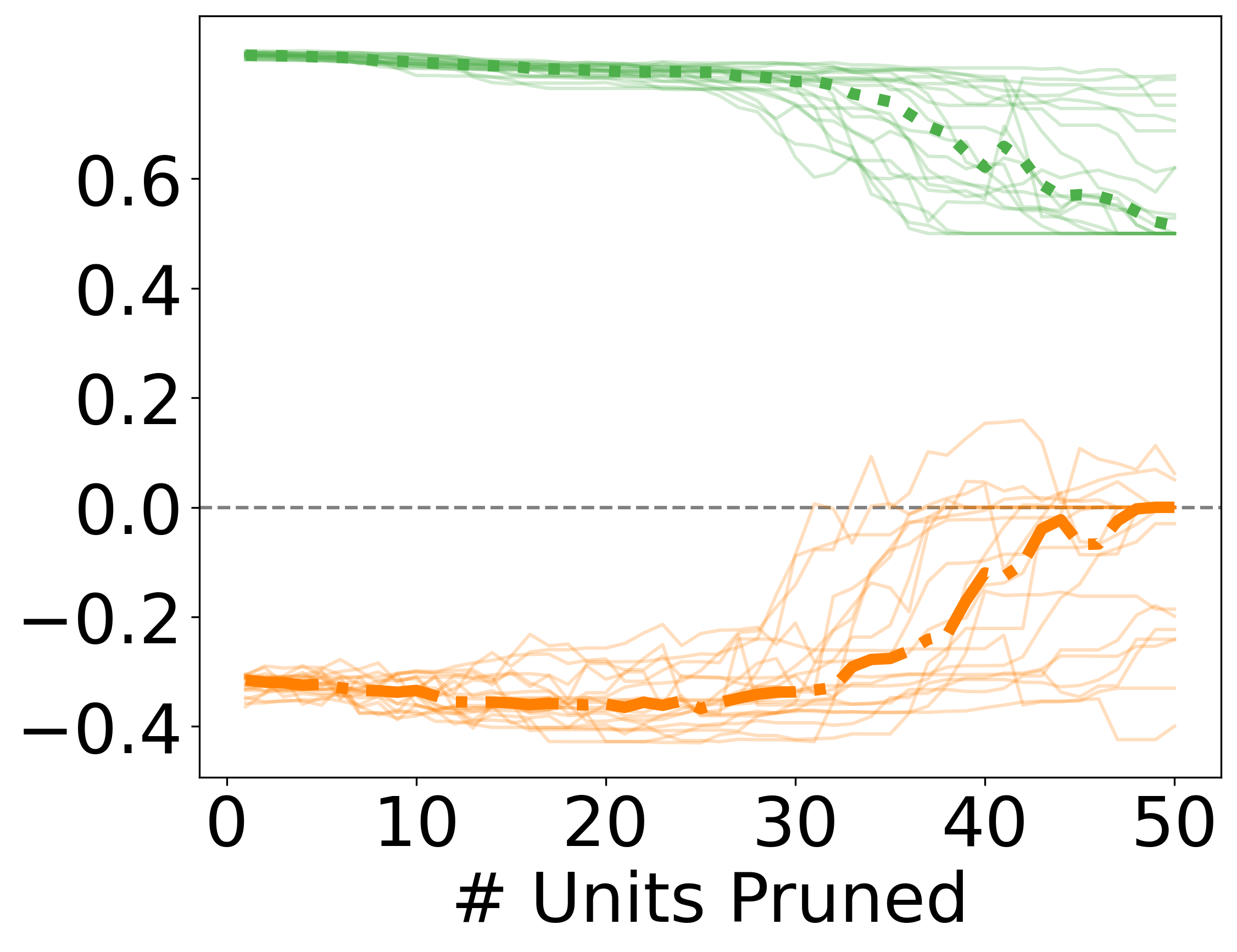}
    } 
    \subfigure[{\scriptsize\centering Adult, \textsc{Pruning}, EOD}]{  
        \includegraphics[width=0.21\linewidth]{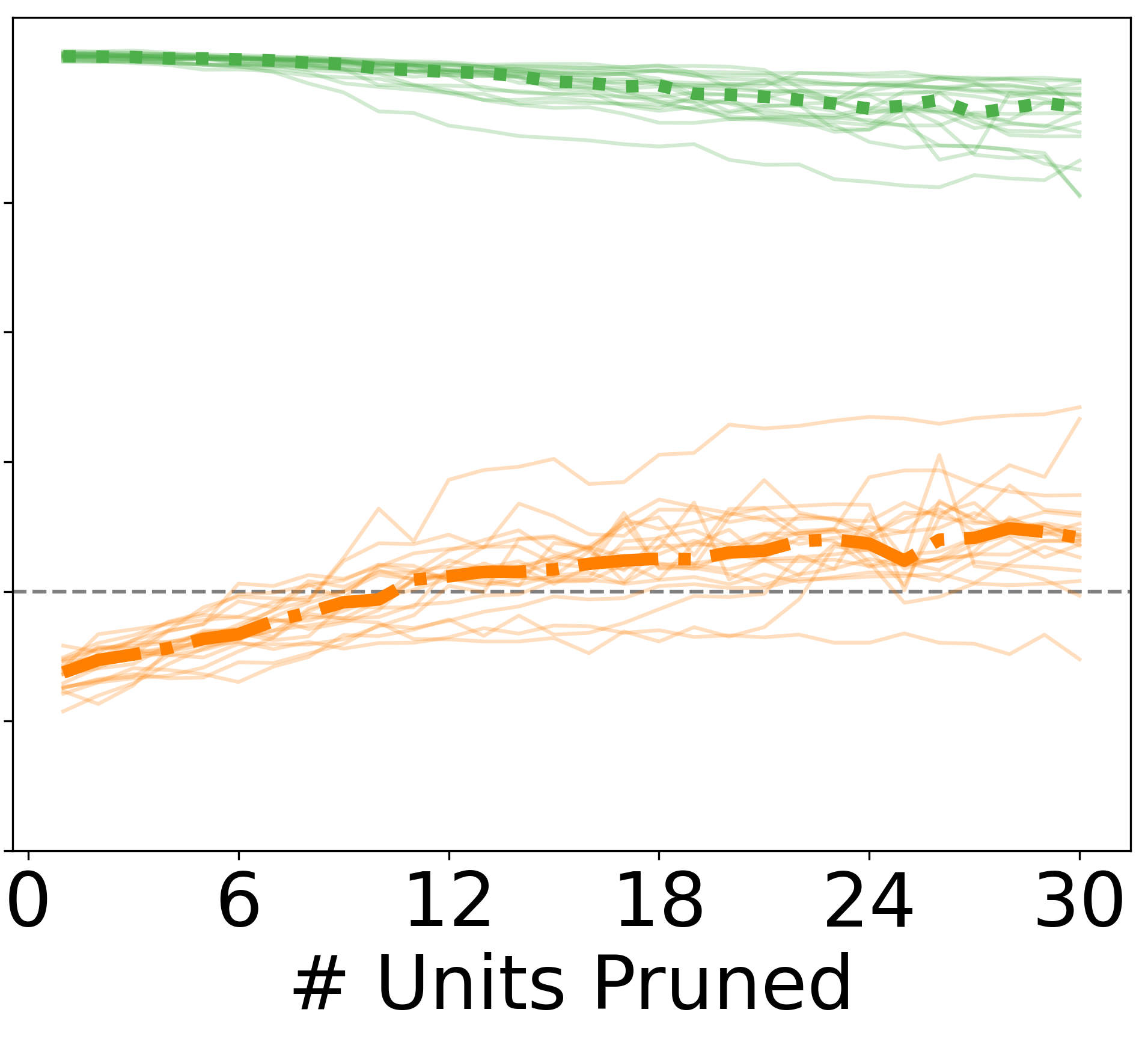}
    }
    \subfigure[{\scriptsize\centering Adult, \textsc{Bias GD/A}, SPD}]{
        \includegraphics[width=0.21\linewidth]{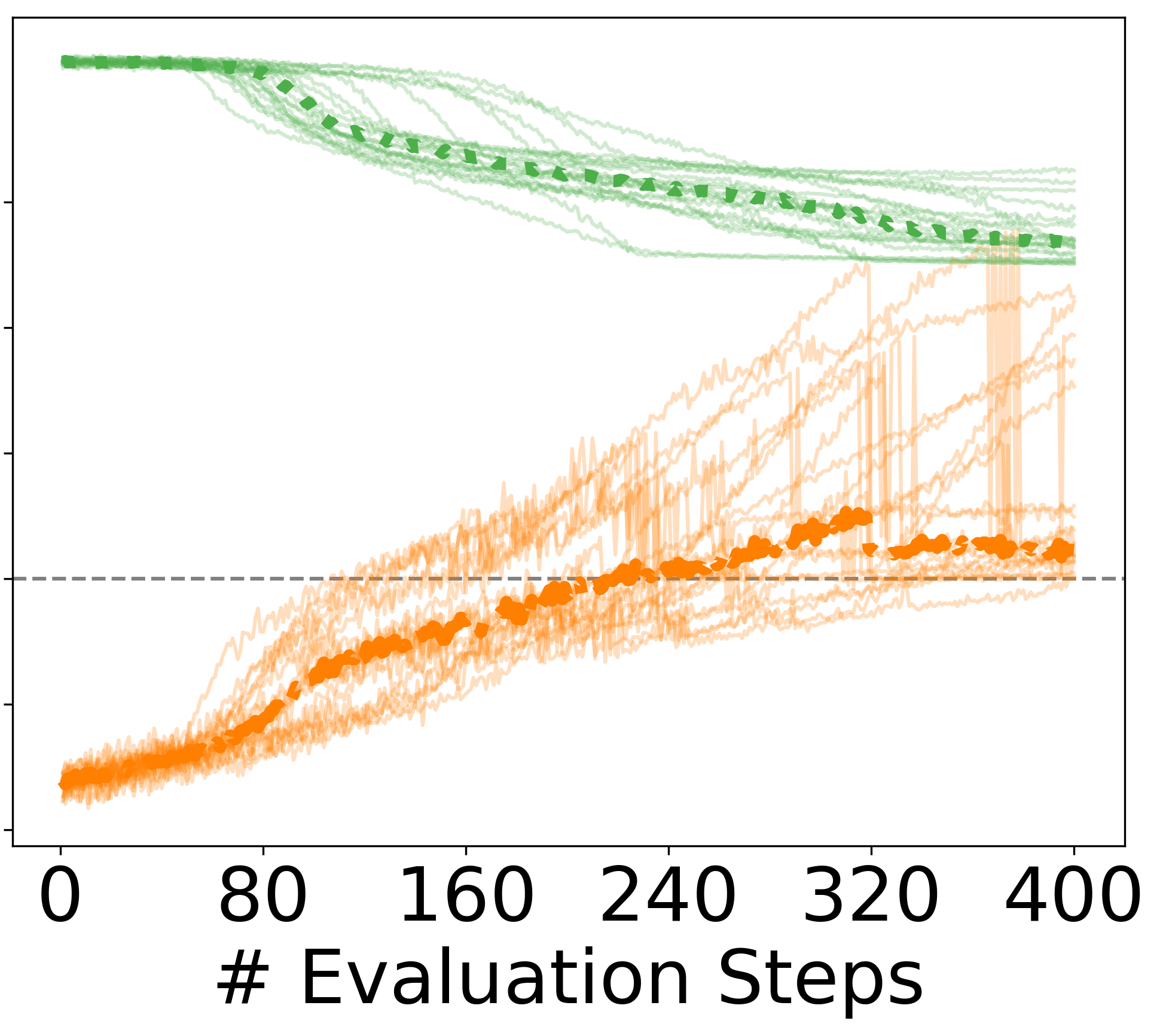}
    }
    \subfigure[{\scriptsize\centering Adult, \textsc{Bias GD/A}, EOD}]{        
        \includegraphics[width=0.21\linewidth]{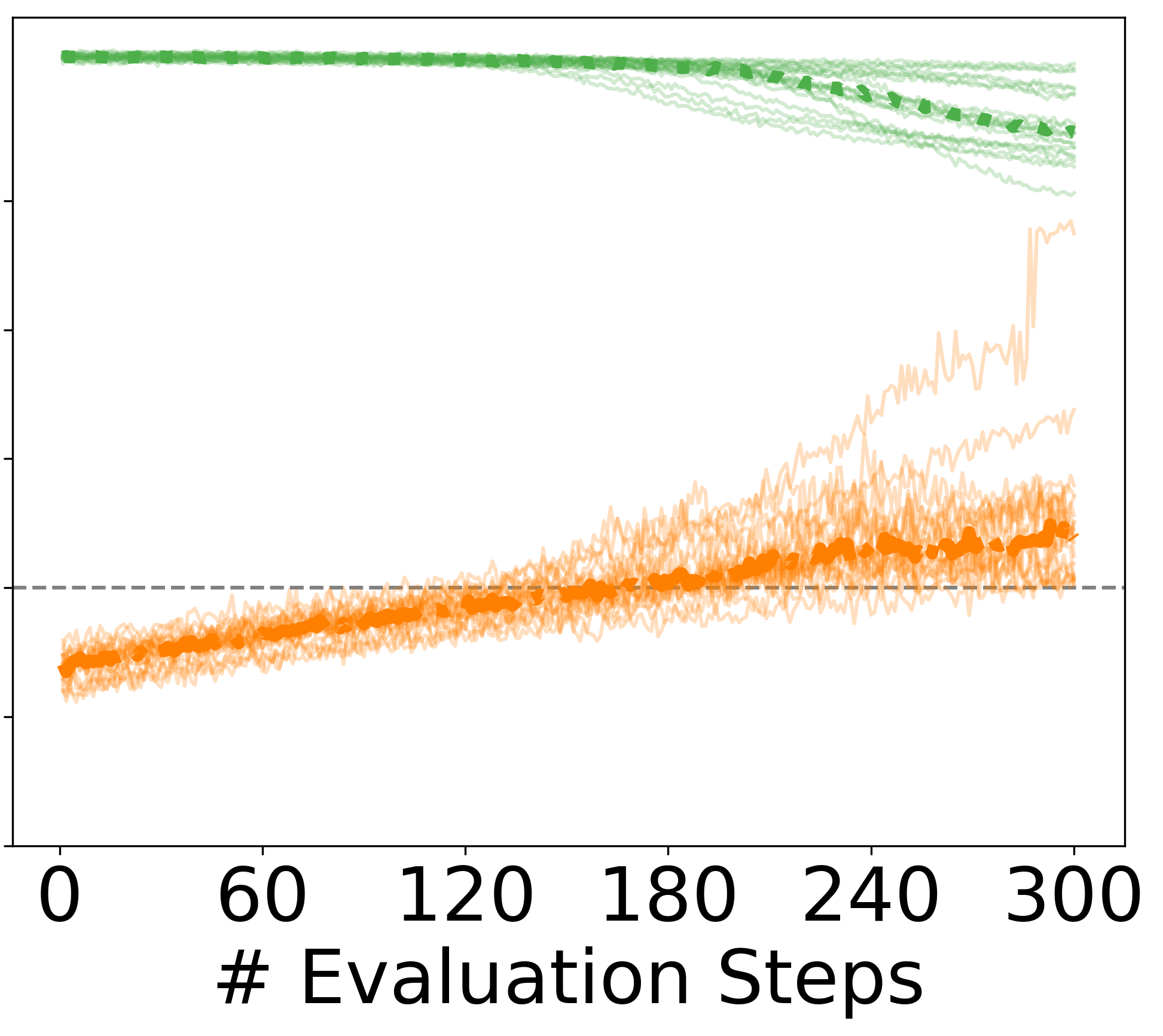}
    }
    \subfigure[{\scriptsize\centering $\;\;\;\;$ Bank, \textsc{Pruning},$\;\;\;\;$ SPD}]{
        \includegraphics[width=0.25\linewidth]{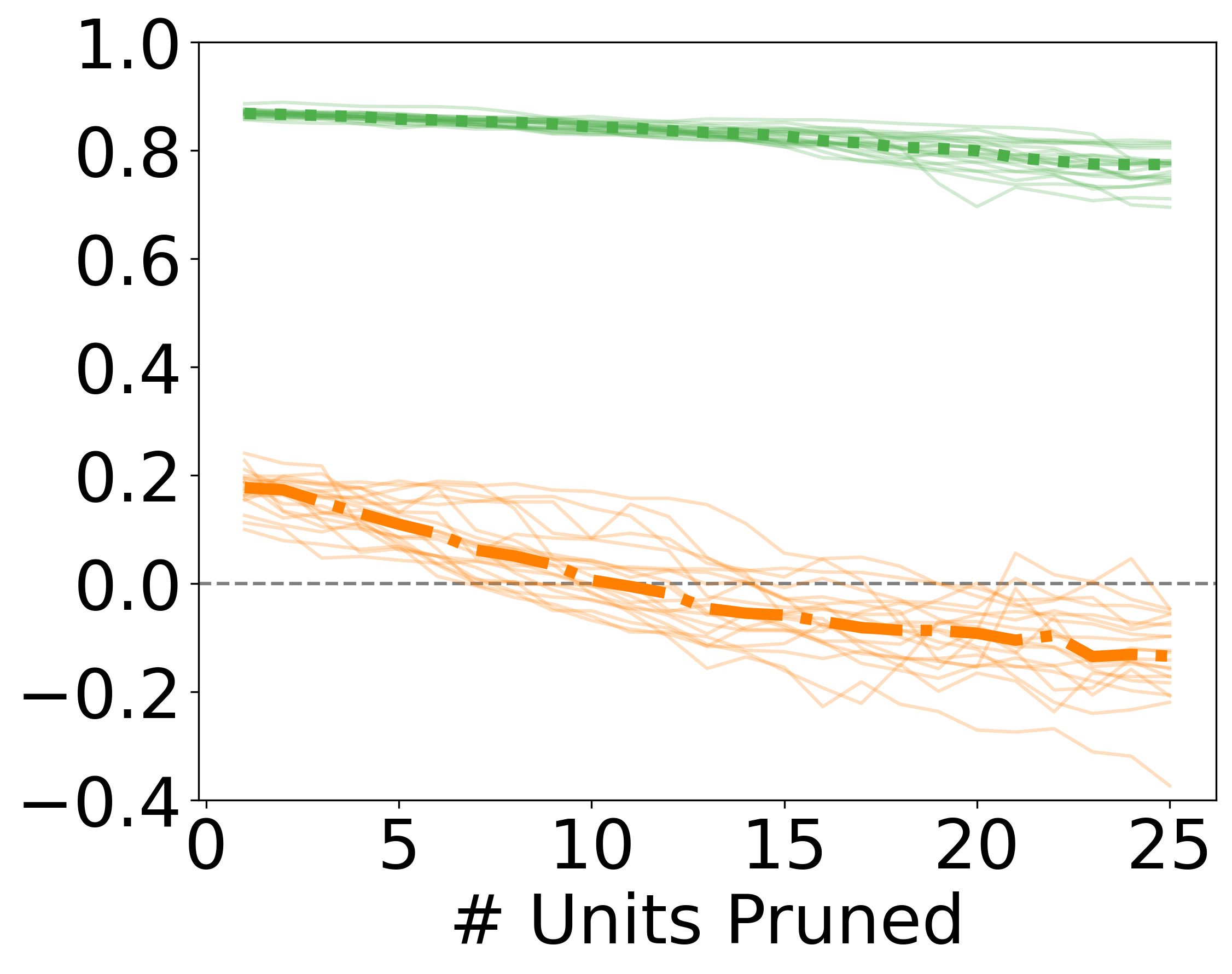}
    }
    \subfigure[{\scriptsize\centering $\;\;\;$ Bank, \textsc{Pruning},$\;\;\;$ EOD}]{    
        \includegraphics[width=0.21\linewidth]{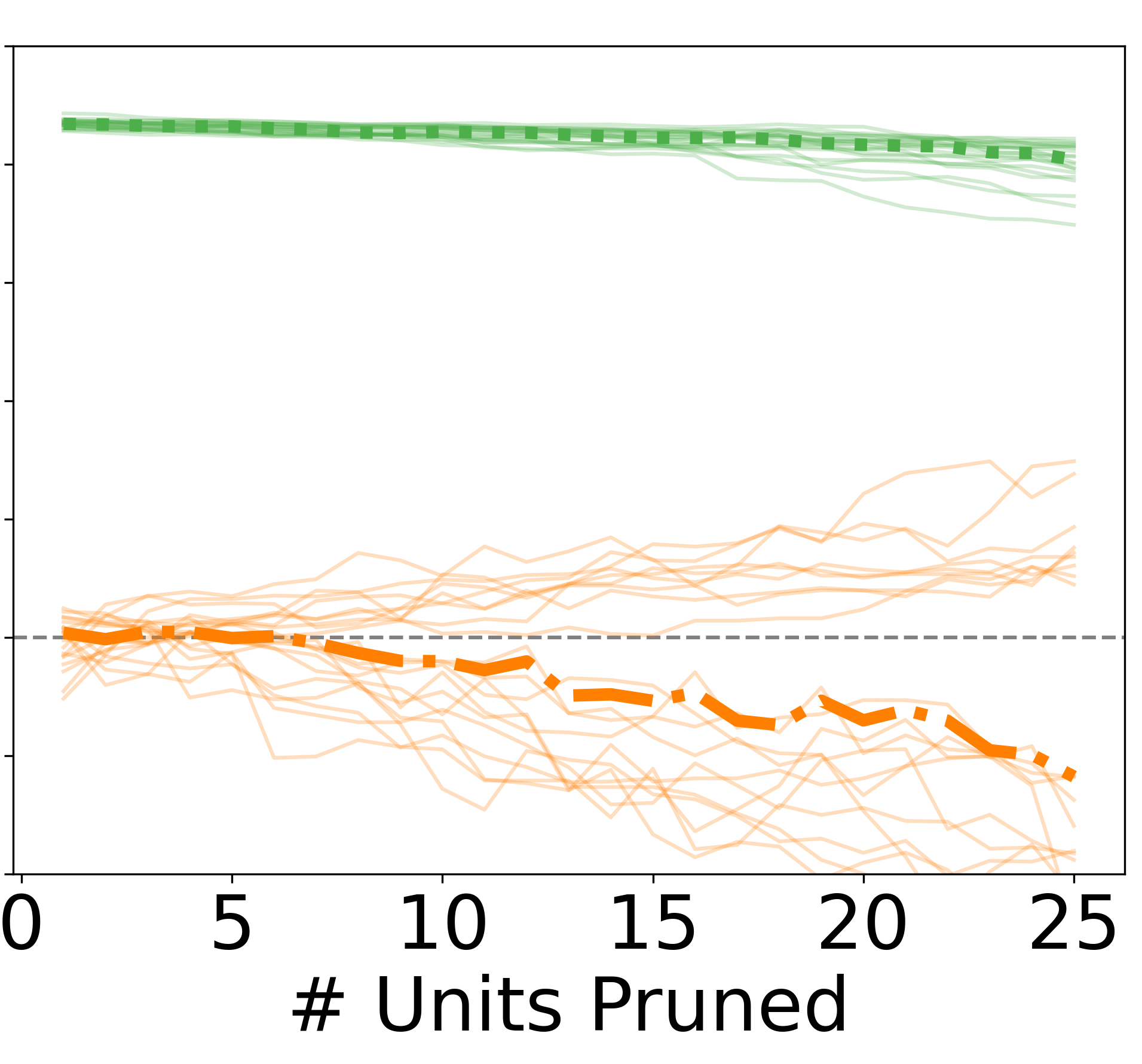}
    }
    \subfigure[{\scriptsize\centering Bank, \textsc{Bias GD/A}, SPD}]{
        \includegraphics[width=0.215\linewidth]{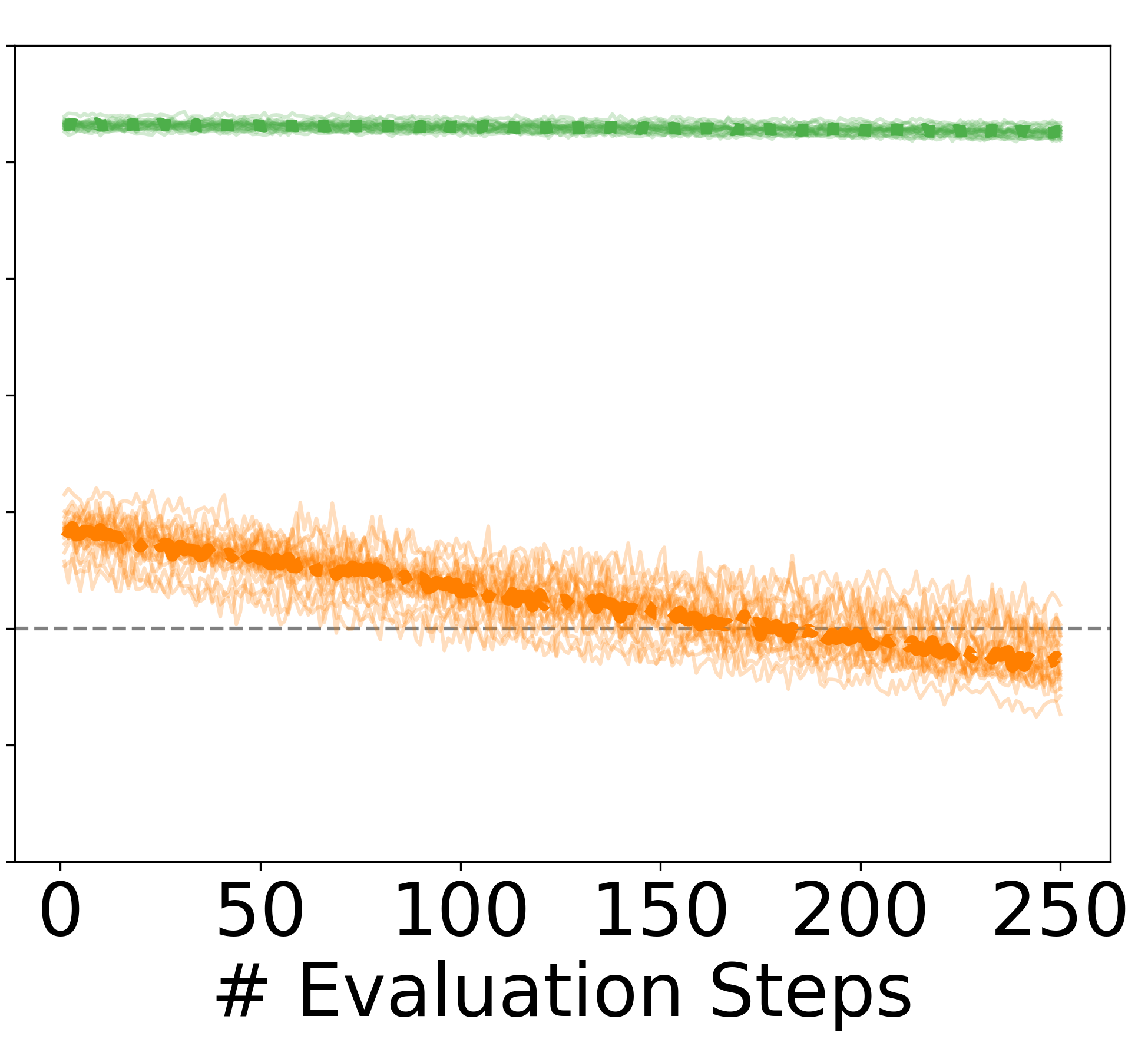}
    }
    \subfigure[{\scriptsize\centering Bank, \textsc{Bias GD/A}, EOD}]{        
        \includegraphics[width=0.215\linewidth]{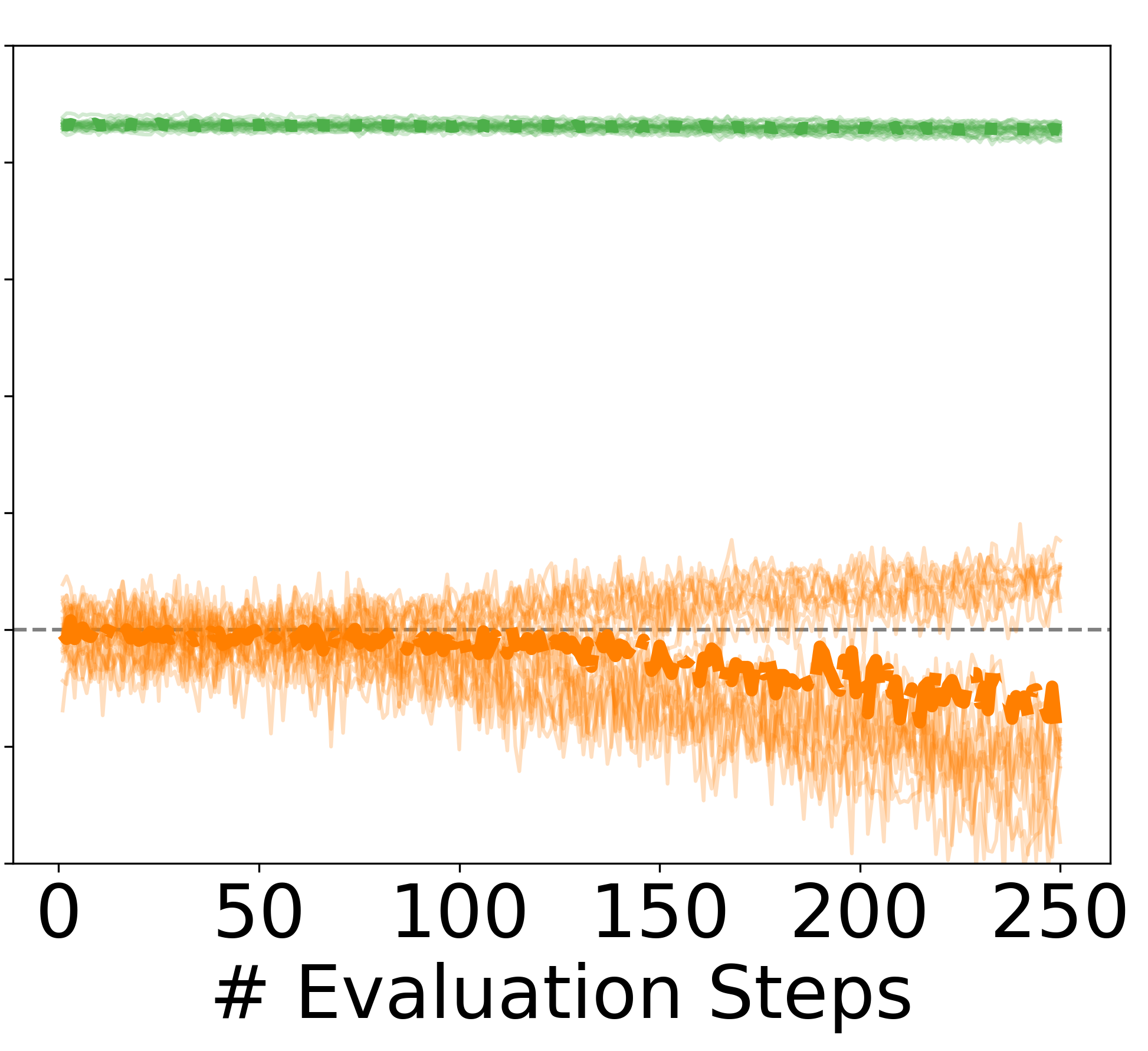}
    }
    
    \centering
    \subfigure[{\scriptsize\centering $\;\;\;$ COMPAS, \textsc{Pruning},$\;\;\;$ SPD}]{
        \includegraphics[width=0.25\linewidth]{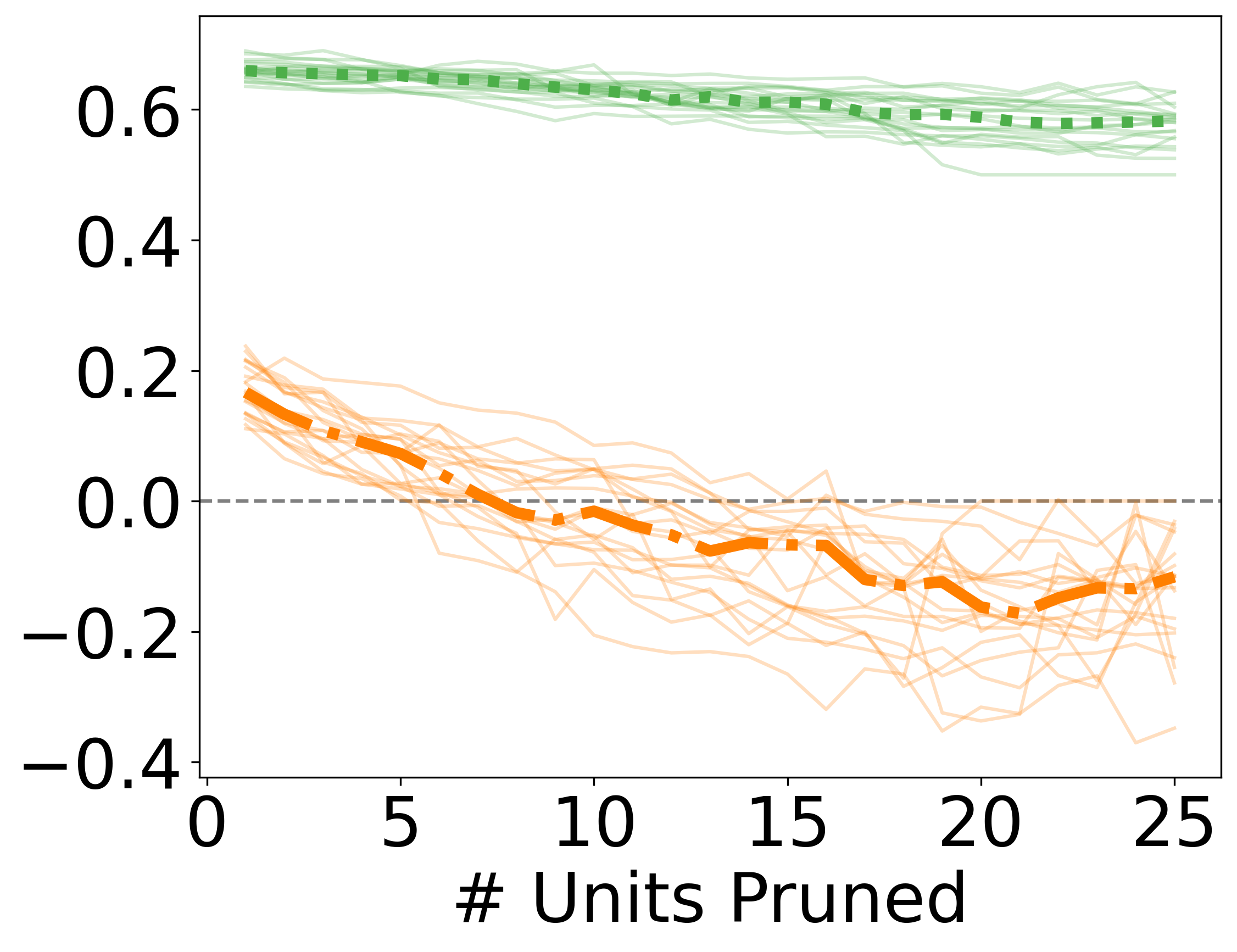}
    }
    \subfigure[{\scriptsize\centering COMPAS, \textsc{Pruning}, EOD}]{    
        \includegraphics[width=0.21\linewidth]{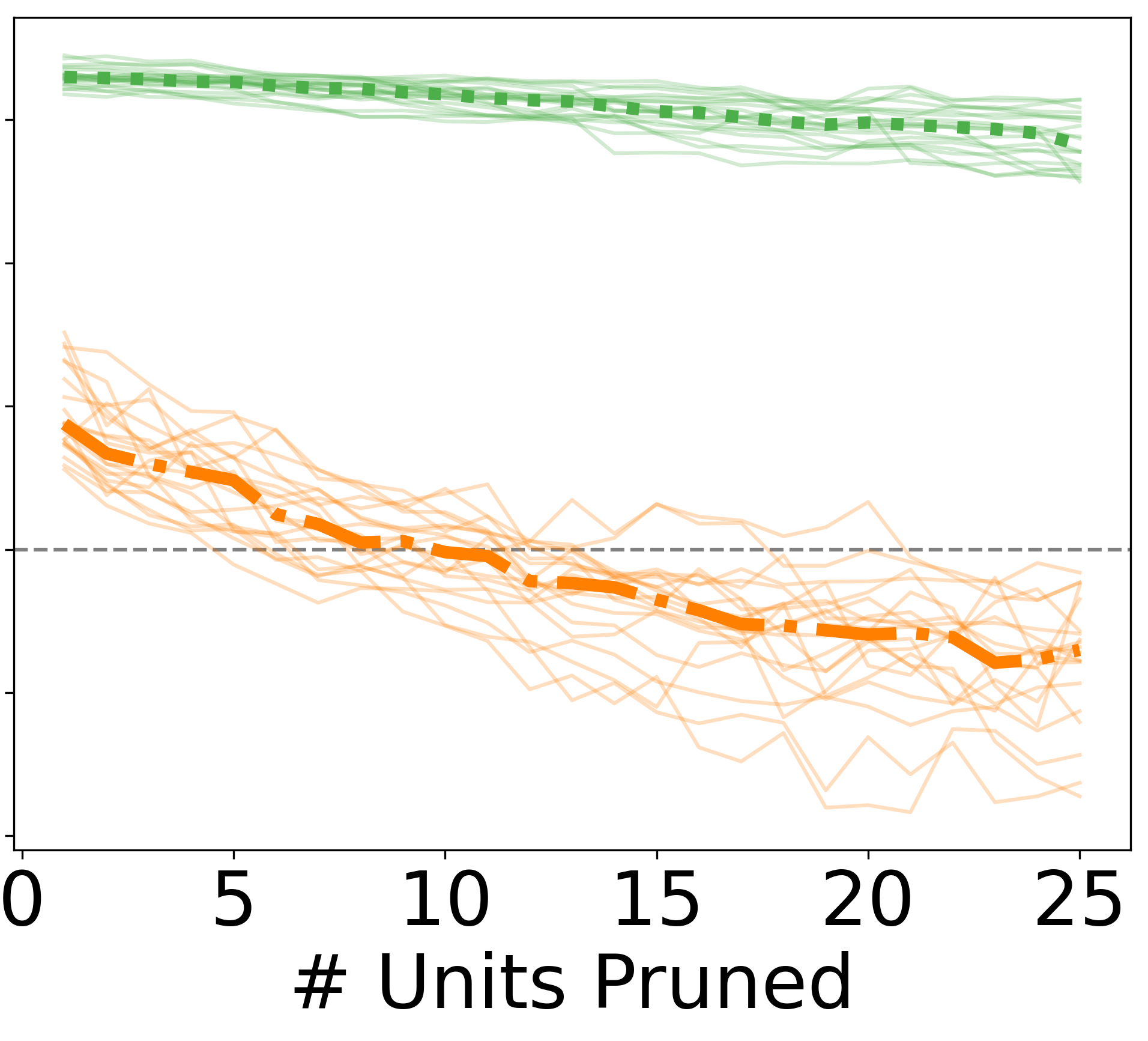}
    }
    \subfigure[{\scriptsize\centering COMPAS, \textsc{Bias GD/A}, SPD}]{
        \includegraphics[width=0.215\linewidth]{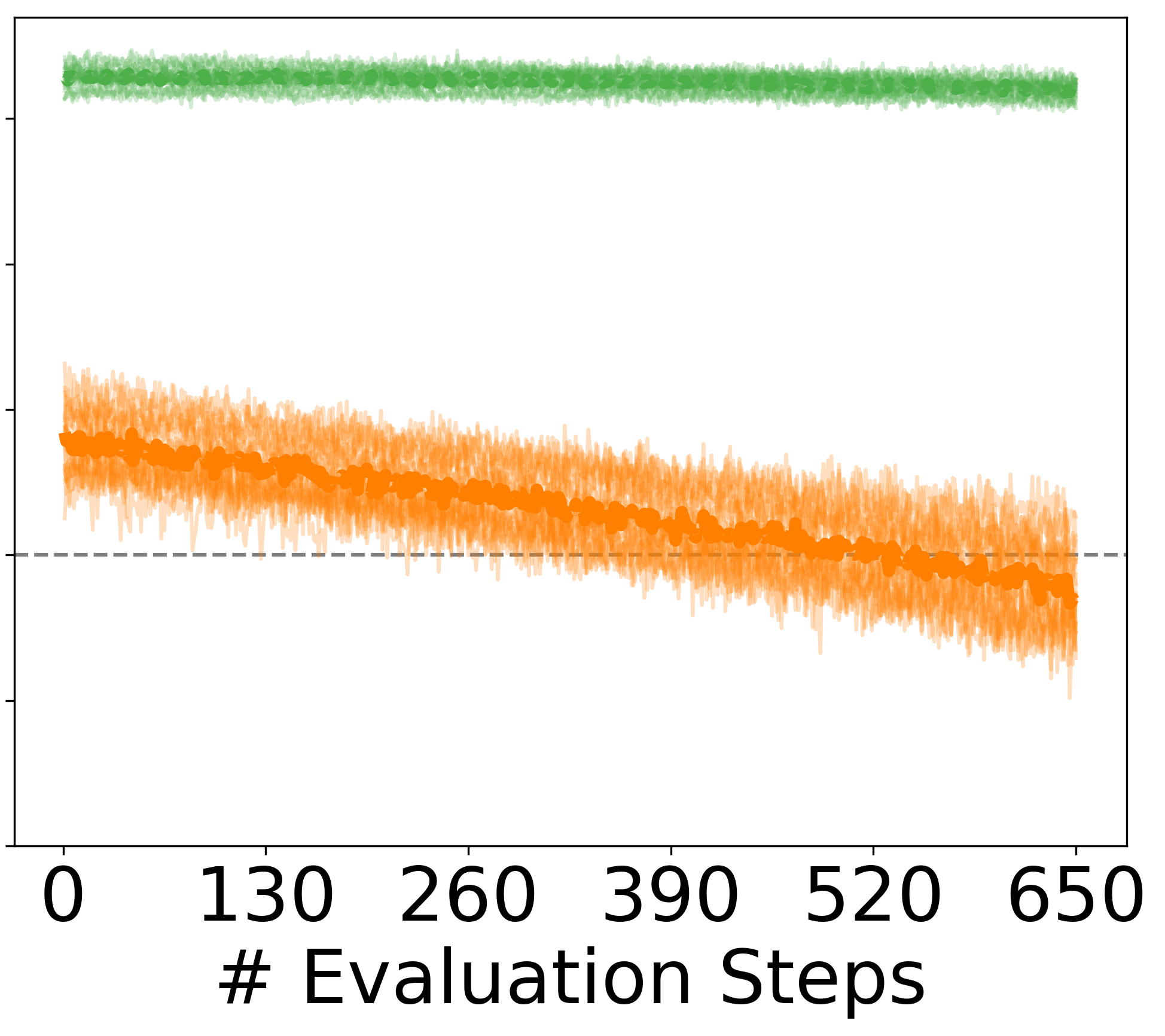}
    } 
    \subfigure[{\scriptsize\centering COMPAS, \textsc{Bias GD/A}, EOD}]{        
        \includegraphics[width=0.215\linewidth]{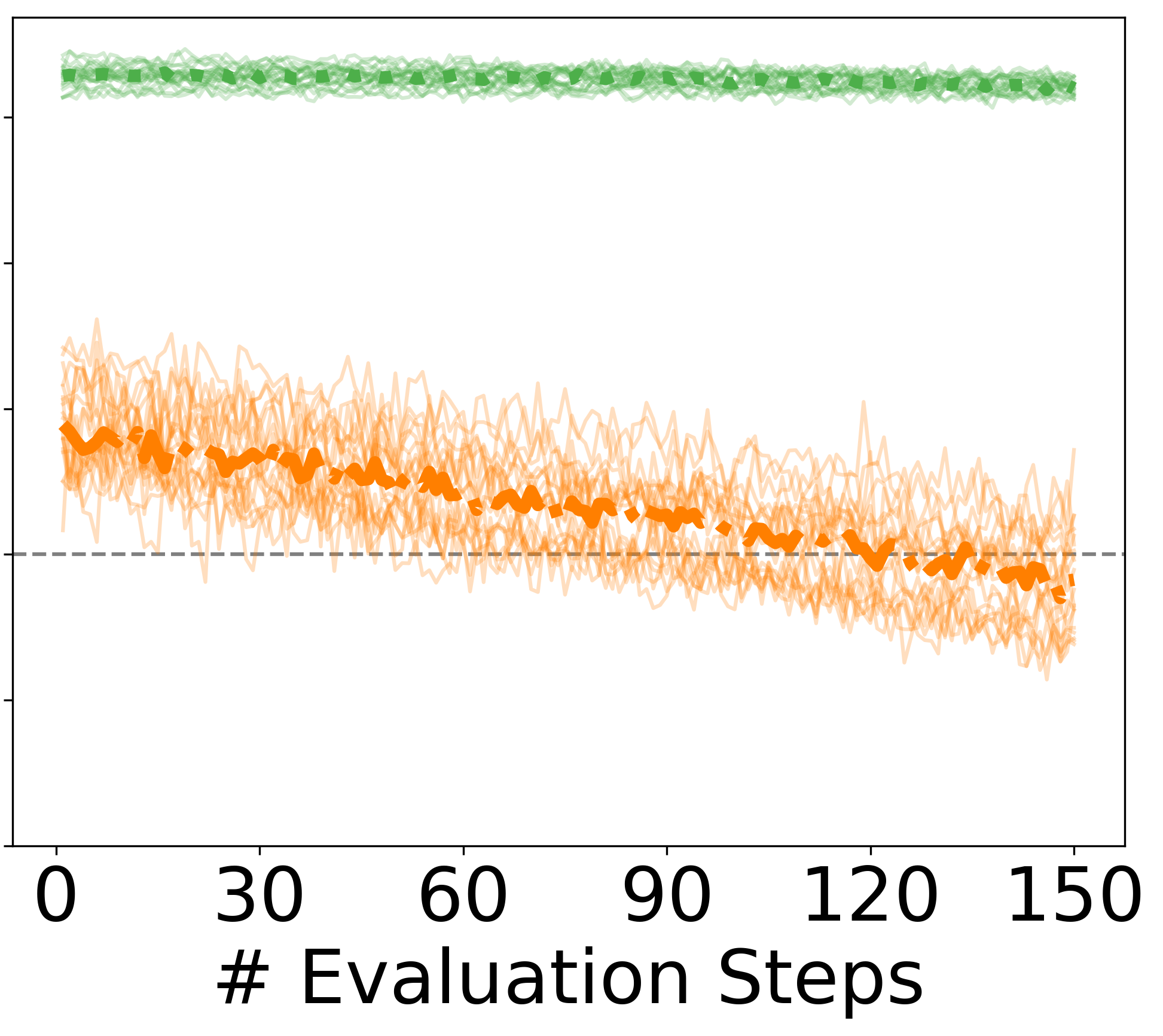}
    }
    \includegraphics[width=0.25\linewidth]{figures/legend_trajectories.png}
    \caption{Changes in the {\color{LegendaryOrange}\textbf{bias}}, given by the SPD (\emph{a, c, e, g, i, k}) and EOD (\emph{b, d, f, h, j, l}), and {\color{LegendaryGreen}\textbf{balanced accuracy}} of the neural network during pruning (\emph{a, b, e, f, i, j}) and bias gradient descent/ascent (\emph{c, d, g, h, k, l}). The results were obtained on Adult (\emph{top}), Bank (\emph{middle}), and \mbox{COMPAS} (\emph{bottom}) from 20 train-test splits. \textbf{Bold} lines correspond to the median across 20 seeds. During the bias GD/A, the model was evaluated several times an epoch. Notably, both procedures reduce bias without a considerable effect on accuracy.\label{fig:trajectories_full}}
\end{figure}

\end{document}